\documentclass{article}

\usepackage[preprint]{neurips_2026}

\usepackage[utf8]{inputenc} 
\usepackage[T1]{fontenc}    
\usepackage{hyperref}       
\usepackage{url}            
\usepackage{booktabs}       
\usepackage{amsfonts}       
\usepackage{amsmath,amssymb,amsthm}
\usepackage{mathtools}
\usepackage{thmtools}
\usepackage{thm-restate}

\usepackage{enumitem}
\usepackage{multirow}

\usepackage{nicefrac}       
\usepackage{microtype}      
\usepackage{xcolor}         

\usepackage{graphicx}
\usepackage{subcaption}
\usepackage{siunitx} 
\usepackage{todonotes}
\usepackage{placeins}
\usepackage{tikz}
\usepackage{pifont} 
\usepackage[export]{adjustbox} 
\usepackage{csquotes} 

\usepackage{comment}

\definecolor{tabblue}{HTML}{1F77B4}
\definecolor{taborange}{HTML}{FF7F0E}
\definecolor{tabgreen}{HTML}{2CA02C}
\definecolor{ourblue}{rgb}{0.368,0.507,0.71}    
\definecolor{ourdarkblue}{HTML}{354d72}         
\definecolor{ourlightblue}{HTML}{b4c5dc}    

\definecolor{ourorange}{rgb}{0.881,0.611,0.142} 
\definecolor{ourdarkorange}{HTML}{8f6213}         
\definecolor{ourlightorange}{HTML}{efc785}    

\definecolor{ourgreen}{rgb}{0.56,0.692,0.195}   
\definecolor{ourdarkgreen}{HTML}{677f24}         
\definecolor{ourlightgreen}{HTML}{cee298}  

\definecolor{ourred}{rgb}{0.923,0.386,0.209}    
\definecolor{ourdarkred}{HTML}{94300f}         
\definecolor{ourlightred}{HTML}{f7c5b2}  

\definecolor{ourviolet}{rgb}{0.528,0.471,0.701} 
\definecolor{ourdarkviolet}{HTML}{463b68}         
\definecolor{ourlightviolet}{HTML}{bcb3d4}

\newcommand{\E}{\mathbb{E}}

\newcommand{\KL}{\mathrm{KL}}

\usepackage{float}
\usepackage{algorithm}
\let\AND\relax
\usepackage{algorithmic}

\usepackage{adjustbox}

\usepackage[capitalize,noabbrev]{cleveref}

\theoremstyle{plain}
\newtheorem{theorem}{Theorem}[section]

\newtheorem{proposition}[theorem]{Proposition}
\newtheorem{lemma}[theorem]{Lemma}
\newtheorem{corollary}[theorem]{Corollary}
\theoremstyle{definition}
\newtheorem{definition}[theorem]{Definition}

\newtheorem{counterexample}[theorem]{Counter-Example}
\theoremstyle{remark}
\newtheorem{remark}[theorem]{Remark}

\usepackage[toc,page,header]{appendix}
\usepackage{minitoc}

\title{Stochastic Decision Horizons for\\Survival-Based Constrained Reinforcement Learning}

\newcommand{\affmpi}{\textsuperscript{1}}
\newcommand{\affscads}{\textsuperscript{2}}
\newcommand{\affhicbr}{\textsuperscript{3}}
\newcommand{\afftub}{\textsuperscript{4}}
\newcommand{\affmpiis}{\textsuperscript{5}}
\newcommand{\eqcontrib}{\textsuperscript{*}}
\newcommand{\eqsupervision}{\textsuperscript{\textdagger}}

\author{%
	\begin{tabular}{c}
		\textbf{Nikola Milosevic}\affmpi\affscads\eqcontrib \quad
		\textbf{Leonard T. Franz}\affhicbr\afftub\eqcontrib \quad
		\textbf{Daniel Haeufle}\affhicbr\afftub \\[0.25em]
		\textbf{Georg Martius}\afftub\affmpiis \quad
		\textbf{Nico Scherf}\affmpi\affscads\eqsupervision \quad
		\textbf{Pavel Kolev}\afftub\eqsupervision
	\end{tabular}
}

\begin{document}

    \doparttoc 
    \faketableofcontents 
    
    \maketitle
    
    \begingroup
    \renewcommand{\thefootnote}{}
    \footnotetext{%
        \textsuperscript{1} Max Planck Institute for Human Cognitive and Brain Sciences, Leipzig.
        \textsuperscript{2} Center for Scalable Data Analytics and Artificial Intelligence (ScaDS.AI), Dresden/Leipzig.
        \textsuperscript{3} Hertie Institute for Clinical Brain Research \& Center for Integrative Neuroscience.
        \textsuperscript{4} University of Tübingen.
        \textsuperscript{5} Max Planck Institute for Intelligent Systems, Tübingen.\linebreak
        \textsuperscript{*} Equal contribution.
        \textsuperscript{\textdagger} Equal supervision.
        Correspondence to: Milosevic \texttt{nmilosevic@cbs.mpg.de}, Franz and Kolev \texttt{\{leonard-tobias.franz, pavel.kolev\}@uni-tuebingen.de}.
        Project website: \url{https://tinyurl.com/sdh-crl}
    }
	\begin{abstract}
    	We propose \emph{stochastic decision horizons} (SDH), a theoretically grounded framework for survival-based constrained reinforcement learning from per-step violation signals. Rather than treating violations as expenditure of a cumulative budget, SDH shortens the effective decision horizon when violations occur, reducing current reward credit and future bootstrapping. We show that this construction represents a geometric-horizon survival constraint and, for a fixed Lagrange multiplier, remains the same SDH problem up to a constant reward offset that acts as an implicit incentive for survival. Using control as inference, we derive two off-policy and replay-compatible, regularized algorithms with post-violation semantics. Absorbing-state ends the decision process, so only surviving decisions pay policy KL, yielding AS-SAC. Virtual termination stops reward credit while the process continues, retaining ordinary discounted KL and yielding VT-MPO. To connect SDH with cumulative-cost CMDPs, we introduce \emph{violation-depth profiles}. Under exponential-tail single-scale profiles, the SDH statistic determines cumulative cost; departures from this form mark where the connection breaks. Experiments validate both the method and this predicted scope. On the 90-muscle H2190 humanoid in Hyfydy, VT-MPO matches the EWA reference in peak gait realism, with its highest observed gait-match checkpoint at 21M environment steps versus 90M for EWA and markedly more stable training. On Safety Gymnasium, violation-depth profiles predict the regimes in which SDH achieves strong reward-violation trade-offs.
    \end{abstract}

    \section{Introduction}
    
    Many control problems expose limits at every interaction step: gait quality depends on velocity, ground-reaction force, joint torque, and muscle use; robotic and physical systems likewise operate within local ranges. These signals are not always well described as consumption of a transferable budget. The standard expected discounted cumulative-cost CMDP~\citep{altman1999cmdp} instead asks whether the sum of costs stays below a budget, which is natural for quantities such as fuel or total exposure. Early termination and continuation already provide practical mechanisms for handling such per-step violation signals~\citep{sun2022constrained,chane2024cat,chane2024soloparkour}. But what constrained objective, if any, do these continuation mechanisms actually optimize? And if they solve a constrained problem, why can they sometimes avoid the usual Lagrange-dual optimization needed to enforce the constraint?
    
    We propose \emph{stochastic decision horizons} (SDH) as a framework for survival-based constrained reinforcement learning. A state-action continuation model $\alpha(s,a)\in[0,1]$ maps one or more local violation signals to a model-defined continuation probability. This induces the transformation $\tilde r=\alpha r$ and $\tilde\gamma=\gamma\alpha$: a violating transition loses current reward credit and reduces future bootstrapping. Thus ``shortening the effective horizon'' is an intuition; formally, SDH uses a survival-weighted reward and a state-action-dependent discount~\citep{pmlr-v70-white17a}. We show that this construction exactly represents a survival-gated geometric-horizon chance-constrained problem whose objective and constraint share the same survival process. Because the local violation signals enter through the same continuation model, the resulting problem has a single survival constraint and hence a single scalar Lagrange multiplier $\eta$. For a fixed multiplier $\eta$, we show that its Lagrangian closes within the SDH Bellman objective through the reward shift $r\mapsto r+\eta$. Since constant rewards are survival-weighted, this shift directly incentivizes survival. We further show that, when a bounded control cost $f$ is written as a reward $r=1-f$, the constant $1$ is no longer an irrelevant reward offset: under SDH, it generates the survival-probability term that appears in the constraint, even at $\eta=0$. If this fixed objective already satisfies the survival constraint, no dual optimization is needed; otherwise, the constrained problem can be optimized with a primal-dual method.
    
    To realize this survival objective in high-dimensional continuous control, we derive replay-compatible off-policy algorithms. This requires specifying the post-failure semantics of policy regularization: should survival weighting apply to policy KL as well as reward, or should policy KL remain on the ordinary discounted horizon? Control as inference (CaI)~\citep{toussaint2006probabilistic,rawlik2010CaI,levine2018reinforcement} provides the machinery to answer this question. We extend CaI to SDH and show that the resulting inference problem has finite evidence and a well-defined ELBO, and formalize two coherent post-violation semantics. Under \emph{absorbing-state} (AS) semantics, continuation failure ends the controlled decision process, so reward and policy KL are weighted by the same survival process; this yields a variable-discount soft Bellman objective and motivates AS-SAC, built on Soft Actor-Critic~\citep{haarnoja2018soft}. Under \emph{virtual termination} (VT), continuation failure stops reward credit but not the decision process, so reward remains survival-weighted while policy KL stays on the ordinary discounted horizon; this yields VT-MPO, which combines the SDH critic with MPO's KL-constrained policy improvement~\citep{abdolmaleki2018maximum}.
    
    \begin{figure*}[t]
        \centering
        \includegraphics[width=\linewidth]{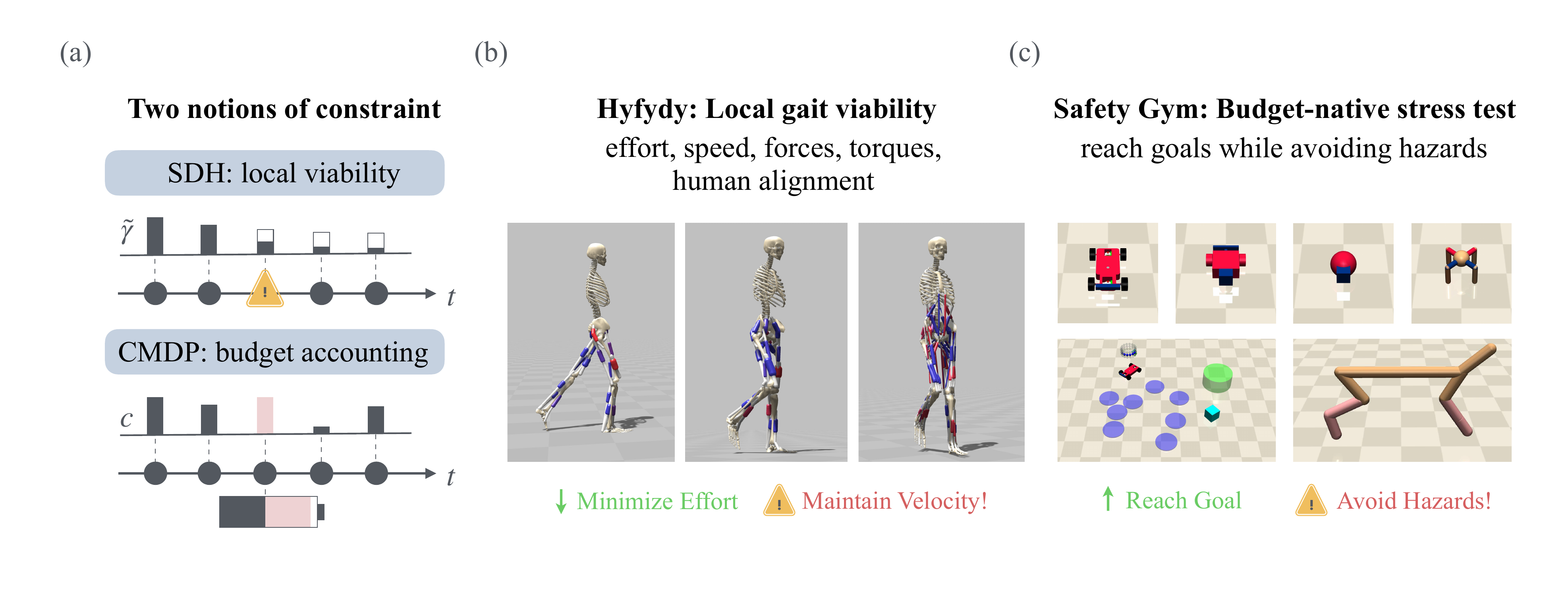}\vspace{-1em}%
        \caption{
            \textbf{Stochastic decision horizons vs. cumulative-cost accounting.}
            Cumulative-cost CMDPs add violations over time. SDH maps per-step violations to continuation probability $\alpha$ and attenuates reward and bootstrapping through $(\tilde r:=\alpha r,\tilde\gamma:=\alpha\gamma)$. The right objective depends on the application: Hyfydy local operating-limit signals are naturally modeled through survival, while violation-depth profiles expose when budget-native Safety Gymnasium tasks require additive-cost accounting.\vspace{-1.5em}
        }
        \vspace{-1em}
        \label{fig:budget-vs-viability}
    \end{figure*}
    
    We next ask whether this local survival formulation can recover cumulative-cost behavior at all, despite never tracking accumulated cost. Exponential continuation makes SDH depend on cumulative violation depth, but through an exponential transform rather than the additive statistic used by cumulative-cost CMDPs. We therefore introduce the \emph{violation-depth profile}: cumulative cost is its unweighted area, whereas SDH reads one Laplace-transform statistic. In general, one does not determine the other; under an exponential-tail single-scale profile, however, we show that the SDH statistic determines cumulative cost. Departures from this single-scale form break this correspondence.

    We validate both regimes empirically. Our main application is high-fidelity musculoskeletal control, where velocity, ground-reaction force, torque, and muscle-use violations indicate deterioration of the current gait rather than expenditure of a transferable budget. We introduce a survival-constrained RL formulation of locomotion under such heterogeneous biomechanical operating limits. In Hyfydy~\citep{Geijtenbeek2021Hyfydy}, locomotion minimizes a bounded \emph{effort} cost $f$ through the reward $r=1-f$, so the same survival-probability term that appears in the constraint is already present in the fixed $\eta=0$ objective, allowing training without dual optimization. On the 90-muscle H2190, VT-MPO matches the EWA reference in gait realism~\citep{schumacher2025emergancenatural} with markedly more stable training; its highest observed gait-match checkpoint occurs at 21M environment steps versus 90M for EWA. On Safety Gymnasium~\citep{ray2019benchmarking}, violation-depth profiles predict where SDH excels and where cumulative constraints are required.

    Together, these results carry the motivating class of real-world control problems from modeling to solution: SDH provides a survival-constrained formulation and yields two principled off-policy algorithms; we then formulate high-dimensional  biomechanical control as an SDH problem and solve it, while the same theory predicts when cumulative-cost control is required.
    
    \paragraph{Related work.}
    Early termination and continuation have been used to handle per-step violation signals~\citep{sun2022constrained,chane2024cat, chane2024soloparkour}, while chance-constrained RL studies geometrically weighted prefix-safety probabilities~\citep{pmlr-v139-wagener21a}. We do not claim continuation or geometric weighting alone as novel. SDH instead introduces a survival-gated geometric-horizon constrained problem in which a survival process jointly gates the task objective and a single survival constraint. This matched structure yields one scalar Lagrange multiplier, fixed-$\eta$ closure with $\eta$ acting as a survival incentive, a well-defined CaI ELBO with principled AS/VT off-policy algorithms, and the violation-depth profile scope analysis.
    
    Standard cumulative-cost methods include CPO and trust-region methods~\citep{achiam2017cpo}, primal-dual schemes~\citep{ray2019benchmarking,stooke2020responsive}, penalty/interior-point methods~\citep{liu2020ipo,zhang2022penalizedproximalpolicyoptimization,zhang2024constrained,dey2024p2bpo,usmanova2024log}, and off-policy constrained critics~\citep{yang2021wcsac,tessler2018reward,roy2021direct,koirala2024fawac,zheng2024safeofflinereinforcementlearning}. State augmentation can instead carry residual budget in the state~\citep{sootla2022saute}, capturing residual-budget semantics that state-action-only SDH does not represent.
    
    Hard state-wise safety is a stronger contract~\citep{zhao2023ijcai-state}. Shielding can enforce temporal specifications by intervening on unsafe actions~\citep{alshiekh2018aaai-safe}; reachability, barrier-function, and projection methods provide related action-level mechanisms~\citep{so2023solving,dalal2018safe,chow2019lyapunov}. Such methods may alter action selection at deployment; SDH instead trains the deployed actor under multiplicative survival weighting and provides no hard action-level guarantee.
    
    Risk-sensitive RL changes how return distributions are evaluated~\citep{baeuerle2024risk}, while robust MDPs optimize against transition-model uncertainty~\citep{nilim2005robust}; SDH instead transforms reward and discount using observed violation signals. This differs from generic reward shaping~\citep{ng1999policy}, because continuation changes both reward credit and future discounting. Our regularized algorithms build on CaI~\citep{toussaint2006probabilistic,rawlik2010CaI,rawlik2012stochastic,levine2018reinforcement,geist2019theory}, SAC~\citep{haarnoja2018soft}, and MPO~\citep{abdolmaleki2018maximum}. Hyfydy provides musculoskeletal simulation~\citep{Geijtenbeek2021Hyfydy}; EWA~\citep{schumacher2025emergancenatural} is the domain-specific reference for emergent natural gait.

    \section{Stochastic Decision Horizons}
    \label{sec:sdh-main}

    We consider an infinite-horizon MDP $\mathcal M=(\mathcal S,\mathcal A,P,r,\gamma,\nu)$. Here $P(\cdot\mid s,a)$ is the transition kernel, $r:\mathcal S\times\mathcal A\to[0,R_{\max}]$ is a bounded non-negative reward function, $\gamma\in[0,1)$ is the discount, and $\nu$ is the initial-state distribution. A policy $\pi$ induces trajectories $\tau=(s_0,a_0,s_1,\ldots)$. The environment also exposes a non-negative per-step violation magnitude $c(s,a)\ge0$, with $c(s,a)=0$ denoting no local violation. We state the finite-state/action case for notation; the same Bellman and trajectory expressions extend to continuous spaces by replacing sums with integrals under the usual measurability and density assumptions.

    \subsection{Geometric-horizon survival constraint and Bellman equation}
    Discounted RL admits a geometric-horizon interpretation~\citep{toussaint2006probabilistic}. Let $T_\gamma$ be independent of the MDP trajectory and continuation variables, with
    $\Pr(T_\gamma=t)=(1-\gamma)\gamma^t$ for $t\in\{0,1,\ldots\}$. We add a model-defined survival process whose continuation probability is constructed from the per-step violation signal. The same survival event gates reward and constraint; this matching yields the SDH Bellman representation and fixed-$\eta$ Lagrangian closure. App.~\ref{app:survival-chance-proofs} gives details, and \Cref{fig:budget-vs-viability-app} summarizes its relation to pointwise survival chance constraints and cumulative-cost CMDPs.

    \begin{definition}[Survival-gated geometric-horizon problem]
        \label{def:survival-chance-main}
        Consider
        \begin{equation}
            C_t\mid s_t,a_t\sim\mathrm{Bernoulli}(\alpha(s_t,a_t)),
            \qquad
            B_t:=\prod_{k=0}^{t}C_k,
            \label{eq:survival-process-main}
        \end{equation}
        where $\alpha:\mathcal S\times\mathcal A\to[0,1]$ is a state-action continuation model. For a given survival-probability level $\rho\in(0,1)$, define
        \begin{equation}
            \max_\pi\;
            \mathbb E_\pi[B_{T_\gamma}r(s_{T_\gamma},a_{T_\gamma})]
            \quad\text{subject to}\quad
            \Pr_\pi(B_{T_\gamma}=1)\ge\rho.
            \label{eq:survival-chance-main}
        \end{equation}
    \end{definition}

    By independence of $T_\gamma$, the survival probability is $\Pr_\pi(B_{T_\gamma}=1)=(1-\gamma)\sum_{t\ge0}\gamma^t\Pr_\pi(B_t=1)$, so the constraint represents geometrically averaged survival over the random horizon, rather than hard pointwise safety. The continuation process terminates with probability $1-\alpha(s,a)$; this model-defined event should be distinguished from physical system failure.

    Marginalizing the continuation variables gives the survival-weighted reward $\tilde r=\alpha r$ and state-action-dependent discount $\tilde\gamma=\gamma\alpha$. The survival-weighted task critic is the unique fixed point
    \begin{equation}
        Q_{\rm surv}^\pi(s,a)
        =
        \tilde r(s,a)
        +
        \tilde\gamma(s,a)
        \mathbb E_{s'\sim P(\cdot\mid s,a),a'\sim\pi(\cdot\mid s')}
        [Q_{\rm surv}^\pi(s',a')],
        \label{eq:survival-bellman-main}
    \end{equation}
    since $0\le\tilde\gamma(s,a)\le\gamma<1$ makes the policy-evaluation operator a $\gamma$-contraction for bounded rewards. This defines the survival-transformed MDP. For any bounded function $f:\mathcal S\times\mathcal A\to\mathbb R$, define the unnormalized survival-weighted occupancy functional
    \begin{equation}
        \mu^\pi_\alpha(f)
        :=
        \mathbb E_\pi
        \Big[
        \sum_{t\ge0}\gamma^t
        \Big(\prod_{k=0}^{t}\alpha(s_k,a_k)\Big)
        f(s_t,a_t)
        \Big].
        \label{eq:survival-occupancy-main}
    \end{equation}

    \begin{theorem}[Geometric-horizon representation and survival Lagrangian]
        \label{thm:survival-lagrangian-main}
        For every bounded function $f$, the identity
        $(1-\gamma)\mu^\pi_\alpha(f)=\mathbb E_\pi[B_{T_\gamma}f(s_{T_\gamma},a_{T_\gamma})]$ holds.
        Consequently, the Lagrangian of~\eqref{eq:survival-chance-main} is
        \begin{equation}
            \mathcal L(\pi,\eta)
            =
            (1-\gamma)\mu^\pi_\alpha(r+\eta)-\eta\rho,
            \qquad \eta\ge0.
            \label{eq:survival-lagrangian-main}
        \end{equation}
        For fixed $\eta$, the Lagrangian closes within the SDH Bellman objective: maximizing $\mathcal L$ is equivalent to maximizing an SDH return with reward shifted by $\eta$.
    \end{theorem}

    The proof is in App.~\ref{app:survival_chance_problem}. This closure is made explicit by the identity
    \[
        (1-\gamma)\mu^\pi_\alpha(r+\kappa)
        =
        (1-\gamma)\mu^\pi_\alpha(r)
        +\kappa\,\Pr_\pi(B_{T_\gamma}=1),\quad\forall \kappa\geq0.
    \]
    Thus a constant reward shift adds exactly a survival-probability term. If a bounded control cost $f$ is encoded as a reward $r=1-f$, then
    \[
        (1-\gamma)\mu^\pi_\alpha(1-f)
        =
        \Pr_\pi(B_{T_\gamma}=1)
        -(1-\gamma)\mu^\pi_\alpha(f),
    \]
    therefore the unit reward already provides a survival incentive even at $\eta=0$. The multiplier $\eta$ strengthens this incentive through the transformed reward $\alpha(s,a)(r(s,a)+\eta)$; a fixed $\eta$ selects one Lagrangian operating point, while a given level $\rho$ can be targeted by primal-dual optimization.

    \subsection{Relation to cumulative-cost CMDPs}
    \label{sec:sdh-cmdp-main}
    The standard expected discounted cumulative-cost CMDP~\citep{altman1999cmdp} is
    \[
        \max_\pi J_r(\pi):=\E_\pi\!\Big[\sum_{t\ge0}\gamma^t r(s_t,a_t)\Big]
        \quad\text{subject to}\quad
        J_c(\pi):=\E_\pi\!\Big[\sum_{t\ge0}\gamma^t c(s_t,a_t)\Big]\le d.
    \]
    SDH is not a general solver for this additive-cost constraint. Separately, a state-action continuation model cannot reproduce exact pathwise residual-budget semantics without accumulated cost in the state (App.~\ref{app:sdh-vs-cmdp}). \Cref{fig:budget-vs-viability-app} illustrates the broader comparison developed below.
    
    Let $G_t:=\sum_{k=0}^t c(s_k,a_k)$ and define the \emph{violation-depth profile}
    $\Omega_\pi(b):=\E_\pi[\sum_{t\ge0}\gamma^t\mathbf 1\{G_t\ge b\}]$.
    
    \begin{proposition}[Additive cost and exponential SDH]\label{prop:additive-cost-exp-sdh}
    At the geometric horizon $T_\gamma$, the identity $(1-\gamma)\Omega_\pi(b)=\Pr_\pi(G_{T_\gamma}\ge b)$ holds.
    Under exponential continuation $\alpha_\lambda(s,a)=\exp\{-\lambda c(s,a)\}$,
    \[
        \bar S_\pi(\lambda):=(1-\gamma)\mu^\pi_{\alpha_\lambda}(1)
        =\Pr_\pi(B_{T_\gamma}=1)=\E_\pi[e^{-\lambda G_{T_\gamma}}],
    \]
    and
    \[
        J_c(\pi)=\E_\pi[G_{T_\gamma}]
        =(1-\gamma)\!\int_0^\infty\!\Omega_\pi(b)\,db,\qquad
        \bar S_\pi(\lambda)
        =1-\lambda(1-\gamma)\!\int_0^\infty\!e^{-\lambda b}\Omega_\pi(b)\,db.
    \]
    For fixed $\lambda_0>0$, one exponential statistic does not upper-bound $J_c$ in general: even in a two-state, two-action MDP, $\bar S_\pi(\lambda_0)$ can remain fixed while $J_c(\pi)$ becomes arbitrarily large (App.~\ref{app:sdh-vs-cmdp}).
    \end{proposition}
    
    Thus $J_c$ is the mean of $G_{T_\gamma}$, while exponential SDH observes one Laplace-transform statistic of the same distribution. This distinction disappears in the exponential-tail single-scale regime: if $\Omega_\pi(b)=(1-\gamma)^{-1}\exp\{-b/\tau_\pi\}$, with the required prefactor $(1-\gamma)^{-1}$, then $J_c(\pi)=\tau_\pi$ and $\bar S_\pi(\lambda)=(1+\lambda\tau_\pi)^{-1}$. Consequently, for any fixed $\lambda>0$, this gives $J_c(\pi)\le d \Longleftrightarrow \bar S_\pi(\lambda)\ge(1+\lambda d)^{-1}$ (App.~\ref{subsec:single-scale-viability-app}). Departures from this single-scale form break this equivalence. \cref{sec:experiments} tests this predicted boundary on Safety Gymnasium~\citep{ray2019benchmarking}.

    \begin{figure*}
        \centering
        \includegraphics[width=\linewidth]{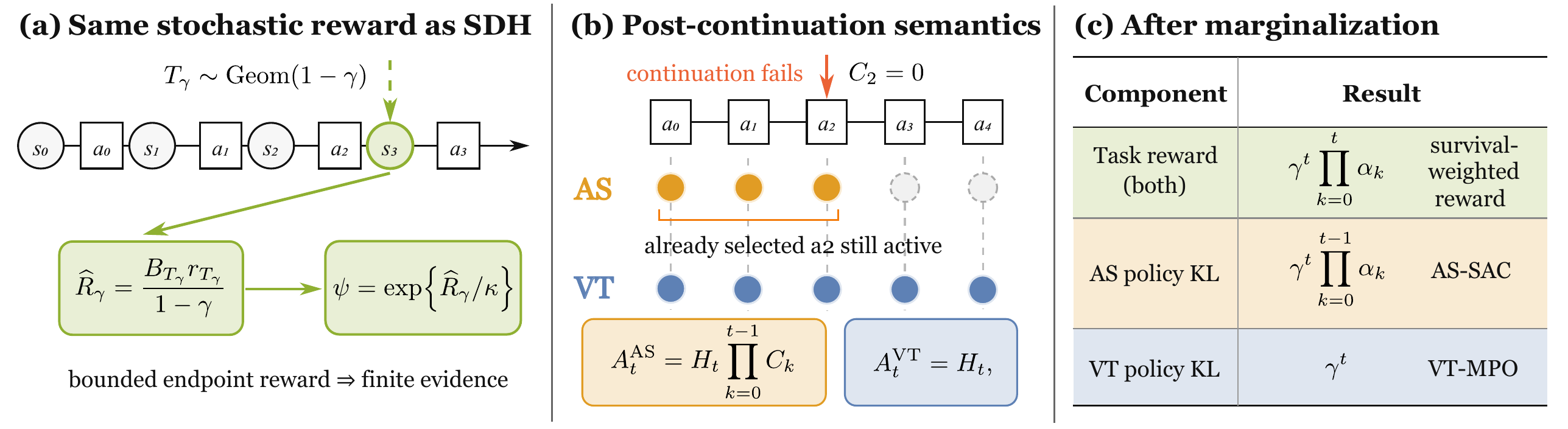}\vspace{-.5em}%
        \caption{
                \textbf{Control as inference with stochastic decision horizons.}
                The CaI construction applies a single optimality potential to the SDH endpoint reward $B_{T_\gamma}r_{T_\gamma}/(1-\gamma)$. AS and VT share this reward but differ in policy-KL horizon: AS removes future policy decisions after continuation failure, while VT retains ordinary discounted KL.
        }
        \label{fig:cai-sdh-as-vt}
    \end{figure*}

    \section{Control as Inference Formulation}
    \label{sec:cai_sdh}

    SDH determines which rewards count under the survival process, but regularized off-policy control leaves a second semantic question: which policy decisions should incur policy KL after continuation failure? Control as inference (CaI)~\citep{toussaint2006probabilistic,rawlik2010CaI,levine2018reinforcement} makes this choice principled, because the regularizer must arise as the trajectory KL between the corresponding reference- and policy-induced trajectory laws, rather than from an arbitrary weighting of per-step KL terms. By \Cref{thm:survival-lagrangian-main}, fixing $\eta$ changes the survival-constrained objective only through a reward shift $r\mapsto r+\eta$ plus a policy-independent constant. Hence, the analysis below applies directly to every fixed-$\eta$ Lagrangian operating point.

    In CaI, the key modeling choice is which trajectory reward defines the optimality potential. For SDH, we use exactly the random-time reward of the survival-constrained problem. Let $T_\gamma$ be the geometric evaluation time and $B_t$ the survival process, and define $\widehat R_\gamma=(1-\gamma)^{-1}B_{T_\gamma}r(s_{T_\gamma},a_{T_\gamma})$. The unnormalized potential is $\psi=\exp(\widehat R_\gamma/\kappa)$ under reference law $p$. By \Cref{thm:survival-lagrangian-main}, $\mathbb E_\pi[\widehat R_\gamma]=\mu_\alpha^\pi(r)$. Crucially, $\widehat R_\gamma$ contains one endpoint reward rather than an infinite trajectory sum; hence $|r|_\infty\le R_{\max}$ gives $\psi\le\exp(R_{\max}/[\kappa(1-\gamma)])$ and thus the partition function $Z:=\mathbb E_{p}[\psi]<\infty$. Therefore, for every $\kappa>0$, the construction matches the SDH objective and yields a well-defined CaI ELBO.
    
    What remains is to specify which policy decisions belong to the trajectory law after continuation failure. We consider two natural semantics. Under absorbing-state (AS) semantics, continuation failure sends the process to a policy-trivial absorbing state, so later actions disappear from the trajectory law. Under virtual termination (VT), continuation failure terminates reward credit only; the simulator and controller continue, so later action likelihoods remain active.
    
    Let $p_\sigma$ and $q_\sigma$ be the reference- and policy-induced trajectory laws for $\sigma\in\{\mathrm{AS},\mathrm{VT}\}$, sharing the environment, continuation variables, and geometric horizon, with policies $\pi_0$ and $\pi$, respectively. Their policy KL is (see \Cref{lem:sdh-semantic-kl})
    \begin{equation}
    \label{eq:semantic-kl-main}
    D_{\mathrm{KL}}(q_\sigma\|p_\sigma)
    =
    \mathbb E_{q_\sigma}\!\Big[
    \sum_{t\ge0}A_t^\sigma
    \log\frac{\pi(a_t\mid s_t)}{\pi_0(a_t\mid s_t)}
    \Big],
    \qquad
    A_t^\sigma=
    \begin{cases}
    H_t\prod_{k=0}^{t-1}C_k,&\sigma=\mathrm{AS},\\
    H_t,&\sigma=\mathrm{VT}.
    \end{cases}
    \end{equation}
    The one-step offset follows from the event ordering: $a_t$ is selected before $C_t$ is realized. Thus under AS the current action still incurs KL when $C_t=0$, whereas its reward credit and all subsequent decisions are gated.    
    Write $\mathcal J_\sigma$ for the $\kappa$-scaled CaI ELBO after marginalizing $C,D$.
    \begin{theorem}[AS and VT CaI objectives]
    \label{thm:as-vt-elbo-main}
    Let $\alpha_t:=\alpha(s_t,a_t)$, $\tilde r_t:=\alpha_t r(s_t,a_t)$, $\tilde\gamma_t:=\gamma\alpha_t$, and $u_t:=\prod_{k=0}^{t-1}\tilde\gamma_k$ with $u_0=1$. Then
    \begin{align}
    \mathcal J_{\sigma}(\pi,\pi_0)
    &=
    \mathbb E_{\tau\sim\pi}\!\Big[
    \sum_{t\ge0}u_t\tilde r_t
    -\kappa\sum_{t\ge0}v_t^\sigma
    \log\frac{\pi(a_t\mid s_t)}{\pi_0(a_t\mid s_t)}
    \Big],\quad
    v_t^\sigma=
    \begin{cases}
    u_t,&\sigma=\mathrm{AS},\\
    \gamma^t,&\sigma=\mathrm{VT}.
    \end{cases}
    \end{align}
    For a fixed Lagrange multiplier $\eta$ of the survival constraint, the same CaI objective applies with $r$ replaced by $r+\eta$, up to the policy-independent constant $-\eta\rho/(1-\gamma)$.
    \end{theorem}
    The endpoint ELBO is established in App.~\ref{app:sdh-elbo}, and its reduction to the AS and VT objectives is shown in \Cref{thm:as-vt-elbo}. AS and VT therefore share the same survival-weighted task objective but induce different ELBO-valid policy-KL horizons. Under AS, reward and policy KL share the same state-action-dependent discount, yielding the KL-regularized Bellman recursion used by AS-SAC. Under VT, policy KL remains ordinarily discounted, so the survival-weighted task critic is paired with KL-constrained policy improvement, yielding VT-MPO.

    \section{Algorithms}
    \label{sec:algorithms-main}

    AS and VT share the same survival-weighted task objective but differ in the policy-KL horizon induced by their post-continuation semantics. Replay buffer retains the raw rewards and per-step violation signals needed to construct each critic target. When constructing a target, we compute the model-defined continuation probability $\alpha_\lambda(s,a)$ (in our experiments, exponential $\alpha_\lambda(s,a)=\exp\{-\lambda\sum_i c_i(s,a)\}$), together with $\tilde r_\eta(s,a)=\alpha_\lambda(s,a)(r(s,a)+\eta)$ and $\tilde\gamma(s,a)=(1-d)\gamma\alpha_\lambda(s,a)$, using the current $\lambda$ and, during primal-dual optimization, the current $\eta$. During $\lambda$-annealing, this uses the current value along the ramp. Thus replayed experience is evaluated under the current continuation model $\alpha_{\lambda}$.
    
    \paragraph{AS-SAC.}
    Under AS, reward and policy KL share the same state-action-dependent discount. The KL-regularized Bellman recursion is
    \begin{align}
        Q_{\mathrm{AS}}^\pi(s,a)
        &=
        \tilde r_\eta(s,a)
        +
        \tilde\gamma(s,a)
        \mathbb E_{s'\sim P(\cdot\mid s,a)}[V_{\mathrm{AS}}^\pi(s')],
        \label{eq:as-soft-bellman-main}\\
        V_{\mathrm{AS}}^\pi(s)
        &=
        \mathbb E_{a\sim\pi(\cdot\mid s)}
        \left[
        Q_{\mathrm{AS}}^\pi(s,a)
        -
        \kappa\log\frac{\pi(a\mid s)}{\pi_0(a\mid s)}
        \right].
        \label{eq:as-soft-value-main}
    \end{align}
    For a normalized uniform prior, this yields maximum-entropy AS-SAC. We show that its survival-weighted living cost is policy dependent, unlike the corresponding constant in ordinary SAC, and give a counterexample where dropping it changes the optimizer. For adaptive $\kappa$, we use standard SAC-style temperature adaptation at policy improvement: the current regularized AS critic is fixed during the actor/temperature update, while subsequent critic updates use the updated temperature (App.~\ref{app:as:sac-temperature}).
    
    \paragraph{VT-MPO.}
    Virtual termination is objective-level reward termination without environment reset. Under VT, the task critic remains survival-weighted while policy KL stays on the ordinary discounted horizon. The critic obeys
    \begin{equation}
        \label{eq:qsurv-main}
        Q_{\mathrm{surv}}^\pi(s,a)
        =
        \tilde r_\eta(s,a)
        +
        \tilde\gamma(s,a)
        \mathbb E_{s'\sim P(\cdot\mid s,a),a'\sim\pi(\cdot\mid s')}
        [Q_{\mathrm{surv}}^\pi(s',a')].
    \end{equation}
    We therefore pair $Q_{\mathrm{surv}}^\pi$ with standard MPO relative-entropy-constrained policy improvement:
    \begin{equation}
        \label{eq:vt-mpo-boltzmann-main}
        q^\star(a\mid s)
        \propto
        \pi_0(a\mid s)
        \exp\!\left(\beta^{-1}Q_{\mathrm{surv}}^\pi(s,a)\right),
    \end{equation}
    where $\beta$ is the MPO temperature for the policy-improvement KL budget; $\eta$ is reserved for the survival constraint. Thus VT-MPO modifies MPO only through the survival-weighted task critic; retaining ordinary actor-side KL constraints is the consequence of the VT CaI semantics, not an approximation.

    \paragraph{Primal-dual optimization of the survival constraint.}
    The survival probability is
    \begin{equation}
        \label{eq:survival-probability-main}
        p_{\mathrm{surv}}(\pi)
        =
        (1-\gamma)\mu^\pi_\alpha(1)
        =
        \Pr_\pi(B_{T_\gamma}=1).
    \end{equation}
    When no specific constraint level is imposed, fixed $\eta$ selects a Lagrangian operating point without dual optimization; notably, for $r=1-f$, the constant $1$ already generates the survival-probability term, so even $\eta=0$ can provide an effective survival incentive. For a given survival-constraint level $\rho$, an auxiliary survival-mass critic ($r\equiv1$) can estimate $p_{\mathrm{surv}}(\pi)$ for the dual update. Since the survival constraint is scalar, this requires only one auxiliary survival-mass critic and one scalar dual variable $\eta$, regardless of the number of violation signals entering $\alpha$. For the exponential continuation, $\lambda$ controls how strongly violations attenuate continuation. Annealing $\lambda$ therefore changes the continuation model, and hence the SDH objective, making it an optimization curriculum rather than dual optimization of a fixed survival constraint.

    \section{Experiments}
    \label{sec:experiments}

    The experiments test two linked hypotheses. \textbf{H1:} a challenging musculoskeletal-control problem can be formulated as a fixed survival objective built from local operating-limit signals, yielding realistic gait and strong empirical limit adherence. \textbf{H2:} violation-depth profiles predict when this survival criterion agrees with, or can diverge from, the standard cumulative-cost CMDP criterion. Hyfydy, our motivating local-operating-regime setting, tests H1; Safety Gymnasium provides a budget-native scope test for H2.

    \subsection{Hyfydy: musculoskeletal operating limits}
    \label{sec:hyfydy-main}

    We evaluate VT-MPO on three Scone-Hyfydy humanoids~\citep{geijtenbeek2019scone,Geijtenbeek2021Hyfydy}: H0918 (2D, 9 DOFs, 18 muscles), H1622 (3D, 16 DOFs, 22 muscles), and H2190 (3D, 21 DOFs, 90 muscles). The domain reference is Effort Weight Adaptation (EWA)~\citep{schumacher2025emergancenatural}, which dynamically reweights effort penalties to resolve the effort-velocity trade-off. We use VT-MPO because constraint events do not reset the physical simulator, its MPO base reduces optimizer confounding with the EWA comparison, and the diagnostic in App.~\ref{app:hyfydy-diagnostics} is more robust across H2190 seeds. An unconstrained low-effort SAC/MPO control is not informative here: removing the gait operating limits makes trivial non-walking behavior optimal, rather than testing how to encode the same gait task.

    \paragraph{Formulation and protocol.}
    SDH uses an effort-based task reward and maps target velocity, ground-reaction force (GRF), joint-torque limits, and muscle-activation sparsity into
    $\alpha=\exp\{-\lambda(c_1+c_2+\cdots)\}$; the GRF limit is 1.3 body weights and the remaining model-specific limits are in App.~\ref{app:hyfydy-full-formulation}. These are operating-limit violation signals: exceeding a limit indicates deterioration of the current gait, not consumption of a transferable resource budget. For fixed $\eta=0$, we linearly ramp $\lambda$ from $0$ to $0.2$ over the first $5$M steps and hold it fixed thereafter. A separate H1622 survival-only diagnostic finds that driving policy improvement with the survival-critic $Q_1$ alone does not increase Monte-Carlo survival while reducing task return; see App.~\ref{app:hyfydy-survival-endpoint}. On the H1622 velocity-only diagnostic, removing this ramp or fixing $\lambda\in\{0.05,\ldots,1.0\}$ mainly affects early training, while the overall learning trend and limiting behavior remain similar; see App.~\ref{app:hyfydy-velocity-only} (\cref{fig:hyfydy-ramp-up,fig:lambda-p0-var}). Thresholds were calibrated per model before training and then held fixed, while the reward and continuation schedule were specified in advance; thus, after domain calibration, training proceeds with a fixed protocol, though the method is not tuning-free. Runs use 100M steps and three seeds; every 200k steps, each seed's gait-match is averaged over 10 deterministic rollouts. Following EWA~\citep{schumacher2025emergancenatural}, for each method we select the step with the highest across-seed mean and report mean $\pm$ std across seeds at that shared step. H0918/H1622 use $2\times256$ MLPs and H2190 uses a $2\times1024$ MLP; MPO uses 20 action samples per state and DEP exploration~\citep{schumacher2022deprl}. Approximate runtimes on an NVIDIA RTX 2080 are 21/25/36 hours for H0918/H1622/H2190, respectively (full settings in App.~\ref{app:hyfydy-full-formulation}).

    \begin{figure}[t]
        \centering
        \includegraphics[width=0.95\linewidth]{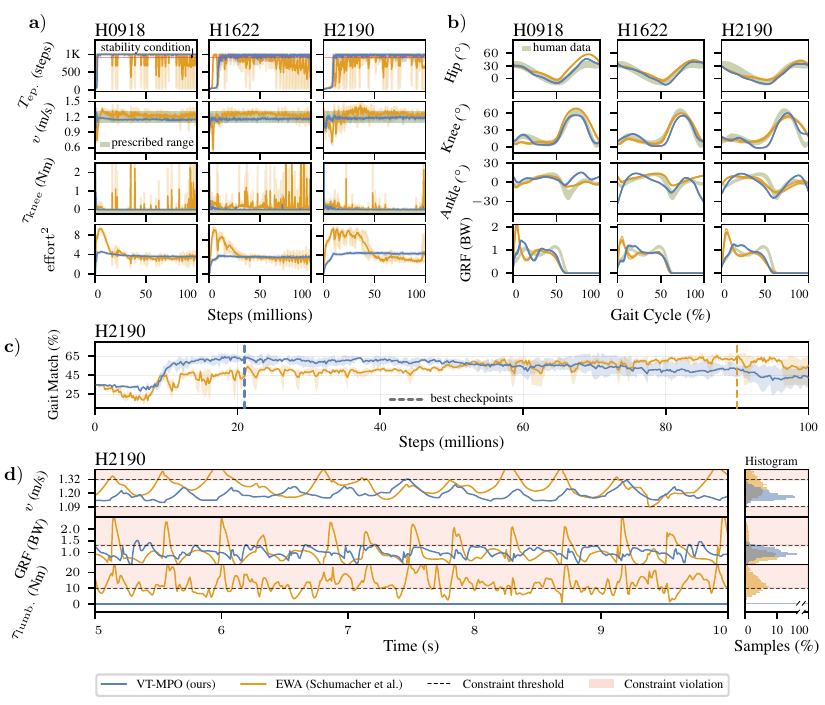}\vspace{-0.5em}%
        \caption{\textbf{HyFyDy natural human locomotion.} $\mathbf{a}$) Training dynamics for VT-MPO and EWA across three humanoids; VT-MPO reaches stable walking with low knee-limit violation. $\mathbf{b}$) Gait-cycle kinematics and GRF at the best-observed gait-match checkpoints. $\mathbf{c}$) H2190 gait-match over training; dashed lines mark the selected checkpoints (21M VT-MPO, 90M EWA). $\mathbf{d}$) H2190 selected-policy velocity, GRF, and lumbar-torque traces and histograms relative to operating limits.}
        \label{fig:cost_realism}
        \vspace{-1em}
    \end{figure}

    \paragraph{Training behavior.}
    Figure~\ref{fig:cost_realism}$\mathbf{a}$) shows that VT-MPO reaches stable walking within about 5-10M steps on all three models: episode length approaches the maximum, velocity stays near the target range, and knee over-extension is driven near zero. EWA shows larger oscillations as its adaptive reward weights change during training, and many low-effort checkpoints fail the episode-length stability requirement. Unlike EWA, SDH does not adapt reward weights during training, avoiding this source of target non-stationarity; the observed training curves are correspondingly smoother. Here, $r=1-\mathrm{effort}$ already provides a survival incentive at $\eta=0$, so no Lagrange-dual update is required. Detailed per-limit curves are in App.~\ref{app:hyfydy-full-results}.

    \paragraph{Operating-limit adherence and gait realism.}
    Gait realism is measured by gait-match percentage~\citep{geijtenbeek2019scone}, the fraction of the gait cycle within one standard deviation of experimental human kinematics; neither method optimizes this metric. Best-observed values under this protocol are comparable (Table~\ref{tab:gait_match}): $75.9\%$ vs. $75.8\%$ on H0918, $70.1\%$ vs. $69.3\%$ on H1622, and $62.9\%$ vs. $64.6\%$ on H2190. Figure~\ref{fig:cost_realism}$\mathbf{b}$) shows similar joint-angle and GRF profiles, panel $\mathbf{c}$) shows the H2190 gait-match trajectories and selected checkpoints, and panel $\mathbf{d}$) compares selected-policy velocity, GRF, and lumbar-torque rollouts against the prescribed operating limits. On H2190, VT-MPO reaches its highest observed gait match at 21M environment steps versus 90M for EWA, giving an earlier best-observed checkpoint under the common evaluation protocol. Model limitations such as missing toes/soft tissues~\citep{dhondt2024DynamicFootModel,buchmann2024EffectIncludingMobile} bound attainable realism. In the full problem formulation, fixed-$\kappa$ AS-SAC confirms that it is a viable but more aggressive alternative: successful runs can attain higher human-match values, reaching nearly $90\%$ on H0918, but with greater seed sensitivity (App.~\ref{par:hyfydy-full-as-sac-app}); separately, the velocity-only diagnostic includes one failed H2190 seed among three runs (App.~\ref{app:hyfydy-velocity-only}).

    \subsection{Safety Gymnasium: violation-depth profiles as a scope test}
    \label{sec:safety-gym-main}

    We evaluate eight Safety Gymnasium tasks~\citep{ji2023safetygymnasium} with standardized cost signals. AS-SAC and VT-MPO are compared with unconstrained MPO, on-policy CPO and C-TRPO, and off-policy WCSAC and MPO-PID. The benchmark is designed around expected cumulative-cost budgets, whereas SDH uses per-step signals under a survival constraint. We map the native budget $d$ to the per-step threshold $d(1-\gamma)$ used by SDH. Safety Gymnasium also has signed rewards, outside the non-negative-reward assumptions of Sec.~\ref{sec:sdh-main}; we therefore use a fixed positive reward shift $\eta$ in SDH updates while reporting native rewards. Hyperparameters are in App.~\ref{app:results-safety-gymnasium}. AS-SAC uses adaptive $\kappa$ in all Safety Gymnasium results (App.~\ref{app:as:sac-temperature}). For interpretation, we first assess whether each method satisfies the native cost limit, then compare reward among feasible methods. This benchmark therefore tests this scope boundary between SDH survival and native additive-cost accounting.

    \begin{figure}[t]
        \centering
        \includegraphics[width=\linewidth]{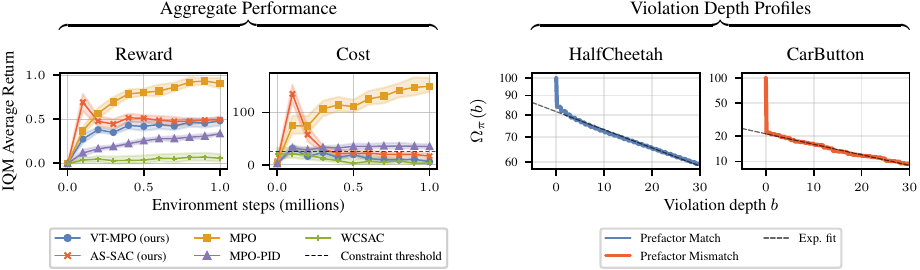}
        \caption{
            \textbf{Safety Gymnasium: aggregate performance and violation-depth profiles.} \emph{Left:} IQM reward and cost across environments and seeds~\citep{agarwal2021deep}. VT-MPO achieves the most favorable aggregate return-violation trade-off; both SDH methods substantially reduce cost relative to unconstrained MPO while retaining competitive reward. \emph{Right:} Violation-depth profiles $\Omega_\pi(b)$ for converged unconstrained SAC policies, with native cost limit $d=25$. \texttt{HalfCheetah} approximately matches the single-scale correspondence, whereas \texttt{CarButton} is also nearly exponential but has a strong prefactor mismatch; the environment-wide profile sweep is in App.~\ref{app:sgym-vdp-sweep}.}
        \label{fig:safety-aggregate-main}
        \vspace{-1em}
    \end{figure}

    \paragraph{Favorable profile geometry.}
    Figure~\ref{fig:safety-aggregate-main} (left) shows that VT-MPO and AS-SAC can substantially reduce violations relative to unconstrained MPO while retaining competitive reward across the benchmark. The full profile sweep across Safety Gymnasium environments (App.~\ref{app:sgym-vdp-sweep}) finds the closest single-scale correspondence for \texttt{Ant} and \texttt{HalfCheetah}. VT-MPO is generally more conservative, consistent with its KL-constrained policy improvement, while maximum-entropy AS-SAC can attain higher reward but is more seed-sensitive.

    \paragraph{Scope boundary.}
    \texttt{CarButton} provides the clearest mismatch: despite an exponential-shaped profile, its prefactor violates the single-scale correspondence, and the fixed-$\lambda$ AS-SAC ablation either misses the native cost limit or collapses reward (Table~\ref{tab:sgym-lambda-sweep}). \texttt{PointGoal} is inconclusive under the profile diagnostic but mixed empirically: neither SDH method is feasible at 1M steps, while at 3M steps AS-SAC becomes feasible and reaches $R=18.1\pm3.3$ with $C=23.3\pm11.4$ (Table~\ref{tab:sgym-increased-budget}). Thus, failed single-scale correspondence limits additive-cost recovery without implying SDH failure.
    
    \section{Limitations}

    SDH does not provide hard per-action or state-wise safety certification, a generic additive-cost budget guarantee, or tuning-free continuation design. A low-$\alpha$ action is discouraged through reduced reward credit and future bootstrapping, but it is not prohibited. Thus, when certified exclusion of unsafe actions or temporal guarantees are required, SDH should be combined with or replaced by complementary mechanisms such as shielding, reachability, barrier, or projection methods. We do not benchmark hard pointwise safety methods, which target a stronger safety contract than SDH.    

    \section{Conclusion}

    SDH formulates survival-based constrained RL from heterogeneous per-step violations through a state-action continuation model. Its geometric-horizon formulation couples the survival-weighted objective to a single survival-probability constraint. For fixed multiplier $\eta$, the Lagrangian closes within the same SDH objective through the reward shift $r\mapsto r+\eta$, making $\eta$ a direct survival incentive. CaI resolves the post-violation semantics of policy regularization from first principles, enabling two new replay-compatible, off-policy algorithms: AS-SAC and VT-MPO.
    
    Violation-depth profiles make the scope relative to expected cumulative-cost CMDPs explicit: under the stated exponential-tail single-scale model, the SDH statistic determines additive cost, whereas profiles that depart from this form generally do not. We then reformulate the motivating high-dimensional biomechanical-control problem in Hyfydy as an SDH survival objective, where heterogeneous operating-limit violations indicate gait deterioration rather than budget expenditure. In Hyfydy, the reward $r=1-\mathrm{effort}$ already favors survival through its unit offset, allowing primal-only training with the fixed $\eta=0$ objective. VT-MPO matches EWA in peak gait realism, trains markedly more stably, and reaches its highest observed H2190 gait-match checkpoint at 21M versus 90M steps. Safety Gymnasium exhibits where SDH excels and where cumulative constraints are required.
    
    Future work includes learned continuation models, full primal-dual validation across survival-constraint levels, offline/pessimistic variants, and evaluation in safety-sensitive applications with complementary certified mechanisms.

    \clearpage
    \newpage

    \section{Acknowledgements} Nikola Milosevic and Nico Scherf are supported by BMFTR (Federal Ministry of Research, Technology and Space) through ACONITE (16IS22065) and the Center for Scalable Data Analytics and Artificial Intelligence (ScaDS.AI.) Leipzig.
	Georg Martius is a member of the Machine Learning Cluster of Excellence, EXC number 2064/1 – Project number 390727645.
	This work was supported by the ERC - 101045454 REAL-RL and the German Federal Ministry of Education and Research (BMBF): Tübingen AI Center, FKZ: 01IS18039A.
	Support from the International Max Planck Research School for Intelligent Systems (IMPRS-IS) for Leonard Franz is gratefully acknowledged.
	The authors thank Pierre Schumacher for valuable feedback on an earlier version of the manuscript.

    \section{Impact Statement}\label{sec:impact}
    This work contributes a technical reinforcement learning objective and accompanying algorithms for simulated constrained control. Its primary expected impact is within machine learning research, safe RL benchmarking, and musculoskeletal control. Deployment in real physical systems would require additional validation, conservative safety mechanisms, and domain-specific oversight beyond the guarantees provided here.
    
    \bibliographystyle{plain}
    \bibliography{ref}
	
    \clearpage
    \newpage
    
    \appendix
    \addcontentsline{toc}{section}{Appendix}

    \part{Appendix}

    \part{Algorithmic Details}
	\label{app:algo-details}

    \section{Overview}
    This section collects only the engineering/implementation points that are shared across methods or are not already spelled out in App.~\ref{app:vt-mpo} (VT-MPO details) or App.~\ref{app:as-sac-unified} (AS-SAC and uniform-prior structure).
	
	\paragraph{Unifying viewpoint (what changes and what does not).}
	Across all SDH-based variants, constraints enter learning solely through a continuation probability
	$\alpha(s,a)\in[0,1]$ computed from constraint signals, inducing shaped quantities
	\[
	\tilde r(s,a)=\alpha(s,a)\,(r(s,a)+\eta),
	\qquad
	\tilde\gamma(s,a)=\gamma\,\alpha(s,a),
	\]
	which appear \emph{only} in critic targets. The replay buffer retains the raw quantities needed to construct these targets, which are transformed on demand using the current $\lambda$ and $\eta$; target networks and actor-side updates (SAC or MPO) are otherwise unchanged.
	AS vs.\ VT is a semantic distinction of the CaI objective that affects how information costs are accumulated (main text),
	but the continuation model $\alpha$ and survival gate used for shaping are shared.
	
	\subsection{Environment termination and \texttt{done} handling}
	\label{app:done-handling}
	
	In our experiments, constraint events do \emph{not} reset the simulator; “termination” is objective-level through $\alpha(s,a)$.
	Episodes end only when the environment returns \texttt{done}.
	
	\paragraph{True terminals vs.\ time limits.}
	For critic bootstrapping we follow standard off-policy practice:
	\begin{itemize}[leftmargin=*,itemsep=0pt]
		\item for true terminal transitions, we disable bootstrapping with $(1-\texttt{done})$ (so $\tilde\gamma=(1-\texttt{done})\gamma\alpha$);
		\item for time-limit truncations, we allow bootstrapping via the usual timeout mask (absorbed into \texttt{done} in notation).
	\end{itemize}
	
	\subsection{Continuation models used in the experiments}
	\label{app:cont-models}
	
	The theory requires only $\alpha(s,a)\in[0,1]$. In practice, the choice of $\alpha$ is a modeling decision that controls how
	aggressively violations shorten the effective horizon and can materially affect off-policy stability. We use an \emph{exponential continuation model} throughout our experiments.
	Given per-step violation magnitudes $c_i(s,a)\ge 0$, we use the smooth attenuation
	\begin{equation}
		\alpha_{\mathrm{exp}}(s,a)
		:=
		\exp\!\Big(-\lambda \sum_{i\in I} c_i(s,a)\Big)\in(0,1],
		\label{eq:alpha_exp_app}
	\end{equation}
	which avoids discontinuities and yields stable critic targets under replay. This continuation is used for both AS-SAC and VT-MPO
	on Safety Gymnasium and for the Schumacher et al. style formulation on Hyfydy.

    \subsection{Practitioner's Guide}
	In practice, we recommend VT-MPO with exponential continuation (Exp) as the default configuration, because it gives the most reliable overall reward/cost trade-off: across Safety Gymnasium it is the more conservative and stable option, consistently achieving feasible or near-feasible behavior with good returns and lower variance, while AS-SAC is better viewed as the high-reward alternative when some additional instability or tighter use of the cost budget is acceptable. Concretely, AS-SAC attains the highest rewards on several easier locomotion-style tasks (e.g., Ant, HalfCheetah, HumanoidVelocity, RacecarCircle), but VT-MPO is stronger when robust feasibility matters more, including tasks such as Hopper and PointPush, and it degrades more gracefully under a single shared hyperparameter setting. 
    We suggest exponential continuation as the default because it is simple and empirically robust, with fixed-$\lambda$ results and no-schedule ablations showing that performance is largely unchanged once $\lambda$ is above a small threshold. CaT~\citep{chane2024cat} uses a different continuation construction for legged locomotion; we do not use that construction here. All experiments in this paper use the exponential continuation in Eq.~\ref{eq:alpha_exp_app}.

    \section{VT-MPO: MPO with a Survival-Shaped Critic}
	\label{app:vt-mpo}
	
	This section specifies the only deviation from standard MPO: we keep MPO’s \emph{actor-side} E/M updates unchanged, but replace the critic by the \emph{survival critic} induced by SDH (Sec.~\ref{sec:cai_sdh}). Concretely, constraints affect learning \emph{only} through critic targets via
	\[
	\tilde r(s,a)=\alpha(s,a)\,(r(s,a)+\eta),
	\qquad
	\tilde\gamma(s,a)=\gamma\,\alpha(s,a).
	\]
	All KL budgets, dual updates, and policy-fitting details are exactly as in MPO.
	
	\paragraph{What is different from original MPO?}
	Relative to MPO, VT-MPO changes only:
	\vspace{-2mm}
	\begin{enumerate}[leftmargin=*,itemsep=0pt]
		\item the \textbf{score} used in the E-step (use $Q_{\mathrm{surv}}$ instead of the standard return critic), and
		\item the \textbf{critic training target} (use shaped reward $\tilde r$ and variable discount $\tilde\gamma$).
	\end{enumerate}
	\vspace{-1mm}
	Everything else (E/M steps, KL constraints, trust region options, action sampling) is standard MPO.
	
	\subsection{Survival critic induced by SDH}
	\label{app:vt-mpo:surv-critic}
	
	Define the (unregularized) survival action-value function of a policy $\pi$ as
	\begin{equation}
		Q_{\mathrm{surv}}^{\pi}(s,a)
		\;\coloneqq\;
		\E_{\tau\sim\pi}\!\Big[\sum_{t\ge 0} u_t\,\tilde r(s_t,a_t)\,\Big|\,s_0=s,a_0=a\Big],
		\qquad
		u_t:=\prod_{k=0}^{t-1}\tilde\gamma(s_k,a_k),
		\label{eq:Qsurv_def_app_v2}
	\end{equation}
	which is the unique fixed point of the variable-discount Bellman recursion
	\begin{equation}
		Q_{\mathrm{surv}}^{\pi}(s,a)
		=
		\tilde r(s,a)
		+
		\tilde\gamma(s,a)\,
		\E_{s'\sim P(\cdot\mid s,a)}\Big[\E_{a'\sim\pi(\cdot\mid s')}\big[Q_{\mathrm{surv}}^{\pi}(s',a')\big]\Big].
		\label{eq:Qsurv_bellman_app_v2}
	\end{equation}
	This is the critic used by VT-MPO.
	
	\subsection{Policy improvement (unchanged MPO E/M steps)}
	\label{app:vt-mpo:policy-improvement}
	
	Fix an update state distribution $d$ (estimated from replay) and a behavior/reference policy $\pi_0$ (typically the current actor, as in MPO).
	Given a critic estimate $Q_{\mathrm{surv}}^{\pi_\theta}$, the nonparametric MPO improvement problem becomes
	\begin{equation}
		\max_{q}\;\; \E_{s\sim d}\Big[\E_{a\sim q(\cdot\mid s)} Q_{\mathrm{surv}}^{\pi_\theta}(s,a)\Big]
		\quad\text{s.t.}\quad
		\E_{s\sim d}\Big[\KL\!\big(q(\cdot\mid s)\,\|\,\pi_0(\cdot\mid s)\big)\Big]\le \varepsilon,
		\label{eq:vt_estep_app_v2}
	\end{equation}
	(with $\int q(\cdot\mid s)\,da=1$ for each $s$).
	The solution is the standard Boltzmann form
	\begin{equation}
		q^\star(a\mid s)\ \propto\ \pi_0(a\mid s)\, \exp\!\Big\{\beta^{-1}Q_{\mathrm{surv}}^{\pi_\theta}(s,a)\Big\},
		\label{eq:vt_boltzmann_app_v2}
	\end{equation}
	where the dual temperature $\beta>0$ is chosen to satisfy the KL budget (exactly as in MPO).
	The M-step then fits the parametric actor to $q^\star$ using the same objective/parameterization as MPO (e.g.\ weighted maximum-likelihood or reverse KL), optionally with an additional trust region to the previous actor.
	\emph{No actor-side changes are required beyond substituting $Q_{\mathrm{surv}}$ for the usual critic.}
	
	\subsection{Off-policy critic learning with variable discount}
	\label{app:vt-mpo:critic-learning}
	
	VT-MPO learns $Q_\phi\approx Q_{\mathrm{surv}}^{\pi_\theta}$ off-policy using replay and a target network $\phi^-$. The only modification to standard off-policy regression is that targets use $(\tilde r,\tilde\gamma)$.
	
	\paragraph{Virtual termination and \texttt{done} handling.}
	Constraint events do \emph{not} reset the simulator; ``termination'' is purely objective-level through $\alpha(s,a)$.
	Environment termination is handled normally via \texttt{done}: for true terminal transitions we disable bootstrapping by multiplying the bootstrap term by $(1-\texttt{done})$. (Time-limit truncations may still bootstrap, using the standard timeout mask.)
	
	\paragraph{TD($n$) replay.}
    The replay buffer retains the raw transitions needed to construct the $n$-step target. For a sampled sequence of length $m\le n$,
    \[
    \{(s_{t+j},a_{t+j},r_{t+j},c_{t+j},s_{t+j+1},d_{t+j})\}_{j=0}^{m-1},
    \]
    where terminal sequences are truncated, we compute at critic-update time, using the current $\lambda$ and $\eta$,
    \begin{equation}
        \alpha_{t+j}
        =
        \exp\!\Big\{-\lambda\sum_i c_{i,t+j}\Big\},
        \qquad
        \tilde r_{\eta,t+j}
        =
        \alpha_{t+j}(r_{t+j}+\eta),
        \qquad
        \tilde\gamma_{t+j}
        =
        (1-d_{t+j})\gamma\alpha_{t+j}.
    \end{equation}
    Define
    \begin{equation}
        R^{(m)}_{t,\tilde\gamma}
        \coloneqq
        \sum_{k=0}^{m-1}
        u_{t,k}\,\tilde r_{\eta,t+k},
        \qquad
        u_{t,k}
        :=
        \prod_{j=0}^{k-1}\tilde\gamma_{t+j},
        \qquad
        u_{t,m}
        :=
        \prod_{j=0}^{m-1}\tilde\gamma_{t+j}.
        \label{eq:compressed_nstep_app_v2}
    \end{equation}
    The TD($n$) target is
    \begin{equation}
        y_t
        =
        R^{(m)}_{t,\tilde\gamma}
        +
        u_{t,m}\,
        Q_{\phi^-}(s_{t+m},a'),
        \qquad
        a'\sim\pi_\theta(\cdot\mid s_{t+m}).
        \label{eq:tdn_target_app_v2}
    \end{equation}
	
	\paragraph{Remark (off-policy corrections).}
	Original MPO suggests sequence-based corrections (e.g.\ Retrace) for off-policy evaluation. We do not apply them. Empirically, TD($n$) with survival-shaped $(\tilde r,\tilde\gamma)$ was stable in our settings.

    \newpage
    
    \section{Algorithmic Descriptions of VT-MPO and AS-SAC}
    We present the VT-MPO and AS-SAC implementation-faithful pseudocode in Algorithms~\ref{alg:vt-mpo} and~\ref{alg:as-sac}, respectively. Given nonnegative per-constraint violation magnitudes $c_{i,t}$, both algorithms use the exponential continuation $\alpha_t = \exp\ \{-\lambda\sum_i c_{i,t}\}$, reward $\tilde r_{\eta,t} = \alpha_t(r_t+\eta)$, and discount $\tilde\gamma_t = \gamma\alpha_t$.
    \begin{algorithm}[H]
    	\caption{VT-MPO (termination=Virtual; actor=MPO, critic=unregularized-survival-TD($n$))}
    	\label{alg:vt-mpo}
    	\begin{algorithmic}[1]
    		\STATE Initialize actor $\pi_\theta$, prior $\pi_0$, critic $Q_\phi$, target critic $Q_{\phi^-}\!\leftarrow Q_\phi$, replay buffer $\mathcal B$
    		\STATE Initialize rolling window of length $n$ for raw transitions $(s,a,r,c,s',\texttt{done})$
    		\WHILE{training}
    		\STATE \textbf{Collect data:} run policy $\pi_0$ in the environment; observe $(s_t,a_t,r_t,c_t,s_{t+1},\texttt{done})$
    		\STATE Append $(s_t,a_t,r_t,c_t,s_{t+1},\texttt{done}_t)$ to the rolling window
            \STATE If the rolling window has length $n$, store the raw window in $\mathcal B$; at episode end also store the remaining shorter windows

            \STATE \textbf{Critic update:} sample a raw sequence of length $m\le n$ from $\mathcal B$. Using the current $\lambda$ and $\eta$:
            \FOR{$j=\{0,\ldots,m-1\}$}
                \STATE $\alpha_{t+j}=\exp\{-\lambda\sum_i c_{i,t+j}\}$, $\tilde r_{\eta,t+j}=\alpha_{t+j}(r_{t+j}+\eta)$, and $\tilde\gamma_{t+j}=(1-\texttt{done}_{t+j})\gamma\alpha_{t+j}$
            \ENDFOR
            \STATE Set
            $R^{(m)}_{t,\tilde\gamma} =\sum_{k=0}^{m-1} (\prod_{j=0}^{k-1}\tilde\gamma_{t+j})\tilde r_{\eta,t+k}$ and $u_{t,m}=\prod_{j=0}^{m-1}\tilde\gamma_{t+j}$
            \STATE Regress $Q_\phi(s_t,a_t)$ towards $y_t = R^{(m)}_{t,\tilde\gamma}+u_{t,m}Q_{\phi^-}(s_{t+m},a')$, where $a'\sim\pi_\theta(\cdot\mid s_{t+m})$
            
    		\STATE \textbf{Actor update (standard MPO):} perform MPO E/M steps using $Q_{\phi^-}$ as score
    		\STATE \hspace{1.5em} (Implementation: MPO regression may use stacked states $\{s_t\}\cup\{s_{t+n}\}$ for improved coverage.)
    		\STATE Periodically update target networks $\phi^- \leftarrow \phi$, and optionally a slow-moving actor
    		\ENDWHILE
    	\end{algorithmic}
    \end{algorithm}
    
    \begin{algorithm}[H]
    	\caption{AS-SAC (termination=AbsorbingState; actor=SAC, critic=regularized-survival-TD(0) )}
    	\label{alg:as-sac}
    	\begin{algorithmic}[1]
    		\STATE Initialize actor $\pi_\theta$ (tanh-Gaussian), critics $Q_{\phi_1},Q_{\phi_2}$ and targets $\phi^-_i\!\leftarrow\phi_i$; replay $\mathcal B$ stores $(s,a,r,c,s',\texttt{done})$
    		\STATE Fix a normalized uniform reference policy $\log \pi_0(a\mid s)=-\ell_c$ and set $b_{\mathrm{AS}}:=\ell_c$
            \STATE Initialize temperature $\kappa>0$ (fixed or autotuned), SAC entropy target $\bar{\mathcal H}$ and parameter $\beta_{\kappa}$
    		\WHILE{training}
            \STATE \textbf{Collect transition:} choose $a_t$ randomly for $t<t_{\mathrm{learn}}$, otherwise sample $a_t\sim\pi_\theta(\cdot\mid s_t)$; execute $a_t$ and observe $(r_t,c_t,s_{t+1},\texttt{done}_t)$
            \STATE Store $(s_t,a_t,r_t,c_t,s_{t+1},\texttt{done}_t)$ in $\mathcal B$
            \IF{$t\ge t_{\text{learn}}$}
              \STATE \textbf{Critic:} sample batch from $\mathcal B$; using the current $\lambda$ and $\eta$, compute $\alpha=\exp\{-\lambda\sum_i c_i\}$, $\tilde r_\eta=\alpha(r+\eta)$, and $\tilde\gamma=(1-\texttt{done})\gamma\alpha$; sample $(a',\log\pi_\theta(a'|s'))$;
    		\STATE set $\widehat V^{-}(s')=\min_i Q_{\phi_i^-}(s',a')-\operatorname{sg}(\kappa)(\log \pi_\theta(a'|s') + b_{\mathrm{AS}})$, where $a'\sim\pi_\theta(\cdot|s')$;
            \STATE set $y=\tilde r_\eta+\tilde\gamma\,\widehat V^{-}(s')$.
    		\STATE Update $\phi_i$ by MSE to target $y$ (twin critics)
    		\STATE \textbf{Periodically update}
    		\STATE \quad\textbf{Actor:} sample $a_\pi\sim\pi_\theta(\cdot|s)$, optimize $\max_\theta \mathbb E[\min_i Q_{\phi_i}(s,a_\pi)-\operatorname{sg}(\kappa)\log\pi_\theta(a_\pi|s)]$
            \STATE \quad\textbf{if} \texttt{autotune}, estimate $\widehat{\mathcal H}$ on sampled replay batch; update $\kappa\leftarrow[\kappa+\beta_\kappa(\bar{\mathcal H}-\widehat{\mathcal H})]_+$
            \STATE \quad update target critics: $\phi^-_i \leftarrow \tau\phi_i+(1-\tau)\phi^-_i$
    		\ENDIF
    		\ENDWHILE
    	\end{algorithmic}
    \end{algorithm}

	\clearpage
	\newpage

    \part{Experimental Results}
    \parttoc 
	
    \section{Hyfydy Experiments}
    \label{app:hyfydy}

    This section details the Hyfydy experiments: full-formulation results achieving state-of-the-art realism, followed by diagnostic variants.
	
    \subsection{Environments, metrics, and evaluation protocol}
    \label{app:hyfydy-models-metrics}
	
    \paragraph{Musculoskeletal environments.} We evaluate on three humanoid musculoskeletal models simulated with Scone-Hyfydy~\citep{geijtenbeek2019scone,Geijtenbeek2021Hyfydy}. The models increase in dimensionality and actuation redundancy from H0918 to H2190, with H2190 being the hardest setting utilizing 90 muscle actuators. The models are among the most realistic options for studying gait with reinforcement learning, as they provide high simulation speed compared to e.g. OpenSim \citep{opensim}, while being significantly more realistic than, e.g., MuJoCo \citep{todorov2012mujoco,schumacher2025emergancenatural}. See \cref{tab:hyfydy-models-app} for an overview of key model properties. We generally use the \texttt{v1} version of the sconegym \citep{geijtenbeek2023sconegym} environments.
	
    \begin{table}[htb]
        \centering
        \caption{Hyfydy model overview. The models we use for our experiments span different complexity levels, ranging from 2D restricted models with only 9 muscles and neglecting trunk muscles, to full range of motion with 90 muscles and trunk muscles being modeled.}
        \label{tab:hyfydy-models-app}
        \begin{tabular}{lrrrcc}
            \toprule
            Model & Motion restriction & DOFs & Muscle count & Trunk muscles & Appearance\\
            \midrule
            H0918 & Sagittal plane only (2D)   & 9  & 18 & \ding{55} & \includegraphics[width=1cm,valign=m]{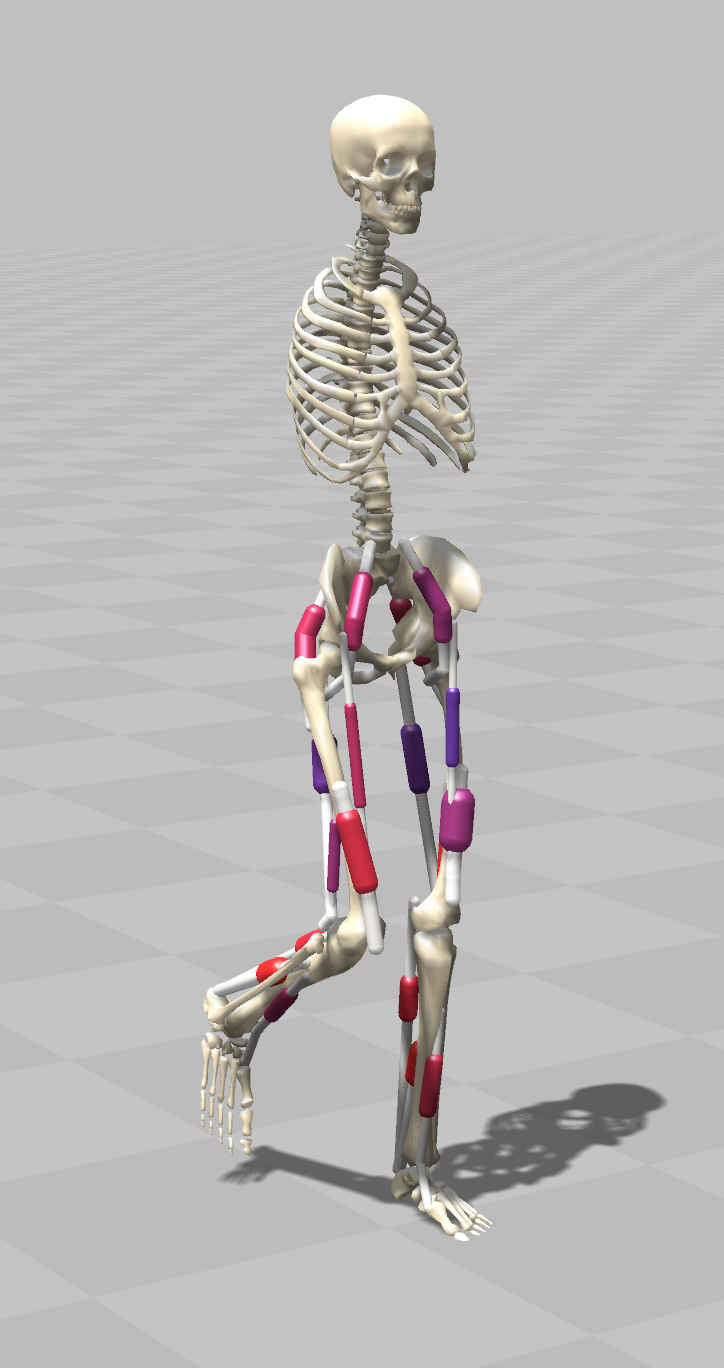}\\
            H1622 & None & 16 & 22 & \ding{55} & \includegraphics[width=1cm,valign=m]{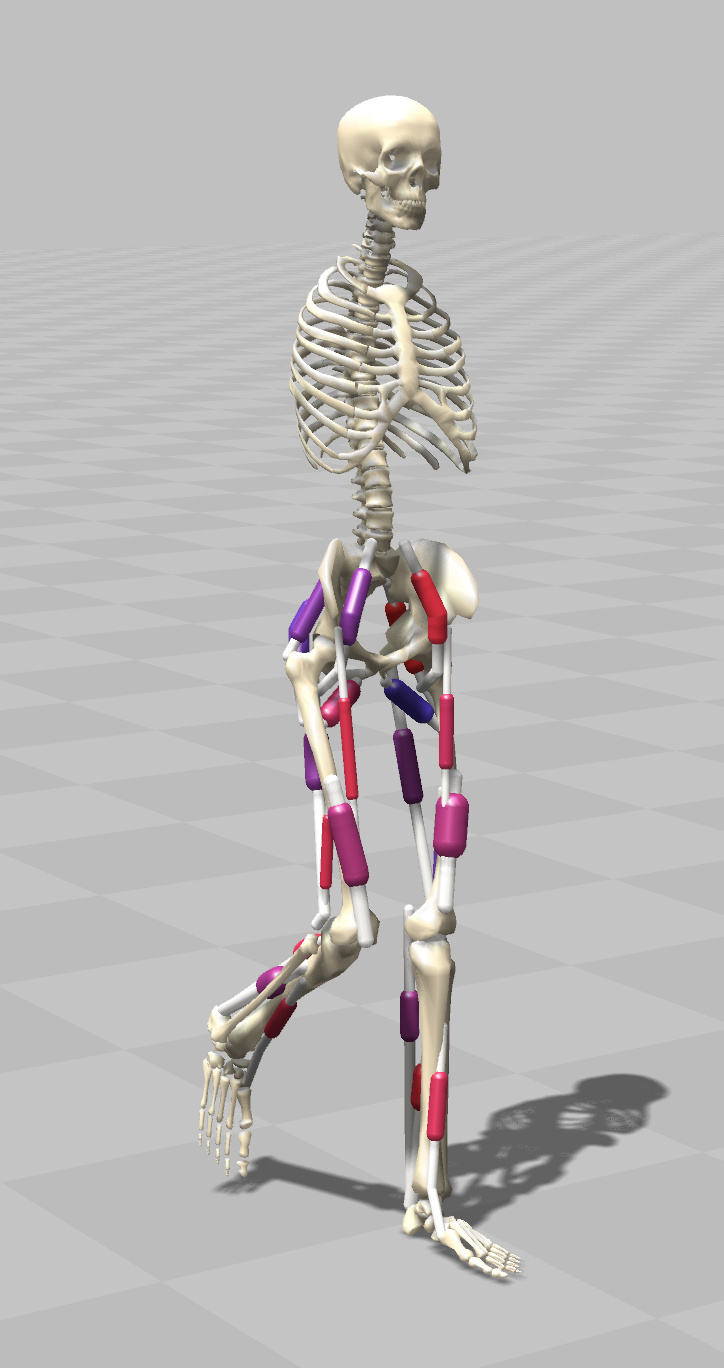}\\
            H2190 & None & 21 & 90 & \ding{51} & \includegraphics[width=1cm,valign=m]{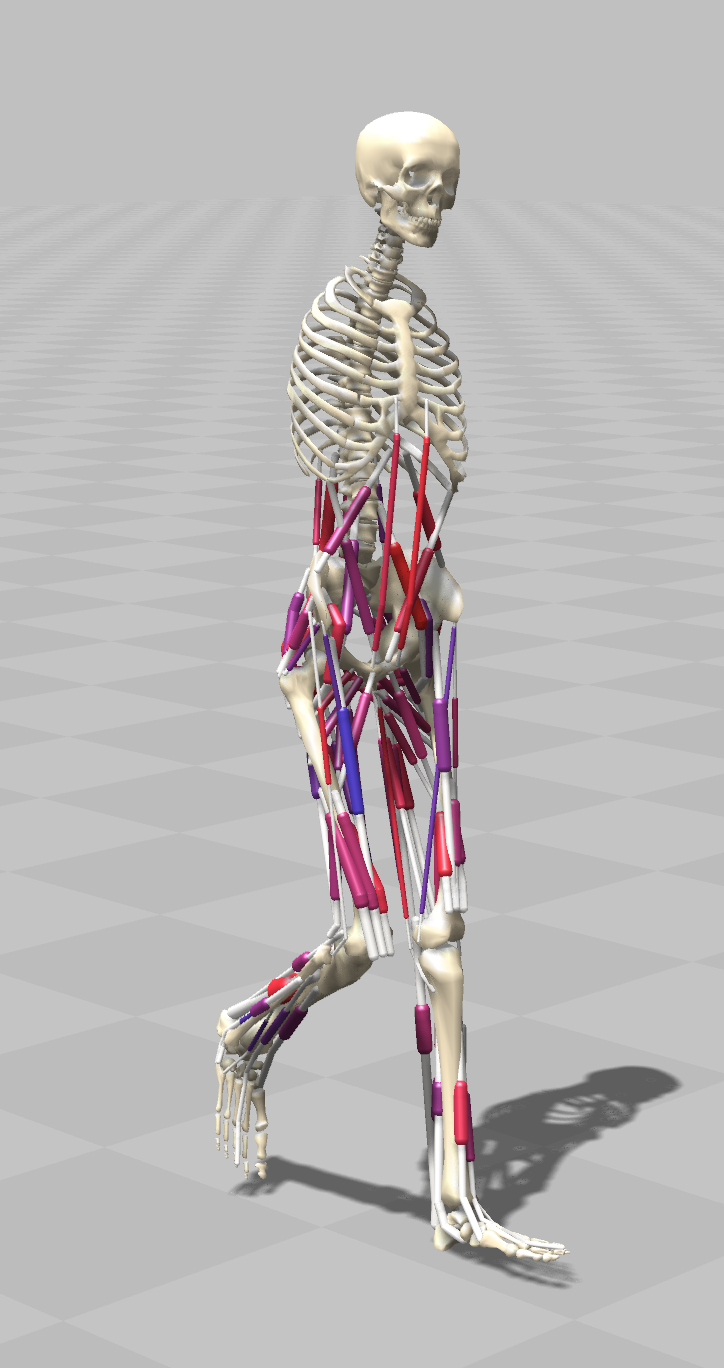}\\
            \bottomrule
        \end{tabular}
    \end{table}
	
    \paragraph{Metrics for gait characterization.}
    We consider the following metrics in our humanoid gait experiments:
    \begin{itemize}[leftmargin=*,itemsep=2pt,topsep=2pt]
        \item \textbf{Forward velocity.} Non-zero velocity is defining for the state of locomotion and thus essential for judging gait dynamics. When we speak about velocity, we refer to the $x$-component of the center of mass velocity of the model throughout this document. Note that we are targeting natural \emph{walking}, not running, in this work. This means that our goal is \emph{not} to maximize velocity, but rather to bring it to a range that is natural for human walking ($\approx \SI{1.2}{\meter\per\second}$ \citep{Ralston1958EnergyspeedRA}).
        \item \textbf{Effort.} We define $p$-effort (with $p \geq 2$) to be the muscle averaged $p$-th power of the muscle activation $a$:
        \begin{equation}
            \operatorname{effort}^p = \frac{1}{|\mathcal M|}\sum_{m \in \mathcal M} a_m^p,
            \label{eq:effort}
        \end{equation}
        where $\mathcal M$ is the model's set of muscles. The metric has a tight relationship to the metabolic energy consumption, but is notably \emph{not} equivalent. Effort penalizes high activations more than metabolic consumption rates would suggest, and reflects the disproportionate preference of humans for low activations to prevent muscle fatigue \citep{ackermann_optimality_2010}. Effort definitions often vary in which exponent of the effort ($p=2,3,4,\dots$) they use, and it is debated which exponent provides the best optimization target. In order to stay consistent with Schumacher et al. \citep{schumacher2025emergancenatural}, who also optimize 3-effort but report 2-effort, we report effort with $p=2$ in our visualizations.
        \item \textbf{Episode length.} Number of simulated control steps (each with a time duration of \SI{0.025}{\second}) before environment termination or timeout (1000 steps). Higher values indicate stable sustained gait. We consider a gait to be stable, if the average attained episode length is higher than 900 steps (\SI{22.5}{\second}).
        \item \textbf{Ground Reaction Forces (GRF).} The force the ground exerts back on the foot when the foot is pushed into the ground. Humans choose motion patterns that lack peaks in ground reaction forces, making it a key metric for gait naturalness \citep{veerkamp_evaluating_2021,falisse_rapid_2019}. Gait simulations, in contrast, tend to produce unrealistically high ground reaction forces \citep{schumacher2025emergancenatural, bunz_spinal_2025}. One hypothesis for this is that soft tissues, which are mostly not modeled in musculoskeletal simulations, help dampen peaks in the force signal \citep{zelik_human_2010}, making peaks more likely in models which lack them. Notably, our models lack soft tissues entirely.
        \item \textbf{Joint limit violation.} The total torque exerted by joint limits. Model joints are constrained to anatomically plausible ranges of motion; when a joint reaches its limit, the simulator applies a restoring force to prevent further excursion. Humans select movement patterns that avoid forceful contact with their joint limits, as repeated or excessive limit forces risk joint damage.
        \item \textbf{Human gait alignment.} We report the gait match percentage as calculated in SCONE Studio's gait analysis \citep{geijtenbeek2019scone}. For each evaluated joint, detected gait cycles are time-normalized and averaged into a single mean trajectory, which is then compared against a reference norm band derived from healthy human walking data (see green band in \cref{fig:cost_realism}). At each sample point, the excess outside the norm envelope is computed and normalized by the band width; the per-joint score is $100 \times \max(1 - \bar{e},\, 0)$, where $\bar{e}$ denotes the mean normalized excess. The overall alignment score is the arithmetic mean across all evaluated joints (Hip, Knee, Ankle) and the ground reaction force, yielding a value in $[0, 100]$ where higher values indicate closer agreement with human gait kinematics. We report values from our own Python reimplementation of the score. Numerical differences to SCONE Studio's implementation arise primarily from gait cycle detection and time discretization (we use uniform step sizes, whereas SCONE uses variable-step simulation time points). Per-joint scores can differ by up to ${\sim}15\%$, though overall alignment scores are consistent in magnitude. We note that the metric is not highly discriminative — it is best understood as a coarse indicator of gait realism rather than a precise ranking tool.
    \end{itemize}
	
    \paragraph{Training and evaluation protocol.}
    We train each method for 100 million environment steps and report mean and standard deviation over three random seeds (in selected plots mean, min and max). Evaluation is performed every 200000 steps using 10 deterministic rollouts. Our algorithms were implemented within the DEP-RL code base \cite{deprl2023github}, thus sharing many structural components with EWA, which is also part of the code base, and thus ensuring a maximally fair comparison. EWA runs were performed with the original implementation of Schumacher et al.\@. All methods use the same simulator settings, time limits, observation preprocessing, action bounds, and DEP exploration settings. The exact hyperparameters used are listed in \cref{tab:hyfydy-vtmpo-hparams-app}.
    
    \subsection{Schumacher-style constrained gait formulation}
    \label{app:hyfydy-full-formulation}
    We cast the problem solved by EWA \citep{schumacher2025emergancenatural} into our instantaneous constraint setting. We first recapitulate Schumacher et al.'s reward and the EWA reweighting procedure, which provides relevant context for RL based gait learning and motivates our continuation and reward design.

    \paragraph{Schumacher et al.'s reward function.}
    Schumacher et al.'s reward is used in the standard MDP setting with cumulative discounted returns. It consists of three components, each targeting a different aspect of natural gait:
    \begin{align}
        r = r_\text{vel} - c_\text{effort} - c_\text{pain}
        \label{eq:nat-rob-walk-reward}
    \end{align}
    where 
    \begin{align}
        r_\text{vel} &= \begin{cases}
            \exp\left[-(v_\text{com}-v_\text{target})^2\right] & \text{if} \quad v < v_\text{target} \\
            1 & \text{otherwise}
        \end{cases}
    \end{align}
    \begin{align*}
        c_\text{effort} = \alpha(t)\operatorname{effort}^3 + w_1 \sum_{m\in\mathcal M}(u - u_\text{prev})^2  + w_2 N_\text{active} \qquad
        c_\text{pain} = w_3 \sum_{j \in \mathcal J} \tau_j^\text{lim} + w_4 \sum_{l \in \mathcal L} F_l^\text{GRF}
    \end{align*}
    Here, $r_\text{vel}$ is the velocity objective, $c_\text{effort}$ combines several effort related objectives (3-effort, activation smoothness and activation sparseness), and $c_\text{pain}$ targets joint limit violations and ground reaction force, labeled jointly as \enquote{pain}. Individual symbols are defined in \cref{tab:reward-symbol-definitions}. Note that the weights $w_1$ to $w_4$ were found via extensive hyperparameter search in the original publication, targeting maximal human gait alignment.

    \begin{table}[htb]
        \centering
        \caption{Symbol definitions for reward and constraint specifications.}
        \label{tab:reward-symbol-definitions}
        \begin{tabular}{cp{10cm}}
            \toprule
            \textbf{Symbol} & \textbf{Definition} \\
            \midrule
            $v_\text{com}$ & $x$-velocity (forward direction) of the center of mass. \\
            $v_\text{target}$ & Target $x$-velocity. \\
            $\alpha(t)$ & Dynamically reweighted effort coefficient. \\
            $\operatorname{effort}^3$ & 3-effort (see \cref{eq:effort}) with third power activation exponent. \\
            $w_1, w_2, w_3, w_4$ & Objective weights, each representing a hyperparameter of the reward formulation. \\
            $\mathcal M, \mathcal J, \mathcal L$ & The sets of muscles, joints and legs of a model. \\
            $\mathcal H, \mathcal K, \mathcal A$ & Sets of hip, knee and ankle joints (containing the left and right joint each). \\
            $u_m, u_{m,\text{prev}}$ & Current and previous muscle excitation (policy output signal) for muscle $m$. This is \emph{not} the activation $a$, which is distinct from the control signal $u$ as excitations first get translated into activations via the activation dynamics of the model. \\
            $N_\text{active}$ & Number of currently active muscles. A muscle is considered to be active if its activation $a_m$ is greater than 0.15. \\
            $\tau_j^\text{lim}$ & Joint torque produced by model joint limit constraints of joint $j$. \\
            $F_l^\text{GRF}$ & Ground reaction force produced by leg $l$ measured in terms of body weight: $F/(M\cdot g)$. \\
            \bottomrule
        \end{tabular}
    \end{table}

    \paragraph{Effort Weight Adaptation (EWA).}
    EWA is a heuristic, adaptive reward scheduling procedure that resolves the tension between effort minimization and other gait objectives. Since effort optimality is trivially attained by a non-acting policy, it conflicts with every other objective. EWA balances these by dynamically adjusting the effort weight $\alpha(t)$ based on a moving average of the task return $R_{\text{mean}}$ (\cref{alg:ewa}).

    \begin{algorithm}[htb]
        \caption{Effort Weight Adaptation}
        \label{alg:ewa}
        \renewcommand{\algorithmiccomment}[1]{\hfill // #1}
        \begin{algorithmic}
            \STATE {\bfseries Input:} threshold $\theta$, smoothing $\beta$, adaptation rate $\Delta \alpha$, decay rate $\lambda$
            \STATE $R_{\text{mean}} \gets 0, \alpha \gets 0, c_{\text{mean}} \gets 0$
            \FOR{episodes $t = 0, 1, 2, \dots$}
            \STATE $R \gets \texttt{train\_episode()}$ \COMMENT{Train and get return $R$}
            \STATE $R_{\text{mean}} \gets \beta R_{\text{mean}} + (1 - \beta)R$ 
            \IF{$R_{\text{mean}} > \theta$ \AND $c_{\text{mean}} < 0.5$}
            \STATE $\Delta \alpha \gets \lambda \Delta \alpha$ \COMMENT{Convergence decay}
            \ELSIF{$R_{\text{mean}} > \theta$ \AND $c_{\text{mean}} > 0.5$}
            \STATE $\alpha(t+1) \gets \alpha(t) + \Delta \alpha$ \COMMENT{Increase effort pressure}
            \ELSE
            \STATE $\alpha(t+1) \gets \alpha(t) - \Delta \alpha$ \COMMENT{Relax effort pressure}
            \ENDIF
            \STATE $c_{\text{target}} \gets \mathbb{I}(R_{\text{mean}} > \theta)$
            \STATE $c_{\text{mean}} \gets \beta c_{\text{mean}} + (1 - \beta)c_{\text{target}}$
            \ENDFOR
        \end{algorithmic}
    \end{algorithm}

    We categorize EWA informally as being a \enquote{Lagrange style} constrained RL approach. This is because the return-based update logic of a sub-objective weight is reminiscent of the dual update in Lagrangian constrained RL in the CMDP setting. Notably, the adaptation logic of the \enquote{Lagrange multiplier} $\alpha(t)$ is more complicated than vanilla Lagrange dual updates, but modifications to Lagrange multiplier updates are commonly found in constrained RL literature, for example in PID Lagrangian RL \citep{stooke2020responsive}, where they smooth training dynamics. We interpret the logic in \cref{alg:ewa} as serving a similar purpose.

    \paragraph{Interpretation of Schumacher et al.'s gait learning procedure.} The structural similarity of EWA to primal-dual optimization suggests the following interpretation: EWA implicitly solves a \emph{constrained problem} where pain, velocity and non-$p$-effort terms form the \emph{constraint}, and $p$-effort is the \emph{objective}. In other words, behavior is constrained to resemble a natural gait, and the optimizer seeks the lowest-effort solution within this feasible set.

    \paragraph{SDH reformulation.} We now cast Schumacher et al.'s procedure into our SDH framework. According to our interpretation, effort must be the objective, which is why we choose the following reward function:
    \begin{equation}
        r = 1 - \frac{1}{|\mathcal M|}\sum_{m \in \mathcal M} \frac{a_m^3}{a_\text{max}^3},
        \label{eq:hyfydy-reward}
    \end{equation}
    where $a_\text{max}$ is the maximal possible activation, enforced on the model side to be $0.5$. Note that this ensures $r \in [0, 1]$, with 0 being achieved with maximal activation and 1 for zero activation. Other gait objectives are encoded as constraints:
    
    \begin{table}[ht]
        \centering
        \caption{Hyfydy constraint violation definitions with thresholds and normalizations. See \cref{tab:reward-symbol-definitions} for definitions of the individual symbols. Due to some models being able to control the lumbar torque, while others do not (see \cref{tab:hyfydy-models-app}), we use different limit torques depending on run and model. For H2190 we always use \SI{10}{\newton\meter}. For H1622 and H0918 we used \SI{80}{\newton\meter}, with the exception for our highest gait match H1622 run, which after tuning resulted in a \SI{40}{\newton\meter} threshold, but showed slightly diminished sample efficiency.}
        \label{tab:hyfydy-full-constraints-app}
        \begin{tabular}{llll}
            \toprule
            Constraint & Violation $c_i(s,a)$ & Threshold & Normalization\\\midrule 
            Minimum velocity & $\left[v_\text{min} - v\right]_+$ & \SI{1.085}{\meter\per\second} & $v_\text{min}$\\
            Maximum velocity & $\left[v - v_\text{max}\right]_+$ & \SI{1.315}{\meter\per\second} & $v_\text{max}$\\
            GRF & $\left[\max_{l \in \mathcal L} F_l^\text{GRF} - F_\text{max}\right]_+$ & $1.3$ & $1.0$\\
            Muscle count & $\left[N_\text{active} / |\mathcal M| - p_\text{max}\right]_+$ &  0.33 & $1.0$ \\
            Hip violation & $\left[\max_{h \in \mathcal H} \tau_{\text{hip},h}^\text{lim} - \tau^\text{lim}_{\text{hip},\text{max}}\right]_+$ & \SI{10}{\newton\meter}& $\tau^\text{lim}_{\text{hip},\text{max}}$ \\
            Knee violation & $\left[\max_{k \in \mathcal K} \tau_{\text{knee},k}^\text{lim} - \tau^\text{lim}_{\text{knee},\text{max}}\right]_+$ & \SI{80}{\newton\meter}& $\tau^\text{lim}_{\text{knee},\text{max}}$ \\
            Ankle violation & $\left[\max_{a \in \mathcal A} \tau_{\text{ankle},a}^\text{lim} - \tau^\text{lim}_{\text{ankle},\text{max}}\right]_+$ & \SI{20}{\newton\meter}& $\tau^\text{lim}_{\text{ankle},\text{max}}$ \\
            Lumbar torque & $\left[\tau_{\text{lumbar}}^\text{lim} - \tau^\text{lim}_{\text{lumbar},\text{max}}\right]_+$ & $\{10, 40, 80\}$\SI{}{\newton\meter}& $\tau^\text{lim}_{\text{lumbar},\text{max}}$ \\
            \bottomrule
        \end{tabular}
    \end{table}
    
    \paragraph{Constraint terms.}
    Let $c_i(s,a)\ge0$ denote nonnegative violation magnitudes. We map each of the non $p$-effort objectives to one or multiple of these violation magnitudes, summarized in \cref{tab:hyfydy-full-constraints-app}.
    
    Similarly to Schumacher et al. \citep{schumacher2025emergancenatural}, we performed an extensive hyperparameter search to find constraint thresholds which yield the highest human match percentage. Our problem formulation is the most sensitive to the knee torque limit. H1622 is uniquely sensitive to the lumbar torque limit. Unlike reward-based formulations, where penalty terms implicitly adapt to model morphology by summing over existing joints, explicit constraint thresholds must reflect model-specific biomechanics — e.g., models with active trunk muscles require fundamentally different lumbar torque limits than those without. The resulting thresholds are, however, physically interpretable (e.g., Nm of joint torque) rather than abstract reward weights.
    
    \paragraph{Continuation construction.}
    Given the per-constraint violations, we construct the SDH continuation probabilities $\alpha_i$, using the following exponential continuation design:
    \begin{equation}
        \label{eq:hyfydy-alpha-exp-perconstraint-app}
        \alpha_i(s,a)=\exp(-\lambda\,\bar c_i(s,a)),
    \end{equation}
    where $\bar c_i$ is the normalized violation and $\lambda$ is a hyperparameter. We explore the influence of $\lambda$ on our solutions in \cref{par:continuation}. As a default we choose $\lambda = 0.2$, which works reliably across models and experiments. The individual continuations are aggregated via a product:
    \begin{equation}
        \label{eq:hyfydy-alpha-aggregate-app}
        \alpha(s,a)
        =\prod_{i=1}^m\alpha_i(s,a)
        =\exp\!\left(-\lambda\sum_i \bar c_i(s,a)\right),
    \end{equation}
    allowing us to simply sum individual violations, to obtain a global one.
    
    \paragraph{Critic target.}
    For every replay transition, the critic target uses
    \begin{equation}
        \tilde r_t=\alpha(s_t,a_t)r(s_t,a_t),
        \qquad
        \tilde\gamma_t=(1-d_t)\gamma\alpha(s_t,a_t),
        \label{eq:hyfydy-shaped-target-app}
    \end{equation}
    with standard timeout handling for time-limit truncations. VT-MPO then uses the survival critic and the standard MPO E/M update, as described in \cref{app:vt-mpo}. We choose VT-MPO due to the higher stability observed in our ablations (\cref{app:hyfydy-diagnostics}).

    \subsection{Constrained gait learning results}
    \label{app:hyfydy-full-results}

    We now analyze the results of the Schumacher-style constrained gait formulation, focusing on constraint satisfaction, gait realism, and comparison with EWA.

    \paragraph{Constraint satisfaction and objective maximization.}
    \begin{figure}[htb]
        \centering
        \includegraphics[width=\linewidth]{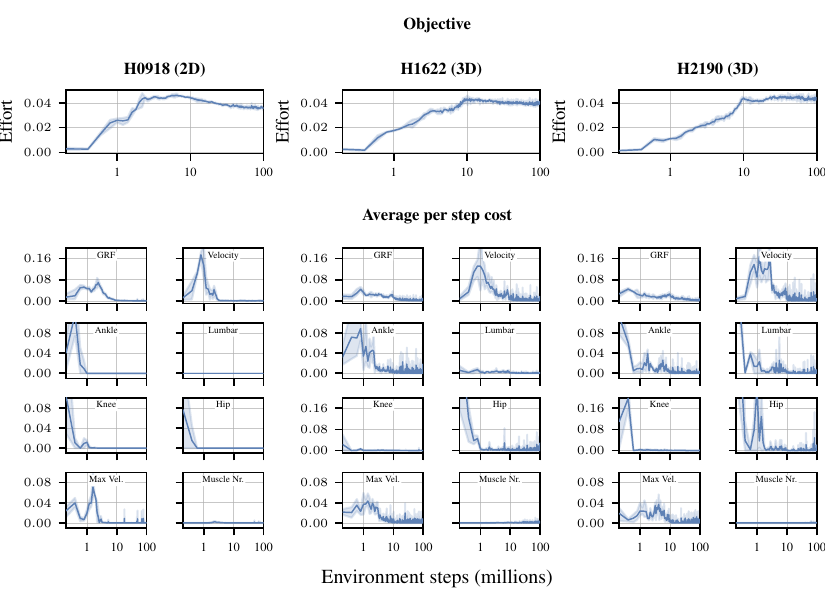}
        \caption{
            Objective and costs over the training for all three models. The plot shows mean and standard deviation over three seeds. We plot the $x$-axis logarithmically to emphasize early training dynamics, as later during the training costs and objective remain relatively constant. We see that the effort objective converges to low values, and the costs are reduced to near zero.
        }
        \label{fig:hyfydy-full-training-app}
    \end{figure}
    
    The training dynamics of selected gait metrics are visualized in \cref{fig:cost_realism}$\mathbf{a}$) and compared with EWA. The most notable features of our method are the \emph{speed} and \emph{stability} with which we converge to a solution. Effort reaches comparable results on the smaller models. On H2190 we see a slight gap in effort minimization. A possible explanation is that the solution EWA recovers is infeasible within our problem formulation — \cref{fig:cost_realism} shows that the EWA policy violates our constraints. There might not exist a feasible policy with equal effort as the EWA solution. 

    The sample efficiency curves of all cost and reward terms are visualized in \cref{fig:hyfydy-full-training-app}. The curves show that average per-step violations quickly converge to near zero across all models, while the effort reaches a plateau at a low effort level.
    
    \paragraph{Human gait alignment.}
    Both EWA and SDH aim to produce human-like gait without demonstrations. We evaluate gait realism using gait-match percentage. Every 200k environment steps, each of the three training seeds is evaluated using ten deterministic rollouts, yielding one seed-level mean gait-match percentage. These values are averaged across seeds at each evaluation step. Separately for each method, the best-observed checkpoint is the evaluation step with the highest across-seed mean. \Cref{fig:hyfydy-gait-match} shows the full training trajectories, while \cref{tab:gait_match} reports the mean $\pm$ standard deviation across seeds at the selected checkpoint.
    
    \begin{figure}[htb]
        \centering
        \includegraphics[width=\linewidth]{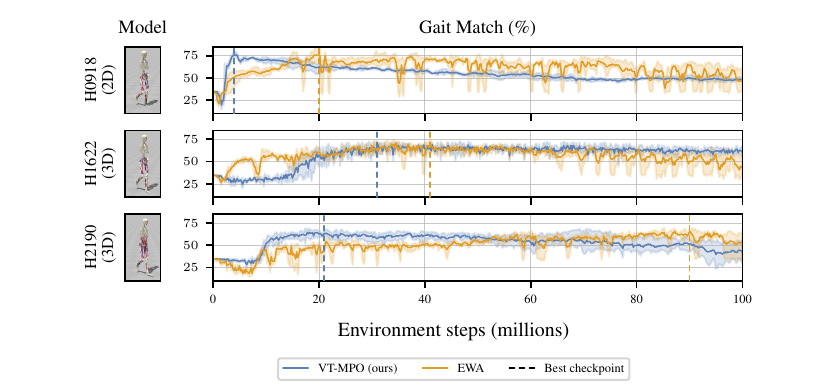}
        \caption{
        \textbf{Gait match over training.}
        VT-MPO and EWA on the three Hyfydy models. Dashed vertical lines mark the best-observed checkpoint selected separately for each method according to the protocol described above. Gait match is an external evaluation metric and is not optimized by either method.
        }
        \label{fig:hyfydy-gait-match}
    \end{figure}
    
    \begin{table}[htb]
        \centering
        \caption{
        \textbf{Best-observed gait match and corresponding checkpoint.}
        Values are mean $\pm$ standard deviation across three seeds at the selected evaluation step.
        }
        \label{tab:gait_match}
        \begin{tabular}{l|cc|cc}
            \toprule
            & \multicolumn{2}{c|}{VT-MPO (ours)} & \multicolumn{2}{c}{EWA} \\
            \cmidrule(lr){2-3} \cmidrule(lr){4-5}
            Model & Gait Match (\%) & Checkpoint (M steps)
                  & Gait Match (\%) & Checkpoint (M steps) \\
            \midrule
            H0918 (2D) & $75.9 \pm 2.5$ & 4  & $75.8 \pm 7.1$ & 20 \\
            H1622 (3D) & $70.1 \pm 1.7$ & 31 & $69.3 \pm 4.0$ & 41 \\
            H2190 (3D) & $62.9 \pm 2.1$ & 21 & $64.6 \pm 4.7$ & 90 \\
            \bottomrule
        \end{tabular}
    \end{table}
    
    The best-observed gait-match values are similar across the two methods on all three models, whereas the selected VT-MPO checkpoint occurs earlier in each case. The largest separation occurs on H2190, at 21M environment steps for VT-MPO versus 90M for EWA. This timing difference is descriptive of the observed trajectories and should not be interpreted as an estimate of convergence time or sample complexity. Because gait match is not part of either training objective, its highest observed value need not occur late in training, consistent with the decline after the selected checkpoints in \cref{fig:hyfydy-gait-match}.

    \paragraph{Full-formulation AS-SAC comparison.}\label{par:hyfydy-full-as-sac-app}
    To test whether AS-SAC provides a viable alternative under the same full Schumacher-style reformulation, we additionally evaluate fixed-$\kappa$ AS-SAC using the complete set of local operating-limit violation signals from \cref{tab:hyfydy-full-constraints-app}. This experiment is distinct from the velocity-only diagnostic in \cref{app:hyfydy-velocity-only}.
    As shown in \cref{fig:hyfydy-full-as-sac-app}, AS-SAC can attain higher human-match values on successful runs, most clearly on H0918 where it approaches $90\%$.
    However, its training is more seed-sensitive, with greater variability in gait acquisition and episode length, particularly on H2190. 
    VT-MPO therefore remains our default for the full-formulation comparison because it provides the more robust optimization behavior, while AS-SAC is a viable but more aggressive alternative.
    
    \begin{figure}[htb]
        \centering
        \includegraphics[width=\linewidth]{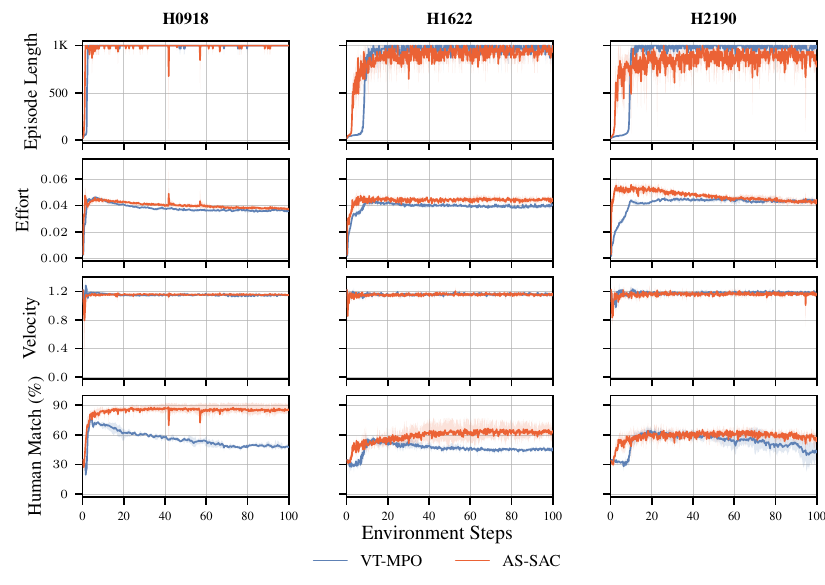}
        \caption{
            \textbf{Full-formulation AS-SAC versus VT-MPO on Hyfydy.}
            Training behavior under the same Schumacher-style reformulation and local operating-limit violation signals used in the main Hyfydy experiment.
            AS-SAC can attain higher human-match values on successful runs, especially on H0918, but exhibits greater seed sensitivity; VT-MPO is more robust on the H2190 model.
        }
        \label{fig:hyfydy-full-as-sac-app}
    \end{figure}
    
    \paragraph{Interpretation of gait cycle curves in \cref{fig:cost_realism}$\mathbf{b}$).}
    \label{par:gait_cycle_intr}
    We consider knee and hip kinematics to be equally well matched by both SDH and EWA on all models. For the GRF and ankle dynamics, we see a different picture: EWA generally possesses better ankle kinematics, while the ground reaction forces are more realistic for SDH. We hypothesize that there is a trade-off between realistic ankle dynamics, and realistic ground reaction force patterns in the considered models, possibly representing a model limitation.

    As we constrain the ground reaction force for SDH, we hypothesize that the only feasible gait remains to be a toe first ground contact, which does not constitute a proper roll-off. This reduces the realism of the ankle kinematics. As discussed previously, a possible explanation is the missing soft tissue in the models, especially in the foot \citep{dhondt2024DynamicFootModel,buchmann2024EffectIncludingMobile}.

    \paragraph{Visual evaluation.} We provide comparative videos on our project page, and static rollout visualizations in \cref{fig:hyfydy-full-rollouts-app}. Both EWA and SDH yield visually comparable policies. For the H2190 model, we observe slightly less postural stability for EWA.
    
    \begin{figure}[htb]
        \centering
        \includegraphics[width=\linewidth]{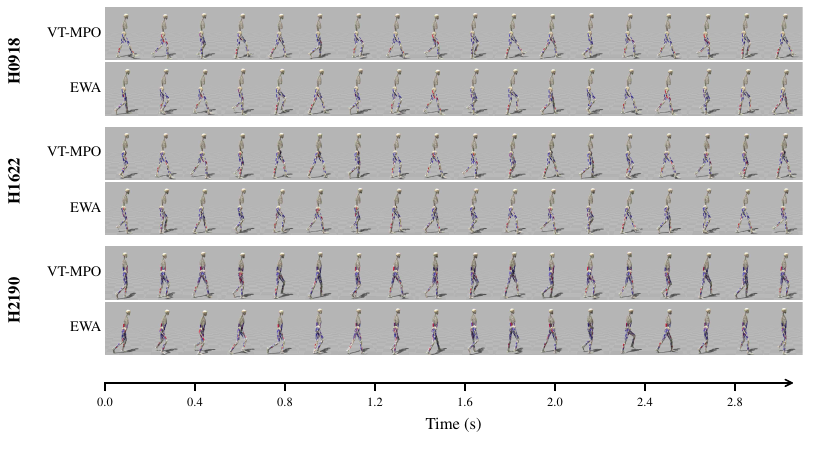}
        \caption{Representative full-formulation rollouts for H0918, H1622, and H2190 of SDH and EWA. Videos are available on the project page. We see no significant gap in visual realism between both approaches.}
        \label{fig:hyfydy-full-rollouts-app}
    \end{figure}

    \subsubsection{Survival-only diagnostic of the fixed $\eta=0$ objective}
    \label{app:hyfydy-survival-endpoint}
    
    For the Hyfydy reward $r=1-f$, the fixed $\eta=0$ SDH objective already contains the survival-probability term:
    \[
        (1-\gamma)\mu^\pi_\alpha(1-f)
        =
        p_{\mathrm{surv}}(\pi)
        -
        (1-\gamma)\mu^\pi_\alpha(f).
    \]
    We therefore ask whether removing the effort term entirely reveals an operating point with materially higher survival.  On H1622, we train two otherwise identical VT-MPO agents using
    \[
        Q_{\mathrm{actor}}
        =
        w_r Q_r + (1-w_r)Q_1,
        \qquad w_r\in\{0,1\},
    \]
    where $w_r=1$ is the task objective used by the main method ($Q_r$, $\eta=0$), whereas $w_r=0$ performs policy improvement using only the survival-mass critic $Q_1$.  Both critics are trained in both runs; only the score used for MPO policy improvement changes. The continuation specification is fixed and no survival target or dual update is used.
    
    \begin{table}[t]
        \centering
        \small
        \caption{\textbf{Task versus survival-only optimization on H1622.} Mean $\pm$ s.d.\ over three seeds, averaged over the final 5M environment steps.  $\hat p_{\mathrm{MC}}$ is computed from evaluation rollouts; $\hat p_{Q_1}$ is the corresponding estimate from the learned survival-mass critic.}
        \label{tab:hyfydy-survival-endpoint}
        \begin{tabular}{lcc}
            \toprule
            & $Q_r$ only ($\eta=0$) & $Q_1$ only \\
            \midrule
            MC survival $\hat p_{\mathrm{MC}}\uparrow$
                & $0.691\pm0.007$ & $0.683\pm0.014$ \\
            Critic estimate $\hat p_{Q_1}$
                & $0.598\pm0.006$ & $0.596\pm0.013$ \\
            Task reward $\uparrow$
                & $0.822\pm0.006$ & $0.794\pm0.003$ \\
            Violation magnitude / step $\downarrow$
                & $0.0816\pm0.0111$ & $0.1114\pm0.0073$ \\
            Episode length $\uparrow$
                & $922.4\pm12.2$ & $935.0\pm19.4$ \\
            \bottomrule
        \end{tabular}
    \end{table}
    
    \begin{figure}[t]
        \centering
        \includegraphics[width=\linewidth]{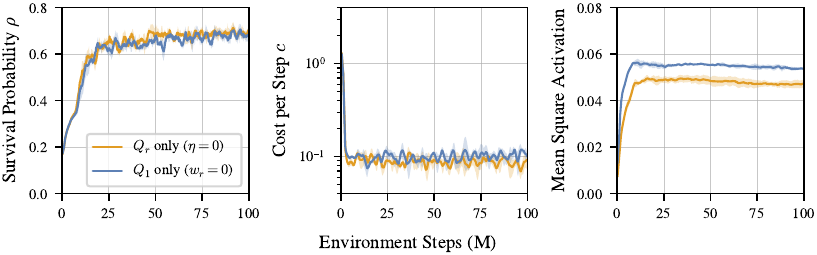}
        \caption{\textbf{Training dynamics for task versus survival-only policy improvement on H1622.} Optimizing $Q_1$ alone does not reveal a higher-survival regime:
        survival remains comparable to the fixed $\eta=0$ task objective, while aggregate operating-limit violation and muscle effort are higher. Curves show mean and standard deviation over three seeds.}
        \label{fig:hyfydy-survival-endpoint}
    \end{figure}
    
    The rollout estimate gives essentially the same survival at the two endpoints (\Cref{tab:hyfydy-survival-endpoint}), and the training curves show that this similarity persists after convergence rather than arising from a single checkpoint (\Cref{fig:hyfydy-survival-endpoint}).  In contrast, removing the effort term decreases task reward and increases the aggregate per-step violation magnitude.  Thus, under this H1622 training protocol, direct optimization of survival mass does not uncover additional attainable survival beyond that obtained by the fixed $\eta=0$ objective.  This supports the practical role predicted by the decomposition above: the unit offset in $r=1-f$ supplies substantial survival pressure without requiring a separately optimized survival multiplier.  We interpret this as an H1622 endpoint diagnostic, not as evidence that the resulting policy globally maximizes
    survival.

    \FloatBarrier
    \subsection{Diagnostic Variants}
    \label{app:hyfydy-diagnostics}
    
    The following experiments are diagnostic and not part of the main comparison to Schumacher et al. \citep{schumacher2025emergancenatural}. They isolate specific algorithmic properties of SDH under simplified problem formulations.

    \subsubsection{Velocity-only diagnostic setting}
    \label{app:hyfydy-velocity-only}

    The velocity-only formulation uses a single instantaneous constraint, omitting the additional biomechanical constraints (GRF, joint limits, muscle count) of the full formulation.
    
    While we retain the reward formulation in \cref{eq:hyfydy-reward}, we modify the cost to only consist of the velocity violation:
    \begin{equation}
    	c_\text{vel} = [v_\text{min}-v]_+.
    	\label{eq:velocity-only-diagnostic-app}
    \end{equation}
    This setting focuses on the core tension of gait optimization: the balancing of effort and velocity.
    
    \paragraph{AS-SAC vs. VT-MPO on Hyfydy.}
    We compare AS-SAC and VT-MPO on the velocity-only diagnostic formulation to separate the effect of the base algorithm from the problem design. The main empirical pattern is that VT-MPO is more robust on H2190, while AS-SAC can be more aggressive on easier models. AS-SAC fails learning to perform a walking gait on the H2190 model for one of the three seeds. See \cref{fig:cross_contin} for a comparison of sample efficiency curves.

    \begin{figure}[htb]
        \centering
        \includegraphics[width=\linewidth]{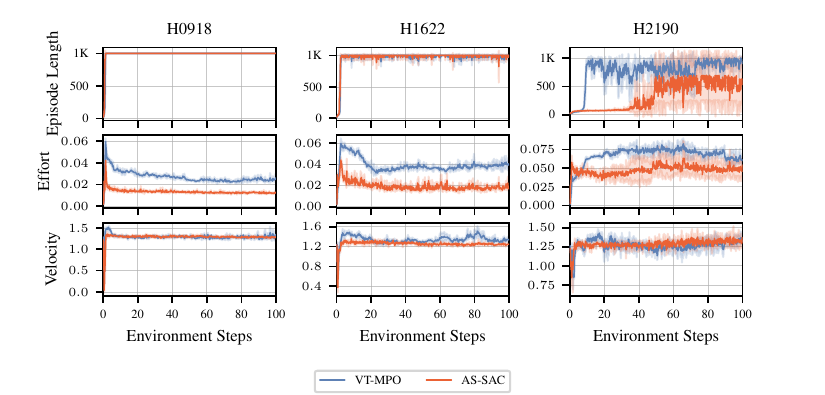}
        \caption{Sample efficiency curves for VT-MPO and AS-SAC on the velocity-only diagnostic problem formulation. Shown are mean and standard deviation over three seeds.}
        \label{fig:cross_contin}
    \end{figure}

    \begin{table}[htb]
    	\centering
    	\caption{Velocity-only diagnostic results at 100 million environment steps. Results reflect common knowledge about the tradeoff between MPO and SAC: SAC can optimize more aggressively, but is potentially unstable. MPO is more stable but peak performance (effort minimization) is lacking.}
    	\label{tab:hyfydy-cross-continuation-app}
    	\begin{tabular}{llccc}
    		\toprule
    		Model & Method & Effort $\downarrow$ & Episode length $\uparrow$ & Velocity \\
    		\midrule
    		H0918 & VT-MPO & $0.026\pm0.004$ & $1000.00\pm0.00$   & $1.352\pm0.076$ \\
    		H0918 & AS-SAC & $0.012\pm0.001$ & $1000.00\pm0.00$   & $1.279\pm0.020$ \\
    		\midrule
    		H1622 & VT-MPO & $0.040\pm0.003$ & $974.03\pm36.72$   & $1.346\pm0.023$ \\
    		H1622 & AS-SAC & $0.017\pm0.002$ & $1000.00\pm0.00$   & $1.251\pm0.010$ \\
    		\midrule
    		H2190 & VT-MPO & $0.061\pm0.007$ & $901.43\pm83.88$   & $1.318\pm0.040$ \\
    		H2190 & AS-SAC & $0.054\pm0.009$ & $622.93\pm394.02$  & $1.322\pm0.050$ \\
    		\bottomrule
    	\end{tabular}
    \end{table}

    An interesting side observation of these experiments is that effort reaches \emph{worse} values with the simplified, velocity-only formulation, compared to the full problem formulation (see \cref{fig:hyfydy-full-training-app}). This suggests that added constraints are synergistic with the effort objective, in the sense that they seem to steer the training in the right direction. Final values at 100 million environment steps are given in \cref{tab:hyfydy-cross-continuation-app}.

    To ensure fair comparison, and decouple algorithmic learning properties (SAC vs. MPO) from the influence of the problem design (EWA vs. SDH), we opted for VT-MPO for our final full problem formulation experiments, as EWA is also MPO based. While SAC might yield even faster convergence, we also expect a stability loss, which seems to be a downside on the H2190 model.
    
    \paragraph{PID-Lagrangian comparison.}
    For diagnostic comparisons, we evaluate SAC-PID-Lagrangian and MPO-PID-Lagrangian variants \cite{stooke2020responsive}. These baselines adapt a multiplier-like parameter using PID feedback on constraint violations. The goal of the comparison is to illustrate that training instability is a general downside of objective reweighting based approaches, such as Lagrange multiplier based ones. We implemented PID Lagrange ourselves in the DEP-RL codebase, and tuned PID parameters on the considered environments. For the PID baseline we designed a similar effort-velocity problem setting in the CMDP framework, where the reward equals \cref{eq:hyfydy-reward}, and the constraint is given by:
    \begin{equation}
        \mathbb{E}\left[\sum_{t=0}^\infty \gamma^t c_t\right] \leq 0 \qquad \text{where } \qquad c = v_\text{target} - v.
    \end{equation}
    
    As we have a cumulative constraint, not an instantaneous one, we choose $v_\text{target} = \SI{1.2}{\meter\per\second}$, the midpoint of the instantaneous velocity band used in our full formulation. The sample efficiency curves are contrasted with SDH based approaches in \cref{fig:pid-lag-hyfydy}.
    
    \begin{figure}[htb]
        \centering
        \includegraphics[width=\linewidth]{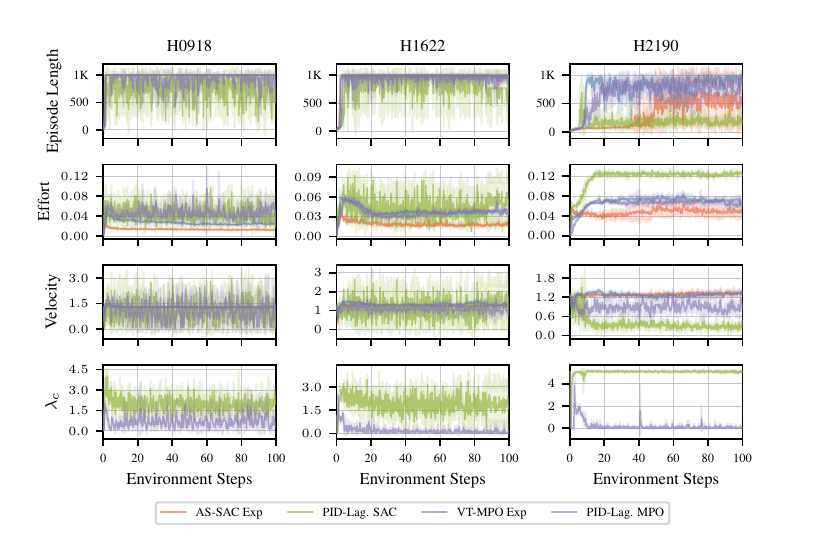}
        \caption{Comparison of PID Lagrangian based approaches (PID SAC, PID MPO) with SDH based approaches (AS-SAC, VT-MPO). Shown are mean and standard deviation for three seeds over 100 million environment steps.}
        \label{fig:pid-lag-hyfydy}
    \end{figure}

    There are three general patterns: PID SAC is more unstable than PID MPO, PID MPO and VT-MPO follow similar mean curve trajectories, and, most importantly: compared to SDH methods, even after tuning, PID methods still show significantly larger objective oscillations. This illustrates the difficulty of getting Lagrange based methods to work properly, clearly highlighting the practical benefits of our method.
    
    \paragraph{Continuation schedule and strength ablations.}
    \label{par:continuation}
    We report fixed-strength and no-schedule ablations to characterize sensitivity to $\lambda$. The purpose is to support the narrower claim that SDH replaces adaptive dual-variable feedback with a continuation-design choice; it is not tuning-free. 

    In our standard setup we use a linear rampup of the $\lambda$ parameter in the exponential continuation over the first 5 million (out of 100 million) environment steps. We ablated this design choice on a test run on H1622. The training curves are shown in \cref{fig:hyfydy-ramp-up}. Additionally, we evaluate the effects of choosing different $\lambda$ values for the exponential continuation. The results are shown in \cref{fig:lambda-p0-var}.
    
    \begin{figure}[htb]
        \centering
        \includegraphics[width=\linewidth]{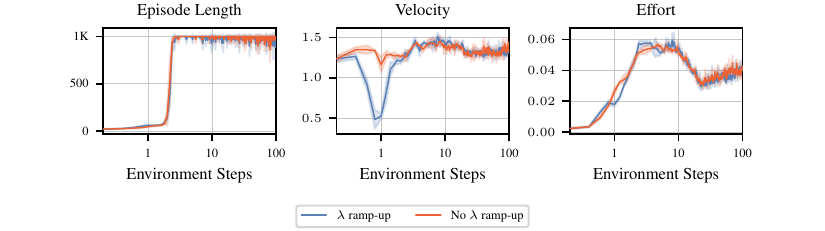}
        \caption{Ablation of initial $\lambda$ ramp-up. The $\lambda$ ramp-up has negligible effects on the training in the limit, only showing moderate influence in the very beginning of training.}
        \label{fig:hyfydy-ramp-up}
    \end{figure}
           
    \begin{figure}[htb]
        \centering
        \includegraphics[width=\linewidth]{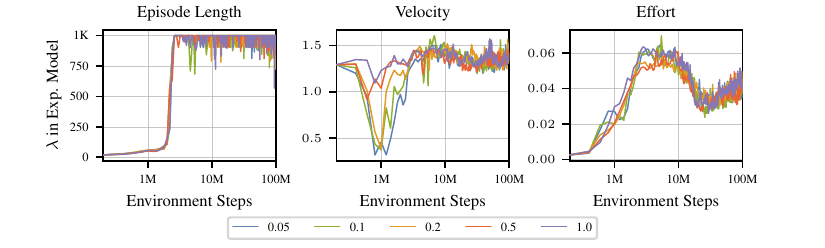}
        \caption{Runs on the velocity-only constraint formulation with SDH, for different choices of $\lambda$ in the exponential continuation on H1622. We see that training dynamics are only affected in the initial phases of training, with the limit behavior being remarkably robust to $\lambda$.}
        \label{fig:lambda-p0-var}
    \end{figure}

    The results show that SDH is generally not sensitive to the exact choice of continuation-related hyperparameters. This robustness is a clear upside of our method.
    
    \subsection{Hyperparameters}
    \label{app:hyfydy-hyperparameters}
    
    \begin{table}[h]
        \centering
        \caption{VT-MPO hyperparameters for Hyfydy.}
        \label{tab:hyfydy-vtmpo-hparams-app}
        \begin{tabular}{lc}
        	\toprule
        	Hyperparameter & Value \\
        	\midrule
        	Discount $\gamma$ & $0.99$ \\
        	Replay buffer size & $10^6$ \\
        	Batch size & $256$ \\
        	Batch iterations & $30$ \\
        	$n$-step returns & $1$ \\
        	Warmup steps & $2\times10^5$ \\
        	Steps between batches & $1\,000$ \\
        	Total training steps & $10^8$ \\
        	Parallel environments & $20$ \\
        	Sequential steps & $10$ \\
        	Network architecture & MLP$(256, 256)$ / MLP$(1024, 1024)$\textsuperscript{$\dagger$} \\
        	Actor learning rate & $3\times10^{-4}$ / $3.53\times10^{-5}$\textsuperscript{$\dagger$} \\
        	Critic learning rate & $10^{-3}$ / $6.08\times10^{-5}$\textsuperscript{$\dagger$} \\
        	Dual variable learning rate & $10^{-2}$ / $2.13\times10^{-3}$\textsuperscript{$\dagger$} \\
        	Action samples per state & $20$ \\
    		Continuation family & Exponential \\
    		$\lambda$ schedule & $[0,\,0.2]$ \\
            $\eta$ & $0$ \\
    		Schedule horizon & $5\times10^6$ \\
        	\bottomrule
        \end{tabular}

        \smallskip
        {\small $\dagger$\,H2190 uses MLP$(1024, 1024)$ and tuned learning rates (second value); H0918 and H1622 use MLP$(256, 256)$ with defaults (first value).}
    \end{table}
    
    \begin{table}[h]
    	\centering
    	\caption{DEP hyperparameters for Hyfydy.}
    	\label{tab:hyfydy-cont-dep-hparams-app}
    	\begin{tabular}{lc}
    		\toprule
    		Hyperparameter & Value \\
    		\midrule
    		DEP intervention probability & $0.00371$ \\
    		DEP intervention length & $8$ \\
    		DEP bias rate & $0.002$ \\
    		DEP buffer size & $200$ \\
    		DEP $\kappa$ & $1\,000$ \\
    		DEP regularization & $32$ \\
    		DEP $\tau_\text{DEP}$ & $40$ \\
    		\bottomrule
    	\end{tabular}
    \end{table}
    
    \subsection{Compute and runtime}
    \label{app:hyfydy-compute}
    
    SDH does not introduce additional environment interaction or policy rollouts relative to the underlying actor-critic method. The only algorithmic change is the critic target through $(\tilde r,\tilde\gamma)$ and the computation of $\alpha(s,a)$ from constraint signals. For transparency, we report the compute used in Table~\ref{tab:hyfydy-compute-app}; we do not claim a systematic wall-clock advantage.
    
    \begin{table}[h]
    	\centering
    	\caption{Hyfydy compute summary. All runs use NVIDIA RTX 2080 on our cluster. SDH modifies only the critic target and incurs no additional runtime overhead.}
    	\label{tab:hyfydy-compute-app}
    	\begin{tabular}{lc}
    		\toprule
    		Resource / setting & Value \\
    		\midrule
    		Hardware & NVIDIA RTX 2080 \\
    		Parallel environments & $20$ \\
    		Total environment steps & $10^8$ \\
    		Throughput H0918 (steps/s) & ${\sim}1386$ \\
    		Throughput H1622 (steps/s) & ${\sim}1278$ \\
    		Throughput H2190 (steps/s) & ${\sim}870$ \\
    		Wall-clock H0918 (h) & ${\sim}21$ \\
    		Wall-clock H1622 (h) & ${\sim}25$ \\
    		Wall-clock H2190 (h) & ${\sim}36$ \\
    		\bottomrule
    	\end{tabular}
    \end{table}
    
    \clearpage
    \newpage    
    \FloatBarrier
    \clearpage

    \section{Experiments on Safety Gymnasium}\label{app:results-safety-gymnasium}
    
    \subsection{Details on the Experimental Setup}
    We evaluate on a subset of 8 continuous-control tasks from Safety Gymnasium \citep{ji2023safetygymnasium}, which provide locomotion and navigation-style environments with standardized cost signals (e.g., hazards). 
    The benchmark consists of 4 locomotion and 4 navigation environments. 
    The full list can be seen in Table~\ref{table:safety-gymnasium-results}.
    We report average episodic returns of reward and cost per episode in the form of sample efficiency curves measured against the number of environment interactions. 
    Each training run is repeated for 5 seeds and the sample efficiency results are aggregated across seeds for each environment. 
    We compare against unconstrained MPO~\citep{abdolmaleki2018maximum}, constrained off-policy baselines MPO-PID~\cite{stooke2020responsive} and WCSAC~\cite{yang2021wcsac} as well as constrained on-policy baselines CPO~\citep{achiam2017cpo} and C-TRPO~\citep{milosevic2025embedding}. Hyperparameters are taken from the respective publications, as they evaluate on the same environments. VT-MPO and AS-SAC introduce one additional hyperparemeter $\lambda$ that controls the mapping from constraint violations to continuation probabilities. 

    For both SDH methods, we use the same linear schedule $\lambda\in[0,0.9]$ over the first $500k$ environment steps. Replay retains the raw reward and violation information, and the SDH transformations are constructed at critic-update time using the current $\lambda$, as described in Sec.~\ref{sec:algorithms-main}. All hyperparameters are summarized in Table~\ref{tab:hyperparameters}.

    \begin{table}[h]
    	\centering
    	\caption{Hyperparameters for the off-policy algorithms on Safety Gymnasium.}
    	\label{tab:hyperparameters}
        \adjustbox{max width=\textwidth}{%
    	\begin{tabular}{lcccc}
    		\toprule
    		\textbf{Hyperparameter} & \textbf{AS-SAC} & \textbf{WCSAC} & \textbf{MPO-PID} & \textbf{VT-MPO} \\
    		\midrule
    		\multicolumn{5}{l}{\textit{General}} \\
    		Discount factor $\gamma$             & $0.99$           & $0.99$           & $0.99$           & $0.99$ \\
    		Replay buffer size                   & $10^6$           & $10^6$           & $10^6$           & $10^6$ \\
    		Batch size                           & $256$            & $256$            & $256$            & $256$ \\
    		Warmup steps                         & $5\,000$         & $5\,000$         & $1\,000$         & $1\,000$ \\
            Network architecture & MLP$(256, 256)$& MLP$(256, 256)$& MLP$(256, 256)$& MLP$(256, 256)$\\
    		\midrule
    		\multicolumn{5}{l}{\textit{SAC}} \\
    		Actor learning rate                  & $3\times10^{-4}$ & $3\times10^{-4}$ & —                & — \\
    		Critic learning rate                 & $10^{-3}$        & $10^{-3}$        & —                & — \\
    		Target smoothing $\tau$              & $5\times10^{-3}$ & $5\times10^{-3}$ & —                & — \\
    		Actor update frequency               & $2$              & $2$              & —                & — \\
    		Target network update frequency      & $1$              & $1$              & —                & — \\
    		Entropy coefficient $\alpha$         & adaptive         & $0$              & —                & — \\
    		\midrule
    		\multicolumn{5}{l}{\textit{MPO}} \\
    		Policy/critic learning rate          & —                & —                & $3\times10^{-4}$ & $3\times10^{-4}$ \\
    		Dual variable learning rate          & —                & —                & $10^{-2}$        & $10^{-2}$ \\
    		$n$-step returns                     & —                & —                & $4$              & $4$ \\
    		$\varepsilon$ (non-parametric KL)    & —                & —                & $0.1$            & $0.1$ \\
    		$\varepsilon_\mu$ (policy mean KL)   & —                & —                & $0.01$           & $0.01$ \\
    		$\varepsilon_\sigma$ (policy std KL) & —                & —                & $10^{-6}$        & $10^{-6}$ \\
    		$\varepsilon_\text{pen}$ (action limit KL) & —          & —                & $10^{-3}$        & $10^{-3}$ \\
    		Target/policy update period          & —                & —                & $100$            & $100$ \\
    		Action samples per state             & —                & —                & $20$             & $20$ \\
    		Gradient norm clip                   & —                & —                & $40$             & $40$ \\
    		\midrule
    		\multicolumn{5}{l}{\textit{Constraint}} \\
    		Cost limit                           & $25.0$           & $25.0$           & $25.0$           & $25.0$ \\
    		$\lambda$-schedule (0-500k steps)                & $[0, 0.9]$           & —                & —                & $[0, 0.9]$ \\
            SDH survival bonus $\eta$                & $0.1$           & —                & —                & $0.1$ \\
    		CVaR confidence level                & —                & $0.9$            & —                & — \\
    		Dual variable $\beta$ learning rate  & —                & $5\times10^{-4}$ & —                & — \\
    		PID gains $(K_P,\,K_I,\,K_D)$       & —                & —                & $(0.1,\,0.01,\,0.5)$ & — \\
    		\bottomrule
    	\end{tabular}}
    \end{table}

    \subsection{Main results on Safety Gymnasium} \paragraph{Aggregate performance.}
    Figure~\ref{fig:safety-aggregate-main} shows that VT-MPO achieves the most favorable aggregate return-violation trade-off.
    Compared to unconstrained MPO, both VT-MPO and AS-SAC substantially reduce cost while retaining competitive reward.
    Compared to the on-policy baselines, they reach similar or lower violation levels with substantially higher return.
    
    \paragraph{Per-environment behavior.}
    Figure~\ref{fig:sample_efficiency_envs} shows that this aggregate trend is consistent across most individual environments.
    VT-MPO and AS-SAC reliably suppress cost towards the end of training while improving reward rapidly early on, whereas CPO and C-TRPO are typically conservative early on and improve reward slowly, reflecting their on-policy nature.
    Overall, AS-SAC tends to achieve higher asymptotic return in environments where constraint satisfaction is maintained throughout training.
    VT-MPO, however, is more robust across tasks and seeds, which we attribute to MPO’s KL-regularized policy improvement under variable discounting.
    Finally, in two of the eight environments (\texttt{CarButton} and \texttt{PointGoal}), neither VT-MPO nor AS-SAC produce a feasible policy.
    
    \paragraph{Effect of survival weighting.}
    Across environments, decreasing the scale of the continuation probability $\alpha$ over training improves the return-violation trade-off.
    Early in learning, higher $\alpha$ allows exploration of infeasible regions and reuse of infeasible transitions.
    Later, reduced survival weighting attenuates long-horizon credit assignment through infeasible states, guiding convergence toward feasible behavior.
    This supports interpreting SDH as a continuous relaxation of hard constraints rather than a binary feasibility mechanism.
    
    \begin{table}[ht]
    	\caption{Expected returns (R) and costs (C) on Safety Gymnasium after one million environment steps. The cost limit is 25.0 mean cost return in all environments~\citep{ji2023safetygymnasium}. Violating entries are marked {\color{ourred}red}. The algorithm with the highest inter quartile mean return (25\%) among the feasible ones in an environment is marked in bold. The larger side of the confidence interval is denoted by $\pm$, rounded to one decimal. In \texttt{CarButton} and \texttt{PointGoal}, no algorithm yields a feasible policy in the prescribed number of environment interactions.}
    	\label{table:safety-gymnasium-results}
    	\setlength\tabcolsep{2pt}
    	\centering
        \adjustbox{max width=\textwidth}{%
    	\begin{tabular}{llllllllll}
    		\toprule
    		&  & Ant & Half    & Humanoid & Hopper & Car    & Point & Racecar & Point \\
    		&  &     & Cheetah &          &        & Button & Goal  & Circle  & Push \\
    		\midrule
    		\multirow[t]{2}{*}{VT-MPO} & $R$ & {829 ± 99} & {2021 ± 71} & {2683 ± 531} & {1293 ± 186} & {\color{ourred}15.5 ± 2.4} & {\color{ourred}15.6 ± 2.0} & {1.0 ± -0.0} & {\bfseries 1.8 ± 0.7} \\
    		(ours) & $C$ & 1.3 ± 0.1 & 1.6 ± 0.6 & 0.3 ± 0.7 & 0.2 ± 0.2 & {\color{ourred}303 ± 27} & {\color{ourred}59.6 ± 7.9} & 7.3 ± 6.8 & 13.4 ± 5.0 \\
    		\cline{1-10}
    		\multirow[t]{2}{*}{AS-SAC} & $R$ & {\bfseries 2779 ± 64} & {\bfseries 2608 ± 37} & {\bfseries 5549 ± 235} & {681 ± 289} & {\color{ourred}0.1 ± 3.4} & {\color{ourred}25.6 ± 0.6} & {\bfseries 26.6 ± 1.9} & {\color{ourred}-0.3 ± 0.4} \\
    		(ours) & $C$ & 0.4 ± 0.2 & 0.2 ± 0.1 & 0.0 ± 0.0 & 3.5 ± 3.1 & {\color{ourred}311 ± 170} & {\color{ourred} 57.6 ± 8.1} & 22.6 ± 11.6 & {\color{ourred}80.7 ± 44.7} \\
    		\cline{1-10}
    		\multirow[t]{2}{*}{C-TRPO} & $R$ & {1634 ± 198} & {1363 ± 182} & {857 ± 198} & {\bfseries 1364 ± 65} & {\color{ourred}-0.5 ± 0.4} & {\color{ourred}8.8 ± 4.9} & {1.4 ± 2.4} & {\color{ourred}0.2 ± 0.2} \\
    		& $C$ & 22.3 ± 9.7 & 17.4 ± 4.5 & 6.3 ± 2.7 & 8.4 ± 11.1 & {\color{ourred}32.1 ± 35.7} & {\color{ourred}28.8 ± 10.3} & 24.5 ± 15.9 & {\color{ourred}25.7 ± 19.1} \\
    		\cline{1-10}
    		\multirow[t]{2}{*}{CPO} & $R$ & {1522 ± 90} & {1163 ± 97} & {881 ± 79} & {1115 ± 28} & {\color{ourred}-0.8 ± 0.4} & {\color{ourred}8.2 ± 3.7} & {1.6 ± 2.1} & {0.4 ± 0.5} \\
    		& $C$ & 20.0 ± 5.8 & 18.2 ± 9.5 & 7.2 ± 1.8 & 18.8 ± 32.2 & {\color{ourred} 41.0 ± 6.8} & {\color{ourred} 32.8 ± 5.3} & 21.5 ± 32.4 & 18.3 ± 6.9 \\
    		\cline{1-10}
    		\multirow[t]{2}{*}{PPO} & $R$ & {44.6 ± 6.9} & {\color{ourred}1671 ± 758} & {506 ± 7} & {\color{ourred}491 ± 174} & {\color{ourred}16.8 ± 1.2} & {\color{ourred}25.2 ± 0.7} & {\color{ourred}35.0 ± 2.7} & {\color{ourred}0.7 ± 0.3} \\
    		& $C$ & 1.4 ± 0.3 & {\color{ourred}151 ± 266} & 1.3 ± 1.4 & {\color{ourred}135 ± 30} & {\color{ourred}333 ± 25} & {\color{ourred}50.6 ± 6.1} & {\color{ourred}194 ± 46} & {\color{ourred}52.8 ± 36.4} \\
    		\cline{1-10}
    		\multirow[t]{2}{*}{MPO} & $R$ & {\color{ourred}1136 ± 185} & {\color{ourred}8654 ± 126} & {\color{ourred}2395 ± 1088} & {\color{ourred}1643 ± 548} & {\color{ourred}26.0 ± 0.9} & {\color{ourred}28.0 ± 0.4} & {1.1 ± -0.0} & {\color{ourred}6.3 ± 2.6} \\
    		& $C$ & {\color{ourred}254 ± 74} & {\color{ourred}979 ± 0} & 1.2 ± 3.7 & {\color{ourred}410 ± 172} & {\color{ourred}250 ± 46} & {\color{ourred}56.5 ± 9.1} & {10.8 ± 13.2} & {\color{ourred}36.3 ± 13.3} \\
    		\cline{1-10}
    		\multirow[t]{2}{*}{MPO-PID} & $R$ & {\color{ourred}543 ± 64} & {\color{ourred}2498 ± 14} & {2500 ± 74} & {\color{ourred}1269 ± 222} & {\color{ourred}0.5 ± 0.4} & {\color{ourred}12.9 ± 1.9} & {1.0 ± 0.0} & {\color{ourred}2.4 ± 0.5} \\
    		& $C$ & {\color{ourred}33.7 ± 5.1} & {\color{ourred}47.9 ± 2.7} & 1.1 ± 0.2 & {\color{ourred}76.7 ± 14.2} & {\color{ourred}61.8 ± 17.2} & {\color{ourred}29.5 ± 13.6} & 9.8 ± 13.3 & {\color{ourred}30.0 ± 21.9} \\
    		\cline{1-10}
    		\multirow[t]{2}{*}{WCSAC} & $R$ & {-2301.4 ± 843.1} & {1824 ± 677} & {1153 ± 1609} & {885 ± 529} & {\color{ourred}-1.3 ± 0.1} & {\bfseries -0.1 ± 0.1} & {\color{ourred}25.8 ± 16.0} & {\color{ourred}-1.1 ± 0.9} \\
    		& $C$ & 0.1 ± 0.2 & 1.3 ± 0.6 & 0.4 ± 0.7 & 0.0 ± 0.0 & {\color{ourred}47.1 ± 20.0} & 0.0 ± 28.1 & {\color{ourred}140 ± 92} & {\color{ourred}46.8 ± 60.8} \\
    		\cline{1-10}
    		\multirow[t]{2}{*}{SafePO Best} & $R$ & {3297.3} & {3336.8} & {6620.7} & {1716.4} & { 0.1} & { 19.0} & {64.8} & {4.4} \\
    		Alg. @10M.& $C$ & 23.6 & 1.1 & 0.0 & 5.4 & 11.9 & 22.9 & 20.2 & 23.5 \\
    		\bottomrule
    	\end{tabular}}
    \end{table}

    \subsection{Violation-depth profiles across Safety Gymnasium}
    \label{app:sgym-vdp-sweep}
    
    To test the violation-depth profile scope claim, we use the same diagnostic in every Safety Gymnasium environment. We compute each profile from the converged unconstrained SAC policy, rather than selecting environments based on AS-SAC or VT-MPO outcomes. The profile analysis is therefore independent of SDH performance. Since $\Omega_\pi$ depends on the policy, this is not a policy-uniform statement. Instead, it is the same method-independent probe applied across environments.
    
    The single-scale correspondence assumption states that
    \[
        \Omega_\pi(b)
        =
        (1-\gamma)^{-1}\exp\{-b/\tau_\pi\}
        \qquad\text{for some }\tau_\pi>0.
    \]
    Empirically, this requires two features of the profile. First, the positive-depth profile must be well described by an exponential, which appears as an approximately linear profile on the logarithmic scale. Second, the fitted constant factor must match the value $1/(1-\gamma)$ required by this single-scale form. Figure~\ref{fig:sgym-vdp-sweep} and Table~\ref{tab:sgym-vdp-sweep} report both factors.
    
    \begin{figure*}[t]
        \centering
        \includegraphics[width=\textwidth]{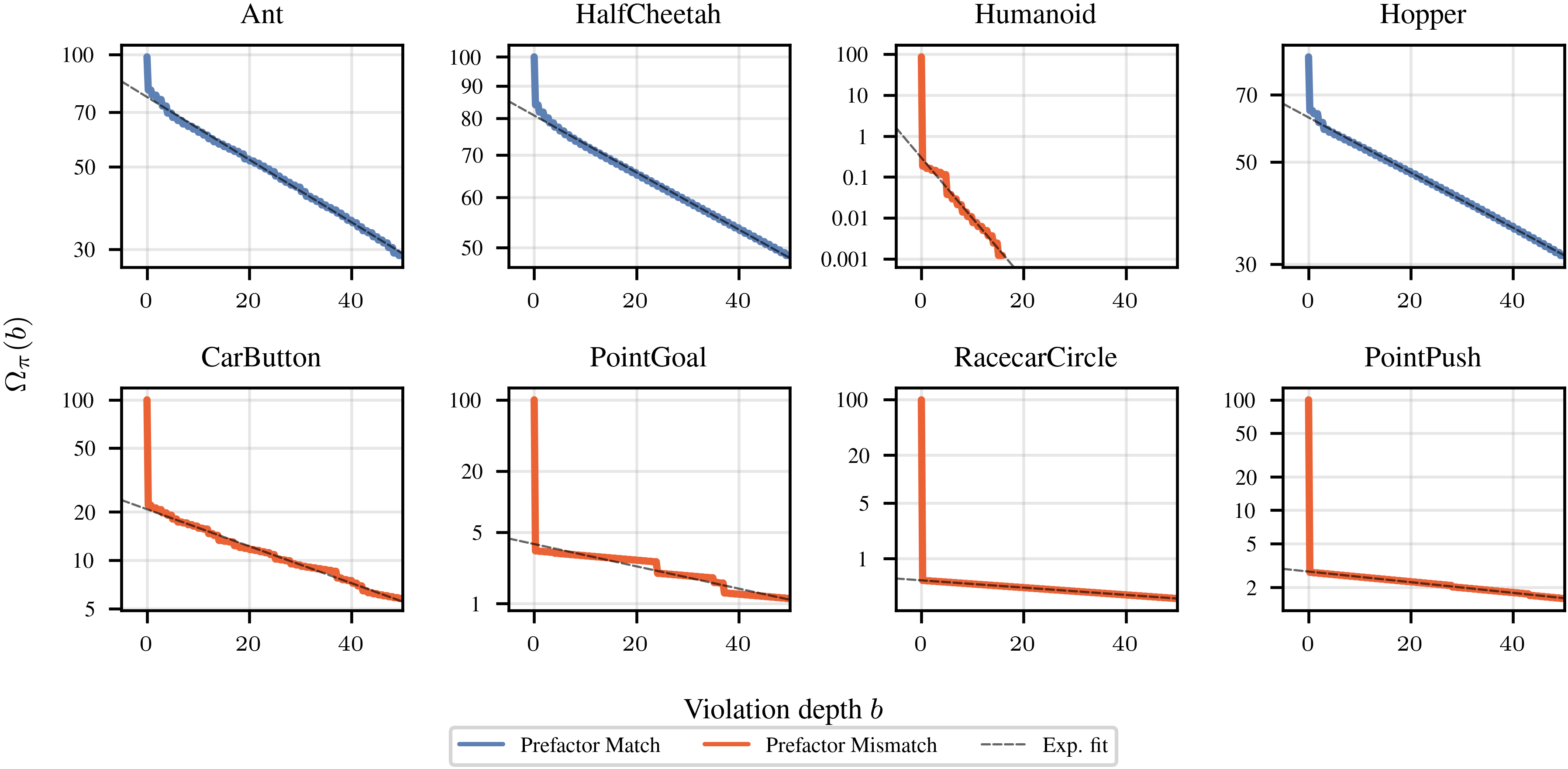}
        \caption{
            \textbf{Violation-depth profiles across Safety Gymnasium.}
            Profiles $\Omega_\pi(b)$ are computed from converged unconstrained SAC policies using the same procedure in every environment. Dashed lines show exponential fits over positive violation depths. \texttt{Ant}, \texttt{HalfCheetah}, \texttt{Hopper}, and \texttt{CarButton} have nearly log-linear positive-depth profiles, but the single-scale form additionally requires the fitted constant factor to equal $1/(1-\gamma)$. This second condition is approximately matched for \texttt{Ant} and \texttt{HalfCheetah}, but not for \texttt{Hopper} or \texttt{CarButton}. For \texttt{Humanoid}, \texttt{PointGoal}, and \texttt{PointPush}, positive violations occur in only $18\%$, $13\%$, and $11\%$ of episodes, respectively, making the diagnostic inconclusive.
        }
        \label{fig:sgym-vdp-sweep}
    \end{figure*}
    
    \begin{table}[t]
        \centering
        \small
        \caption{
            \textbf{Violation-depth profile diagnostic.}
            $R^2$ measures the exponential fit of the positive-depth profile; $A$ is the fitted prefactor; with $\gamma=0.99$, the single-scale form requires $A=(1-\gamma)^{-1}=100$; $J_c$ is the discounted cumulative cost and $\tau$ is the fitted exponential scale. ``Single-scale'' denotes a near-exponential positive-depth profile. The final column reports whether the fitted prefactor approximately matches the value required by the single-scale form.
        }
        \label{tab:sgym-vdp-sweep}
        \begin{tabular}{llrrrrc}
            \toprule
            Environment & Type & $\tau$ & $R^2$ & $A$ & $J_c$ & Prefactor \\
            \midrule
            Ant & Velocity & 51.6 & 0.997 & 76.9 & 38.4 & Match \\
            HalfCheetah & Velocity & 96.5 & 0.997 & 80.8 & 80.1 & Match \\
            Humanoid & Velocity & 2.9 & 0.977 & 0.3 & 0.0 & Mismatch \\
            Hopper & Velocity & 72.4 & 0.998 & 62.5 & 36.7 & Match \\
            \midrule
            CarButton & Navigation & 37.7 & 0.992 & 20.8 & 10.6 & Mismatch \\
            PointGoal & Navigation & 39.7 & 0.942 & 3.8 & 4.2 & Mismatch \\
            RacecarCircle & Navigation & 94.8 & 1.000 & 0.5 & 1.4 & Mismatch \\
            PointPush & Navigation & 89.9 & 0.997 & 2.8 & 4.8 & Mismatch \\
            \bottomrule
        \end{tabular}
    \end{table}
    
    The sweep sharpens the empirical characterization of the boundary predicted by the theory. Log-linearity alone is not enough: \texttt{Hopper} and \texttt{CarButton} have almost perfectly exponential positive-depth profiles, yet their fitted constant factors differ substantially from the value required by the single-scale form. \texttt{Ant} and \texttt{HalfCheetah} are closest to satisfying both conditions. The difference is also reflected in the relation between $J_c$ and the fitted $\tau$: they are considerably closer for \texttt{Ant} and \texttt{HalfCheetah} than for \texttt{Hopper} and \texttt{CarButton}. For \texttt{Humanoid}, \texttt{PointGoal}, and \texttt{PointPush}, the unconstrained SAC probe rarely enters positive violation depths, so the profile test is inconclusive.
    
    Accordingly, we use the violation-depth profile as a scope diagnostic for when the specific correspondence of Sec.~\ref{sec:sdh-cmdp-main} is empirically supported. Failure of this diagnostic does not imply that an SDH optimizer must fail or violate the native cost limit; it means only that one SDH Laplace statistic is not justified as a surrogate for additive cost by the stated single-scale result.
    
    \subsection{Fixed-$\lambda$ ablations}
    We demonstrate additional fixed-$\lambda$ ablations (no annealing) for exponential continuation with AS-SAC on a subset of environments in Table~\ref{tab:sgym-lambda-sweep}. The results demonstrate that SDH undergoes a ``phase transition'' in the environments \texttt{CarButton} and \texttt{PointGoal}, where the converged policy either satisfies the constraints with close-to-zero reward or achieves high reward by ignoring the constraint
    \begin{table}[h]
    	\centering
        \small
    	\caption{Expected returns (R) and costs (C) on Safety Gymnasium for \textbf{fixed-$\lambda$ exponential continuation} with AS-SAC after one million environment steps. The cost limit is 25.0 mean cost return in all environments~\citep{ji2023safetygymnasium}. Violating entries are marked {\color{ourred}red}. The algorithm with the highest inter quartile mean return (25\%) among the feasible ones in an environment is marked in bold. The larger side of the confidence interval is denoted by $\pm$, rounded to one decimal.}
    	\label{tab:sgym-lambda-sweep}
    	\begin{tabular}{lllll}
    		\toprule
    		&  & HalfCheetahVelocity & CarButton & PointGoal \\
    		\midrule
    		\multirow[t]{2}{*}{AS-SAC (lambda=0.005)} & $R$ & {\color{ourred}8797 ± 2271} & {\color{ourred}11.3 ± 1.9} & {\color{ourred}26.9 ± 0.4} \\
    		& $C$ & {\color{ourred}975 ± 4} & {\color{ourred}232 ± 75} & {\color{ourred}45.7 ± 3.1} \\
    		\cline{1-5}
    		\multirow[t]{2}{*}{AS-SAC (lambda=0.05)} & $R$ & {\bfseries 2918 ± 36} & {\color{ourred}-0.1 ± 0.3} & {\bfseries -0.3 ± 0.3} \\
    		& $C$ & 2.2 ± 0.5 & {\color{ourred}77.4 ± 84.4} & 22.9 ± 6.8 \\
    		\cline{1-5}
    		\multirow[t]{2}{*}{AS-SAC (lambda=0.1)} & $R$ & {2868 ± 64} & {\bfseries 0.0 ± 0.0} & {\color{ourred}-0.3 ± 0.1} \\
    		& $C$ & 1.2 ± 0.6 & 0.0 ± 0.0 & {\color{ourred}34.7 ± 20.5} \\
    		\cline{1-5}
    		\multirow[t]{2}{*}{AS-SAC (lambda=0.2)} & $R$ & {2817 ± 36} & {\color{ourred}0.5 ± 0.4} & {\color{ourred}-0.1 ± 0.2} \\
    		& $C$ & 0.6 ± 0.2 & {\color{ourred}143 ± 35} & {\color{ourred}55.9 ± 80.6} \\
    		\cline{1-5}
    		\multirow[t]{2}{*}{AS-SAC (lambda=0.5)} & $R$ & {2766 ± 57} & {\color{ourred}0.1 ± 0.3} & {\color{ourred}-0.1 ± 0.2} \\
    		& $C$ & 0.3 ± 0.1 & {\color{ourred}137 ± 108} & {\color{ourred}69.4 ± 4.1} \\
    		\cline{1-5}
    		\bottomrule
    	\end{tabular}
    \end{table}
    
    \subsection{Increased step budget}
    \Cref{tab:sgym-increased-budget} shows performance on the most challenging Safety Gymanasium tasks from our benchmark for an increased step budget of 3 million steps. Both algorithms have substantially improved performance after 3 million steps on \texttt{PointGoal}, while reward remains close to zero on \textit{CarButton} for our default setting.
    \begin{table}[h]
    	\centering
        \small
    	\caption{Expected returns (R) and costs (C) on the challenging \textit{CarButton} and \textit{PointGoal} environments in Safety Gymnasium with an \textbf{increased step budget of 3 million steps}. The cost limit is 25.0 mean cost return in all environments~\citep{ji2023safetygymnasium}. Violating entries are marked {\color{ourred}red}. The algorithm with the highest inter quartile mean return (25\%) among the feasible ones in an environment is marked in bold. The larger side of the confidence interval is denoted by $\pm$, rounded to one decimal.}
    	\label{tab:sgym-increased-budget}
    	\begin{tabular}{llll}
    		\toprule
    		&  & CarButton & PointGoal \\
    		\midrule
    		\multirow[t]{2}{*}{VT-MPO (ours)} & $R$ & {\color{ourred}0.8 ± 1.1} & {\color{ourred}17.7 ± 1.9} \\
    		& $C$ & {\color{ourred}58.8 ± 9.4} & {\color{ourred}30.3 ± 17.9} \\
    		\cline{1-4}
    		\multirow[t]{2}{*}{AS-SAC (ours)} & $R$ & {\bfseries -1.6 ± 0.6} & {\bfseries 18.1 ± 3.3} \\
    		& $C$ & 4.0 ± 9.6 & 23.3 ± 11.4 \\
    		\cline{1-4}
    		\bottomrule
    	\end{tabular}
    \end{table}
    
    \subsection{Evaluated AS baselines and ablations}\label{app:sac-ablations}
    We distinguish three AS-SAC implementation variants:
    (i) \textbf{AS-SAC (adaptive $\kappa$)}, our main implementation, which uses the AS living-cost-aware critic together with standard SAC-style temperature adaptation (App.~\ref{app:as:sac-temperature}),
    (ii) \textbf{AS-SAC (fixed $\kappa$)}, which uses the same critic but keeps $\kappa$ constant, and
    (iii) \textbf{AS-SAC (no living cost)}, which \emph{incorrectly} removes the prior-normalization term and the induced horizon regularization from AS policy evaluation.
    The ``no living cost'' ablation directly tests the theoretical prediction that, under AS, the uniform-prior term is not a removable constant, while the comparison between adaptive and fixed $\kappa$ tests the effect of temperature adaptation.
    Table~\ref{table:safety-gymnasium-sac-ablation} shows that AS-SAC with adaptive $\kappa$ achieves the highest feasible return among these AS-SAC variants on 3 out of 8 environments.
    We hypothesize that this variant strikes a favorable balance between the theoretically justified AS critic and optimization stability from standard SAC-style temperature adaptation with a constant entropy target.
    This is also the version we use in the main benchmark.
    
    \begin{table}[ht]
    	\caption{Expected returns (R) and costs (C) on Safety Gymnasium for the AS-SAC ablations after one million environment steps with cost limit 25.0 mean cost return~\citep{ji2023safetygymnasium}. The algorithm with the highest inter quartile mean return (25\%) among the feasible ones in an environment is marked in bold.}
    	\label{table:safety-gymnasium-sac-ablation}
    	\adjustbox{max width=\textwidth}{
    	\setlength\tabcolsep{2pt}
    	\centering
    	\begin{tabular}{llllllllll}
    		\toprule
    		&  & Ant & HalfCheetah & Humanoid & Hopper & CarButton & PointGoal & RacecarCircle & PointPush \\
    		\midrule
    		\multirow[t]{2}{*}{AS-SAC (adaptive $\kappa$)} & $R$ & {\bfseries 2779 ± 64} & {2608 ± 37} & {\bfseries 5549 ± 235} & {681 ± 283} & {0.1 ± 3.4} & {25.6 ± 0.6} & {\bfseries 26.6 ± 1.9} & {-0.3 ± 0.5} \\
    		& $C$ & 0.4 ± 0.2 & 0.2 ± 0.1 & 0.0 ± 0.0 & 3.5 ± 3.1 & 311 ± 170 & 57.6 ± 8.1 & 22.6 ± 11.6 & 80.7 ± 42.4 \\
    		\cline{1-10}
    		\multirow[t]{2}{*}{AS-SAC (fixed $\kappa$)} & $R$ & {2553 ± 18} & {2699 ± 17} & {5384 ± 343} & {\bfseries 1009 ± 50} & {-3.4 ± 1.5} & {-0.7 ± 0.3} & {0.1 ± 0.1} & {-0.6 ± 0.3} \\
    		& $C$ & 0.5 ± 0.2 & 0.3 ± 0.0 & 0.0 ± 0.2 & 1.1 ± 9.2 & 270 ± 160 & 115 ± 66 & 0.3 ± 2.3 & 27.9 ± 22.0 \\
    		\cline{1-10}
    		\multirow[t]{2}{*}{AS-SAC (no-living-cost)} & $R$ & {1895 ± 838} & {\bfseries 2722 ± 15} & {5314 ± 430} & {794 ± 196} & {3.3 ± 7.2} & {25.6 ± 0.3} & {26.2 ± 1.1} & {-0.6 ± 0.2} \\
    		& $C$ & 0.3 ± 0.1 & 0.4 ± 0.4 & 0.0 ± 0.4 & 3.1 ± 22.9 & 305 ± 67 & 64.0 ± 6.2 & 22.4 ± 36.9 & 86.9 ± 31.1 \\
    		\cline{1-10}
    		\bottomrule
    	\end{tabular}}
    \end{table}
    
    \begin{figure}
    	\centering
    	\includegraphics[width=0.49\linewidth]{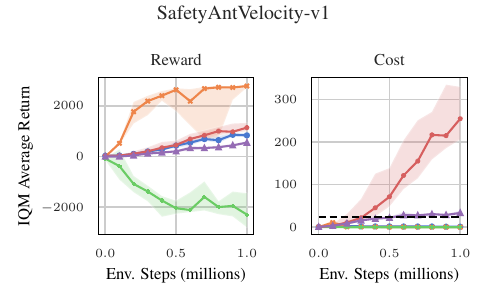}
    	\includegraphics[width=0.49\linewidth]{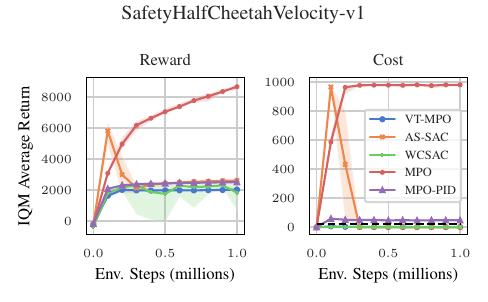}
    	\includegraphics[width=0.49\linewidth]{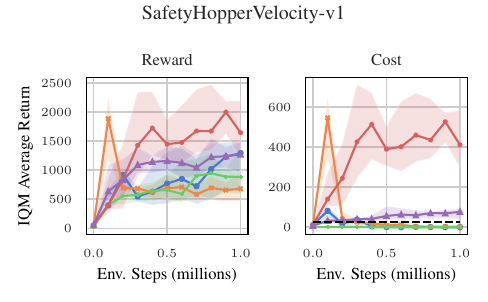}
    	\includegraphics[width=0.49\linewidth]{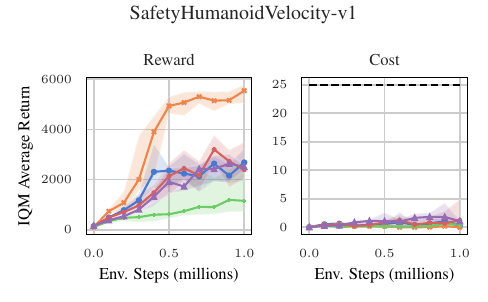}
    	\includegraphics[width=0.49\linewidth]{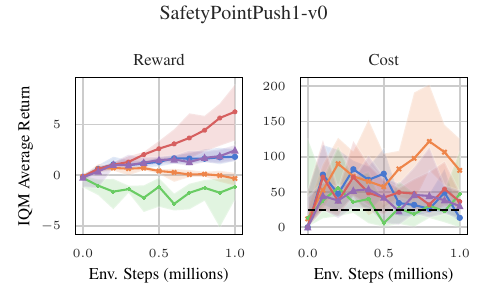}
    	\includegraphics[width=0.49\linewidth]{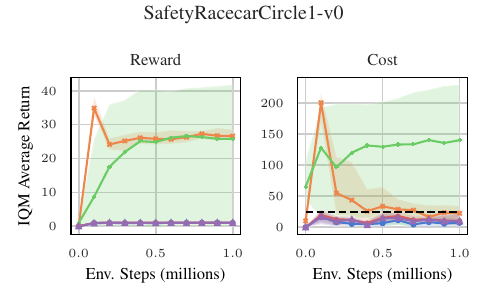}
    	\includegraphics[width=0.49\linewidth]{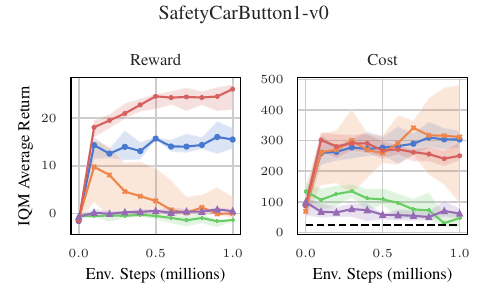}
    	\includegraphics[width=0.49\linewidth]{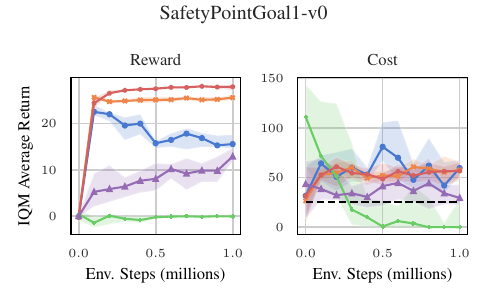}
    	\caption{Results on the Safety Gymnasium benchmark for individual environments.}
    	\label{fig:sample_efficiency_envs}
    \end{figure}

    \subsection{Compute and runtime}
    \label{app:safety-gym-compute}
    
    SDH does not introduce additional environment interaction or policy rollouts relative to the underlying actor-critic method. The only algorithmic change is the critic target through $(\tilde r,\tilde\gamma)$ and the computation of $\alpha(s,a)$ from constraint signals. The only extra critic appears in the survival-only diagnostic (App.~\ref{app:hyfydy-survival-endpoint}), where the survival critic $Q_1$ is trained alongside the task critic $Q_r$; no Lagrange dual update is used. For transparency, we report the compute used in Table~\ref{tab:safety-gym-compute-app}; we do not claim a systematic wall-clock advantage.
    
    \begin{table}[h]
    	\centering
    	\caption{Safety Gymnasium compute summary. All runs use NVIDIA A100 on our cluster. SDH modifies only the critic target and incurs no additional runtime overhead.}
    	\label{tab:safety-gym-compute-app}
    	\begin{tabular}{lc}
    		\toprule
    		Resource / setting & Value \\
    		\midrule
    		Hardware & NVIDIA A100 \\
    		Parallel environments & $10$ ($1$ for MPO) \\
    		Total environment steps & $10^6$ \\
    		Throughput (steps/s) & ${\sim}90$ \\
            Wallclock time (h) & ${\sim}3$ \\
    		\bottomrule
    	\end{tabular}
    \end{table}

    \clearpage
    \newpage

    \part{Theoretical Foundations}
    \parttoc 

    \section{Survival Occupancy and Stochastic Survival Chance Constraints}
	\label{app:survival-chance-proofs}
	
	This appendix formalizes the chance-constrained interpretation of stochastic decision horizons. 
    The purpose is not to reinterpret SDH as an ordinary additive-budget CMDP. 
    Rather, we identify the precise constrained-control problem for which SDH is exact: a survival-gated chance problem evaluated at the same random time that represents discounting.
    
    Figure~\ref{fig:budget-vs-viability-app} gives the conceptual map for the theory.
    The central distinction is that SDH treats violations as instantaneous-(in)feasibility signals rather than entries in a residual-budget ledger.
    This changes both the objective and the constraint: reward is counted only on prefixes that survive to the geometrically sampled evaluation time, and feasibility is summarized by the same survival event. 
    The resulting formulation is therefore exact for a stochastic survival-chance problem, but not for arbitrary additive-budget CMDPs.
    The violation-depth profile makes this boundary precise: additive cost depends on the full area under the profile, whereas exponential SDH observes one Laplace-weighted slice. Thus SDH is predictive when violation geometry is effectively single-scale, and structurally mismatched when budget use is broad, history-dependent, or multi-scale.
    
    \begin{figure*}[ht]
        \centering
        \includegraphics[width=\linewidth]{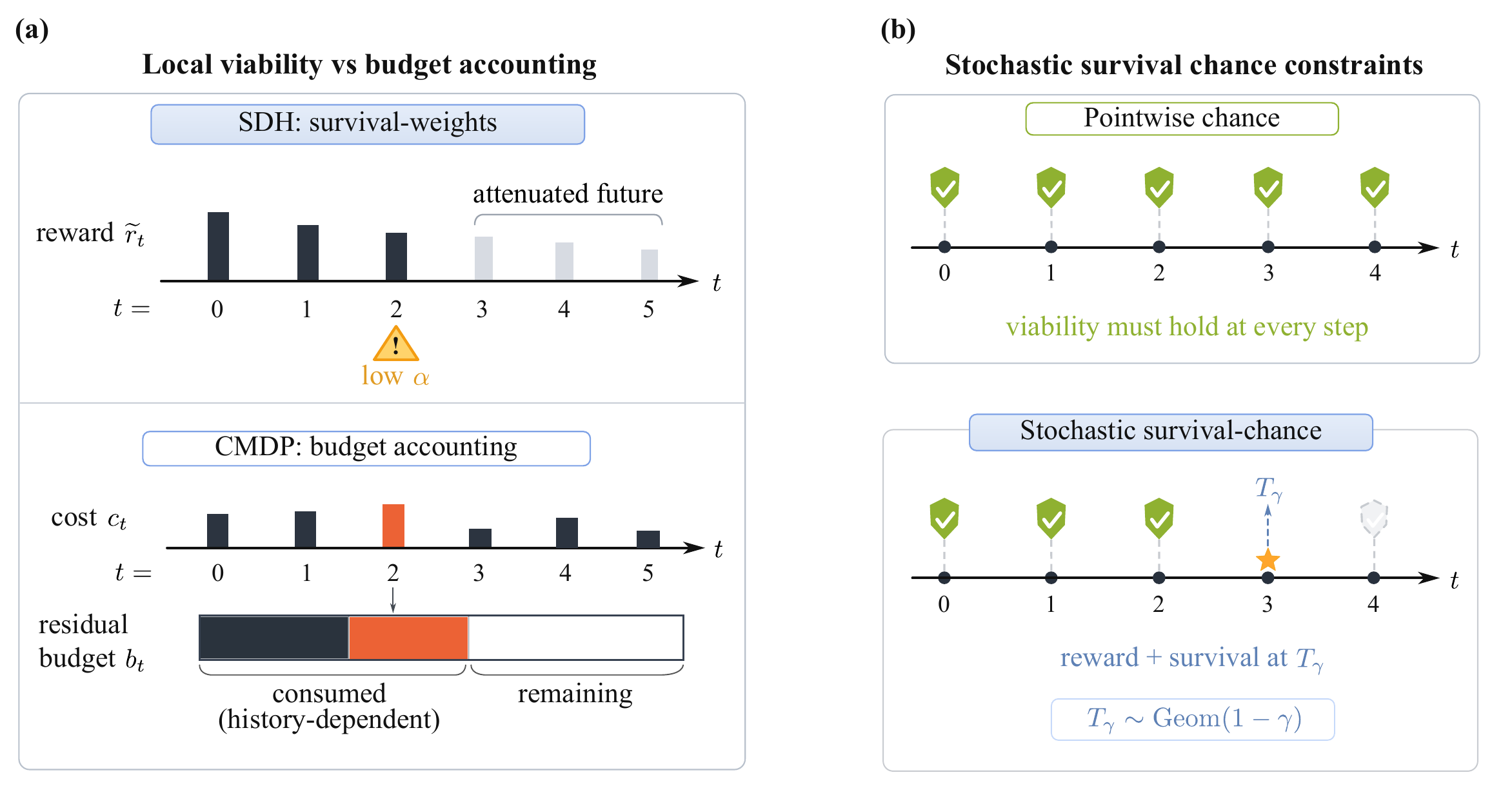}\\
        \includegraphics[width=\linewidth]{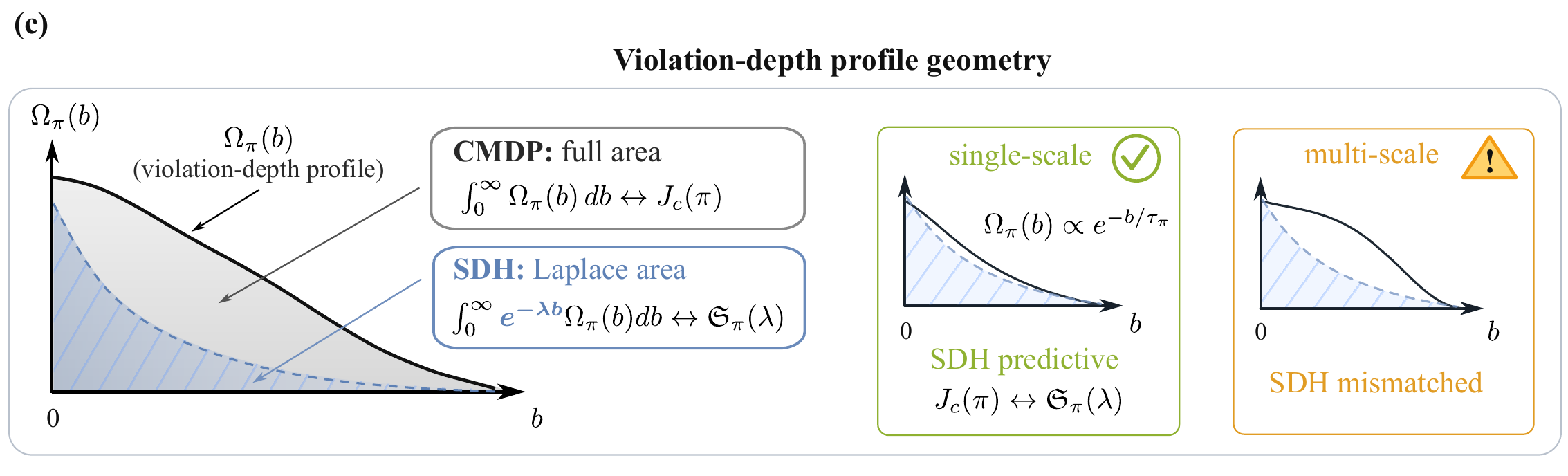}
        \caption{
            \textbf{Stochastic decision horizons as instantaneous-feasibility signals.}
            \textbf{(a)} CMDPs treat violations as additive cost in a history-dependent residual-budget ledger; SDH treats them as local continuation signals that attenuate reward and bootstrapping through $(\tilde r=\alpha r, \tilde\gamma=\gamma\alpha)$.
            \textbf{(b)} Unlike pointwise chance constraints, SDH evaluates reward and feasibility at the same geometrically sampled time $T_\gamma$, yielding a stochastic survival-chance objective.
            \textbf{(c)} The violation-depth profile explains the scope: additive cost is the full area under $\Omega_\pi(b)$, whereas exponential SDH observes one Laplace-weighted slice.
            Thus SDH is predictive in single-scale local-viability regimes, but not a general replacement for budget accounting.
        }
        \label{fig:budget-vs-viability-app}
    \end{figure*}
    
    The section has four steps. First, we relate SDH to previously studied survival chance constraints. Second, we show that SDH changes not only the
    constraint but also the objective: reward is counted only on surviving prefixes.
    Third, we prove the exact Lagrangian closure property of SDH. 
    Finally, we separate this formulation from additive-budget CMDPs and identify regimes where SDH is exact or informative.
	
	\subsection{Survival processes and chance constraints}
	
	We start with a formal survival process.
	
	\begin{definition}[Continuation model and survival process]
		\label{def:survival-process-app}
		Consider an infinite-horizon MDP
		$(\mathcal S,\mathcal A,P,r,\gamma)$ with $\gamma\in(0,1)$ and bounded reward.
		A continuation model is a measurable function
		$\alpha:\mathcal S\times\mathcal A\to[0,1]$.
		For a trajectory $\tau=(s_0,a_0,s_1,a_1,\ldots)$, define independent Bernoulli variables
        \begin{equation}
            C_t\mid\tau\sim \mathrm{Bernoulli}(\alpha(s_t,a_t)),
            \qquad
            B_t:=\prod_{k=0}^{t}C_k .
            \label{eq:survival-process-def-app}
        \end{equation}
		The event $B_t=1$ indicates the stochastic survival process has survived through the reward at time $t$.
	\end{definition}

    For a policy $\pi$, write $s_t(\pi):=\Pr_\pi(B_t=1)$.  A direct survival
	chance-constrained control problem requires these probabilities to remain high
	at every horizon.
	
	\begin{definition}[Pointwise survival chance problem]
		\label{def:pointwise-survival-chance-app}
		Given tolerances $(\delta_t)_{t\ge0}$, the pointwise survival chance problem is
		\begin{equation}
			\begin{aligned}
				\max_\pi\quad
				&J_\gamma(\pi)
				:=
				\mathbb E_\pi\!\left[\sum_{t\ge0}\gamma^t r(s_t,a_t)\right]
				\\
				\mathrm{s.t.}\quad
				&\Pr_\pi(B_t=1)\ge 1-\delta_t,
				\qquad t=0,1,2,\ldots .
			\end{aligned}
			\label{eq:pointwise-survival-chance-app}
		\end{equation}
	\end{definition}
	
	This formulation is multi-scale: it attaches a distinct feasibility requirement
	to every time horizon.  Discounted control, however, has a natural random-time
	representation.
	
	\begin{definition}[Discounted random time]
		\label{def:discounted-random-time-app}
		Let $T_\gamma$ be independent of the trajectory and survival process, with
		\begin{equation}
			\Pr(T_\gamma=t)=(1-\gamma)\gamma^t,
			\qquad t=0,1,2,\ldots .
			\label{eq:discounted-random-time-app}
		\end{equation}
	\end{definition}
	
	The random-time survival probability is exactly the geometrically weighted
	average of the pointwise survival probabilities.
	
	\begin{proposition}[Pointwise survival implies random-time survival]
		\label{prop:pointwise-implies-random-time-app}
		Let $w_t=(1-\gamma)\gamma^t$.  If
		$\Pr_\pi(B_t=1)\ge1-\delta_t$ for every $t\ge0$, then
		\begin{equation}
			\Pr_\pi(B_{T_\gamma}=1)
			\ge
			1-\sum_{t\ge0}w_t\delta_t .
			\label{eq:pointwise-random-time-bound-app}
		\end{equation}
		Thus the pointwise constraints imply the random-time survival constraint
		$\Pr_\pi(B_{T_\gamma}=1)\ge\rho$ whenever
		$\rho\le1-\sum_{t\ge0}(1-\gamma)\gamma^t\delta_t$.
		The converse does not hold in general.
	\end{proposition}
	
	\begin{proof}
		Independence of $T_\gamma$ gives
		$\Pr_\pi(B_{T_\gamma}=1)=\sum_{t\ge0}w_t\Pr_\pi(B_t=1)$.
		The lower bound follows immediately from the pointwise constraints.
		
		For non-converseness, fix $m$ and consider a deterministic one-action
		trajectory with $\alpha(s_t,a_t)=1$ for $t<m$ and
		$\alpha(s_m,a_m)=0$.  Then survival is certain before $m$ and impossible from
		$m$ onward.  Hence the pointwise constraints can fail for all $t\ge m$, while
		$\Pr_\pi(B_{T_\gamma}=1)=1-\gamma^m$, which can be arbitrarily close to one.
	\end{proof}
	
	This proposition identifies the constraint-side simplification made by SDH:
	it replaces infinitely many horizon-wise survival requirements by one discounted
	survival requirement.  This is not yet the full SDH formulation, because the
	objective is also changed.
	
	\subsection{The objective mismatch}
	
	The pointwise chance problem in Def.~\ref{def:pointwise-survival-chance-app}
	maximizes ordinary discounted reward.  In random-time form, ordinary discounted
	reward evaluates $r(s_{T_\gamma},a_{T_\gamma})$.  SDH instead evaluates reward
	only if the trajectory has survived until $T_\gamma$.
	
	\begin{definition}[Survival-gated random-time reward]
		\label{def:survival-gated-reward-app}
		The survival-gated random-time reward of a policy $\pi$ is
		\begin{equation}
			J_{\mathrm{surv}}(\pi)
			:=
			\mathbb E_\pi[
			B_{T_\gamma}r(s_{T_\gamma},a_{T_\gamma})
			].
			\label{eq:survival-gated-reward-app}
		\end{equation}
	\end{definition}
	
	\begin{proposition}[Objective mismatch]
		\label{prop:sdh-objective-mismatch-app}
		For every policy $\pi$,
		\begin{equation}
			(1-\gamma)J_\gamma(\pi)
			-
			J_{\mathrm{surv}}(\pi)
			=
			\mathbb E_\pi[
			(1-B_{T_\gamma})r(s_{T_\gamma},a_{T_\gamma})
			].
			\label{eq:objective-mismatch-app}
		\end{equation}
		If $0\le r(s,a)\le R$ for every state-action pair $(s,a)$, then
		\begin{equation}
			0
			\le
			(1-\gamma)J_\gamma(\pi)
			-
			J_{\mathrm{surv}}(\pi)
			\le
			R\,\Pr_\pi(B_{T_\gamma}=0).
			\label{eq:objective-mismatch-nonnegative-app}
		\end{equation}
	\end{proposition}
	
	\begin{proof}
		The random-time identity for discounted reward gives
		$(1-\gamma)J_\gamma(\pi)=\mathbb E_\pi[r(s_{T_\gamma},a_{T_\gamma})]$.
		Subtracting Def.~\ref{def:survival-gated-reward-app} proves
		\eqref{eq:objective-mismatch-app}.  If $0\le r\le R$, then
		$(1-B_{T_\gamma})r(s_{T_\gamma},a_{T_\gamma})$ is bounded between zero and
		$R(1-B_{T_\gamma})$, giving \eqref{eq:objective-mismatch-nonnegative-app}.
	\end{proof}
	
	Thus SDH is not merely a relaxation of pointwise chance constraints.  It makes a
	matched pair of changes: survival is evaluated at the discounted random time,
	and reward is gated by the same survival event.  This matching is exactly what
	produces the Lagrangian closure property below.
	
	\subsection{Survival occupancy and the exact SDH chance problem}\label{app:survival_chance_problem}
	
	We now define the invariant optimized by SDH.
	
	\begin{definition}[Survival-weighted occupancy]
		\label{def:survival-occupancy-app}
		For a bounded measurable test function $f$, define
		\begin{equation}
			\mu_\pi^\alpha(f)
			:=
			\mathbb E_\pi
			\left[
			\sum_{t\ge0}
			\gamma^t
			\left(
			\prod_{k=0}^{t}\alpha(s_k,a_k)
			\right)
			f(s_t,a_t)
			\right].
			\label{eq:survival-occupancy-app}
		\end{equation}
		The SDH return is $\mu_\pi^\alpha(r)$.
	\end{definition}
	
	The corresponding Bellman representation uses shaped reward
	$\tilde r(s,a)=\alpha(s,a)r(s,a)$ and variable discount
	$\tilde\gamma(s,a)=\gamma\alpha(s,a)$.  Since
	$\tilde\gamma(s,a)\le\gamma$, the induced policy-evaluation operator remains a
	$\gamma$-contraction for bounded rewards.
	
	\begin{theorem}[Survival occupancy identity]
		\label{thm:survival-occupancy-identity-app}
		For every bounded measurable $f$,
		\begin{equation}
			(1-\gamma)\mu_\pi^\alpha(f)
			=
			\mathbb E_\pi[
			B_{T_\gamma}f(s_{T_\gamma},a_{T_\gamma})
			].
			\label{eq:survival-occupancy-identity-app}
		\end{equation}
		In particular,
		$(1-\gamma)\mu_\pi^\alpha(1)=\Pr_\pi(B_{T_\gamma}=1)$.
	\end{theorem}
	
	\begin{proof}
		Conditioned on the trajectory, the survival variables are independent and
		$\mathbb E[B_t\mid\tau]=\prod_{k=0}^{t}\alpha(s_k,a_k)$.  Averaging over the
		geometric time $T_\gamma$ gives
		\[
		\mathbb E_\pi[
		B_{T_\gamma}f(s_{T_\gamma},a_{T_\gamma})
		]
		=
		(1-\gamma)
		\mathbb E_\pi
		\left[
		\sum_{t\ge0}
		\gamma^t
		\left(
		\prod_{k=0}^{t}\alpha(s_k,a_k)
		\right)
		f(s_t,a_t)
		\right],
		\]
		which is exactly \eqref{eq:survival-occupancy-identity-app}.  Taking $f\equiv1$ gives the survival
		probability identity.
	\end{proof}
	
	The exact constrained problem associated with SDH can now be stated.
	
	\begin{definition}[Survival-gated random-time chance problem]
		\label{def:survival-gated-random-time-problem-app}
		Given a survival threshold $\rho\in[0,1]$, the SDH chance problem is
		\begin{equation}
			\begin{aligned}
				\max_\pi\quad
				&
				\mathbb E_\pi[
				B_{T_\gamma}r(s_{T_\gamma},a_{T_\gamma})
				]
				\\
				\mathrm{s.t.}\quad
				&
				\Pr_\pi(B_{T_\gamma}=1)\ge\rho .
			\end{aligned}
			\label{eq:survival-gated-random-time-problem-app}
		\end{equation}
	\end{definition}
	
	This formulation evaluates feasibility and reward at the same random decision
	time.  That shared timing is the structural reason the Lagrangian remains inside
	the SDH objective class.
	
	\begin{theorem}[Exact SDH Lagrangian]
		\label{thm:survival-lagrangian-app}
		The Lagrangian of Def.~\ref{def:survival-gated-random-time-problem-app}
		is
		\begin{equation}
			\mathcal L(\pi,\eta)
			=
			(1-\gamma)\mu_\pi^\alpha(r+\eta)-\eta\rho,
			\qquad \eta\ge0 .
			\label{eq:survival-lagrangian-app}
		\end{equation}
		Consequently, for fixed $\eta$, maximizing the Lagrangian over policies is
		equivalent to maximizing an SDH return with shifted reward $r+\eta$.
	\end{theorem}
	
	\begin{proof}
		The Lagrangian is
		$\mathbb E_\pi[B_{T_\gamma}r(s_{T_\gamma},a_{T_\gamma})]
		+\eta(\Pr_\pi(B_{T_\gamma}=1)-\rho)$.
		Applying Thm.~\ref{thm:survival-occupancy-identity-app} with $f=r$ and
		with $f\equiv1$ gives
		$\mathbb E_\pi[B_{T_\gamma}r_{T_\gamma}]=(1-\gamma)\mu_\pi^\alpha(r)$ and
		$\Pr_\pi(B_{T_\gamma}=1)=(1-\gamma)\mu_\pi^\alpha(1)$.  Linearity of
		$\mu_\pi^\alpha$ then yields \eqref{eq:survival-lagrangian-app}.
	\end{proof}
	
	The dual variable $\eta$ is therefore a survival bonus, not an additive cost
	penalty.  At the Bellman level, the Lagrangian critic uses
	$\tilde r_\eta(s,a)=\alpha(s,a)(r(s,a)+\eta)$ and
	$\tilde\gamma(s,a)=\gamma\alpha(s,a)$.  This is the SDH Lagrange closure
	property.
	
	\subsection{Two interpretable continuation models}
	
	The previous results hold for any continuation model. We now record two
    important special cases constructed from a per-step violation magnitude
    $c(s,a)\ge 0$. Throughout the remainder of this appendix, we write
    \[
        G_t:=\sum_{k=0}^{t} c(s_k,a_k)
    \]
    for cumulative violation depth. We reserve $C_t$ for the Bernoulli
    continuation indicator, see \eqref{eq:survival-process-def-app}.
	
	\begin{definition}[Hard continuation]
		\label{def:hard-continuation-app}
		For a binary violation signal $c:\mathcal S\times\mathcal A\to\{0,1\}$, hard
		continuation is the model
		\begin{equation}
			\alpha_\infty(s,a)=\mathbf 1\{c(s,a)=0\}.
			\label{eq:hard-continuation-app}
		\end{equation}
	\end{definition}
	
    \begin{proposition}[Hard SDH is exact for no-violation chance constraints]
        \label{prop:hard-no-violation-app}
        Under hard continuation,
        $B_t=\mathbf 1\{G_t=0\}$. Therefore, for every fixed $t$,
        $\Pr_\pi(B_t=1)\ge1-\delta_t$ is equivalent to
        $\Pr_\pi(\exists k\le t:c(s_k,a_k)=1)\le\delta_t$. At the discounted random
        time,
        $\Pr_\pi(B_{T_\gamma}=1)=\Pr_\pi(G_{T_\gamma}=0)$.
    \end{proposition}
	
	\begin{proof}
		Under \eqref{eq:hard-continuation-app}, survival at time $k$ is deterministic given the trajectory
		and occurs exactly when $c(s_k,a_k)=0$.  Therefore $B_t=1$ exactly when no
		violation occurs through time $t$.
	\end{proof}
	
	\begin{definition}[Exponential continuation]
		\label{def:exp-continuation-app}
		For a nonnegative violation magnitude $c:\mathcal S\times\mathcal A\to
		\mathbb R_{\ge0}$ and scale $\lambda>0$, exponential continuation is
		\begin{equation}
			\alpha_\lambda(s,a)=\exp(-\lambda c(s,a)).
			\label{eq:exp-continuation-app}
		\end{equation}
	\end{definition}
	
	\begin{proposition}[Exponential continuation is cumulative-depth survival]
        \label{prop:exp-survival-app}
        Under exponential continuation,
        \begin{equation}
            \Pr_\pi(B_{T_\gamma}=1)
            =
            \mathbb E_\pi[e^{-\lambda G_{T_\gamma}}].
            \label{eq:exp-survival-random-time-app}
        \end{equation}
    \end{proposition}
    
    \begin{proof}
        Conditioned on the trajectory,
        $\mathbb E[B_t\mid\tau]
        =\prod_{k=0}^{t}\exp(-\lambda c(s_k,a_k))
        =\exp(-\lambda G_t)$.
        Evaluating at $T_\gamma$ and taking expectation proves
        the claim.
    \end{proof}
	
	Thus exponential SDH observes a Laplace transform of cumulative violation depth
	at the discounted random time.  This statistic gives a one-sided chance
	certificate.
	
	\begin{proposition}[Cumulative-depth chance diagnostic]
        \label{prop:chance-bound-exp-app}
        Let $\bar S_\pi(\lambda):=
        \mathbb E_\pi[e^{-\lambda G_{T_\gamma}}]$.  For
        every $b>0$,
        \begin{equation}
            \Pr_\pi(G_{T_\gamma}\ge b)
            \le
            \frac{1-\bar S_\pi(\lambda)}{1-e^{-\lambda b}}.
            \label{eq:cumulative-depth-diagnostic-app}
        \end{equation}
        Consequently,
        $\bar S_\pi(\lambda)\ge1-\delta(1-e^{-\lambda b})$ implies
        $\Pr_\pi(G_{T_\gamma}\ge b)\le\delta$.
    \end{proposition}
    
    \begin{proof}
        Let $Y=1-e^{-\lambda G_{T_\gamma}}$.
        If $G_{T_\gamma}\ge b$, then
        $Y\ge1-e^{-\lambda b}$.
        Markov's inequality gives the result since
        $\mathbb E[Y]=1-\bar S_\pi(\lambda)$.
    \end{proof}
	
	The certificate is useful but limited: one exponential scale does not determine
	an arbitrary cumulative-cost profile.  This is the first place where SDH
	separates from additive-budget CMDPs.
	
	\subsection{Why local SDH is not an additive-budget CMDP solver}\label{app:sdh-vs-cmdp}
	
    The first is history dependence.
    While CMDPs constrain expected cumulative cost, exact budget accounting is Markov only after augmenting the state with the spent cost or remaining budget. 
    Thus any local continuation rule that tries to represent pathwise budget feasibility must depend on residual-budget information, whereas $\alpha(s,a)$ is history-blind.
	
	\begin{theorem}[History-aliasing obstruction]
		\label{thm:history-aliasing-app}
		There exists a finite deterministic constrained-control problem with nonnegative costs and budget $d>0$ such that no history-blind continuation model $\alpha:\mathcal S\times\mathcal A\to[0,1]$ can agree with the exact residual-budget continuation rule
		\begin{equation}
            \alpha_d^\star(h_t,a)
            :=
            \mathbf 1\{G_{t-1}(h_t)+c(s_t,a)\le d\}
            \label{eq:budget-aware-continuation-app}
        \end{equation}
		on all reachable histories.
	\end{theorem}
	
	\begin{proof}
		Consider a deterministic two-stage problem.  From $s_0$, action $a_L$ has cost zero and action $a_H$ has cost one; both lead to the same state $x$.  
        At $x$, action $a_{\mathrm{risk}}$ has cost one.  Set the budget to $d=1$.
		
		There are two histories reaching $x$: one with spent cost zero and one with spent cost one.
        The residual-budget rule allows $a_{\mathrm{risk}}$ after the first history and forbids it after the second.
        Both decisions correspond to the same state-action pair $(x,a_{\mathrm{risk}})$, so no single history-blind value $\alpha(x,a_{\mathrm{risk}})$ can match both.
	\end{proof}
	
	The second obstruction is functional.  Additive cost and exponential survival
	probe different summaries of the same cumulative-depth profile.
	
	\begin{definition}[Violation-depth profile]
        \label{def:violation-depth-profile-app}
        For a fixed policy $\pi$, define
        \begin{equation}
            \Omega_\pi(b)
            :=
            \mathbb E_\pi
            \left[
            \sum_{t\ge0}\gamma^t\mathbf 1\{G_t\ge b\}
            \right],
            \qquad b\ge0 .
            \label{eq:violation-depth-profile-app}
        \end{equation}
    \end{definition}
	
	\begin{proposition}[Profile identities]
		\label{prop:profile-identities-app}
		For a policy $\pi$, define the expected discounted cumulative cost and the exponential survival statistic by
        \[
        J_c(\pi):=\mathbb E_\pi\Big[\sum_{t\ge0}\gamma^t c(s_t,a_t)\Big],\qquad \mathfrak S_\pi(\lambda):=\mathbb E_\pi\Big[\sum_{t\ge0}\gamma^t e^{-\lambda G_t}\Big].
        \]
        Then
		\begin{equation}
			\frac{J_c(\pi)}{1-\gamma}
			=
			\int_0^\infty\Omega_\pi(b)\,db,
			\qquad
			\frac{1}{\lambda}
			\left(
			\frac{1}{1-\gamma}
			-
			\mathfrak S_\pi(\lambda)
			\right)
			=
			\int_0^\infty e^{-\lambda b}\Omega_\pi(b)\,db .
			\label{eq:profile-identities-app}
		\end{equation}
	\end{proposition}
	
	\begin{proof}
		The first identity follows from the tail-integral representation $x=\int_0^\infty\mathbf 1\{x\ge b\}\,db$ applied to $G_t$, together with $\sum_{t\ge0}\gamma^t G_t=(1-\gamma)^{-1}\sum_{t\ge0}\gamma^t c(s_t,a_t)$. The second follows similarly from
		$(1-e^{-\lambda x})/\lambda=\int_0^\infty e^{-\lambda b}\mathbf 1\{x\ge b\}\,db$.
	\end{proof}
	
	Thus additive CMDP cost is the total area under the violation-depth profile,
	whereas exponential SDH observes a Laplace-weighted slice of the same profile.
	
	\begin{proposition}[One exponential statistic does not control arbitrary additive cost]
    \label{prop:one-slice-obstruction-app}
        Fix $\gamma\in(0,1)$, $\lambda_0>0$, and $s\in(0,1)$.
        For every $K>0$, there exists a finite MDP with nonnegative costs and a
        stationary randomized policy $\pi$ such that
        \begin{equation}
            \bar S_\pi(\lambda_0)
            :=
            \mathbb E_\pi[e^{-\lambda_0 G_{T_\gamma}}]
            =
            s,
            \qquad
            J_c(\pi)>K .
        \end{equation}
        Hence one exponential SDH statistic cannot in general determine or
        upper-bound expected discounted cumulative cost.
    \end{proposition}
    
    \begin{proof}
        Consider an initial state $s_0$ with two actions. Action $a_0$ has cost
        zero and action $a_M$ has cost $M>0$; both transition to an absorbing
        zero-cost state. Let
        \[
            \pi_M(a_M\mid s_0)
            =
            q_M
            :=
            \frac{1-s}{1-e^{-\lambda_0 M}},
        \]
        where $M\ge\lambda_0^{-1}\log(1/s)$ so that $q_M\in[0,1]$.
        Since all cost is incurred at time zero,
        $G_{T_\gamma}\in\{0,M\}$ with probabilities $1-q_M$ and $q_M$.
        Therefore
        \[
            \bar S_{\pi_M}(\lambda_0)
            =
            1-q_M(1-e^{-\lambda_0M})
            =
            s .
        \]
        On the other hand,
        \[
            J_c(\pi_M)=q_M M
            =
            \frac{(1-s)M}{1-e^{-\lambda_0M}}
            \ge (1-s)M .
        \]
        Choosing
        $M>\max\{\lambda_0^{-1}\log(1/s),K/(1-s)\}$
        gives $J_c(\pi_M)>K$ while
        $\bar S_{\pi_M}(\lambda_0)=s$.
    \end{proof}
	
	These results explain the scope of SDH.  Local continuation cannot reproduce
	residual-budget semantics without augmenting the state, and one exponential
	scale cannot recover an arbitrary additive-cost profile.
	
    \subsection{A positive regime: single-scale survival}
    \label{subsec:single-scale-viability-app}
    
    The preceding obstruction is a worst-case statement: one exponential SDH scale
    cannot identify an arbitrary violation-depth profile.  The obstruction disappears
    when the profile has only one effective depth parameter.  In that case, the
    additive CMDP cost, the random-time chance tail, and the exponential SDH statistic
    are all different views of the same scalar.
    
    \begin{theorem}[Single-scale recoverability]
    	\label{thm:single-scale-app}
    	Fix a policy $\pi$.  Suppose that, for some $\tau_\pi>0$,
    	\begin{equation}
    		\Omega_\pi(b)=\frac{1}{1-\gamma}e^{-b/\tau_\pi},
    		\qquad b\ge0 .
    		\label{eq:single-scale-profile-app}
    	\end{equation}
    	Equivalently, $\Pr_\pi(G_{T_\gamma}\ge b)=e^{-b/\tau_\pi}$. Then, for every
    	$\lambda>0$,
    	\begin{equation}
    		J_c(\pi)=\tau_\pi,
    		\qquad
    		\mathfrak S_\pi(\lambda)
    		=
    		\frac{1}{(1-\gamma)(1+\lambda\tau_\pi)},
    		\qquad
    		\bar S_\pi(\lambda)
    		=
    		\frac{1}{1+\lambda\tau_\pi}.
    		\label{eq:single-scale-recovery-app}
    	\end{equation}
    	Consequently, for any fixed $\lambda_0>0$,
    	\begin{equation}
    		J_c(\pi)
    		=
    		\frac{1}{\lambda_0}
    		\left(
    		\frac{1}{(1-\gamma)\mathfrak S_\pi(\lambda_0)}-1
    		\right)
    		=
    		\frac{1}{\lambda_0}
    		\left(
    		\frac{1}{\bar S_\pi(\lambda_0)}-1
    		\right).
    		\label{eq:single-slice-recovery-app}
    	\end{equation}
    	Moreover, for any $b>0$ and $\delta\in(0,1)$,
    	$\Pr_\pi(G_{T_\gamma}\ge b)\le\delta$ is equivalent to
    	\begin{equation}
    		\bar S_\pi(\lambda)
    		\ge
    		\frac{1}{1+\lambda b/\log(1/\delta)} .
    		\label{eq:single-scale-threshold-app}
    	\end{equation}
    \end{theorem}
    
    \begin{proof}
    	The equivalence between \eqref{eq:single-scale-profile-app} and the random-time
    	tail follows from
    	\[
            \Pr_\pi(G_{T_\gamma}\ge b)
            =
            (1-\gamma)\Omega_\pi(b).
        \]
    	Using Prop.~\ref{prop:profile-identities-app},
    	\[
    		\frac{J_c(\pi)}{1-\gamma}
    		=
    		\int_0^\infty \Omega_\pi(b)\,db
    		=
    		\frac{\tau_\pi}{1-\gamma},
    	\]
    	so $J_c(\pi)=\tau_\pi$.  The same profile identity gives
    	\[
    		\frac{1}{\lambda}
    		\left(
    		\frac{1}{1-\gamma}-\mathfrak S_\pi(\lambda)
    		\right)
    		=
    		\int_0^\infty e^{-\lambda b}\Omega_\pi(b)\,db
    		=
    		\frac{\tau_\pi}{(1-\gamma)(1+\lambda\tau_\pi)} ,
    	\]
    	which rearranges to the expression for $\mathfrak S_\pi(\lambda)$.  Since
    	$\bar S_\pi(\lambda)=(1-\gamma)\mathfrak S_\pi(\lambda)$, this also gives the
    	expression for $\bar S_\pi(\lambda)$ and its inversion in
    	\eqref{eq:single-slice-recovery-app}.  Finally,
    	$e^{-b/\tau_\pi}\le\delta$ is equivalent to
    	$\tau_\pi\le b/\log(1/\delta)$; substituting this into
    	$\bar S_\pi(\lambda)=(1+\lambda\tau_\pi)^{-1}$ yields
    	\eqref{eq:single-scale-threshold-app}.
    \end{proof}
    
    This theorem is the positive counterpart to
    Prop.~\ref{prop:one-slice-obstruction-app}.  For arbitrary profiles, one
    exponential slice is only one Laplace-weighted summary.  For the single-scale
    family \eqref{eq:single-scale-profile-app}, the total area under the profile, the
    random-time chance tail, and the exponential SDH statistic are all controlled by
    the same depth parameter $\tau_\pi$.  Exponential SDH is therefore exact in this
    regime not because it solves additive-budget CMDPs in general, but because the
    task's violation geometry has collapsed to one effective scale.
	
	\subsection{Interpretation and scope}
	
	The appendix establishes the following formal picture.  Standard survival chance
	constraints are horizon-wise.  SDH replaces them by survival at the discounted
	random time $T_\gamma$.  This is a weaker, geometrically averaged feasibility
	condition.  The objective is also changed: reward is survival-gated at the same
	random time.
	
	Because both reward and feasibility are represented by the same survival
	occupancy measure, the Lagrangian of the survival-gated random-time chance
	problem closes exactly as an SDH objective with shifted reward.  This closure is
	the precise constrained-control interpretation of SDH.
	
	The same analysis also clarifies scope.  SDH is exact for hard no-violation
	chance constraints and for the survival-gated random-time formulation.  It is
	not an exact additive-budget CMDP solver on the original state space, because
	budget feasibility is history-dependent and additive costs depend on the full
	violation-depth profile.  Exponential SDH is therefore best understood as a
	single-scale survival model: limited for arbitrary profiles, but
	exactly informative when violation depth is governed by one effective scale.

    \newpage

    \section{Control as Inference with Stochastic Decision Horizons}
    \label{sec:cai-stoch-horizons}
    
    The preceding survival-constrained formulation identifies the control problem for which SDH is exact. Its Lagrangian is
    \[
    \mathcal L(\pi,\eta)
    =
    (1-\gamma)\mu^\pi_\alpha(r+\eta)-\eta\rho.
    \]
    After division by $1-\gamma$, fixed-$\eta$ policy optimization is SDH with reward shift $r\mapsto r+\eta$, plus the policy-independent constant $-\eta\rho/(1-\gamma)$. Thus $\eta$ remains the Lagrange multiplier of the survival constraint; through the constant reward shift it acts as a uniform survival incentive in the transformed reward.
    
    For the CaI construction we fix an arbitrary bounded reward and write it as $r$. All statements below extend pointwise to any fixed finite $\eta$ by replacing $r$ with $r+\eta$. Crucially, we apply inference to the same geometric evaluation time and survival event (\Cref{def:survival-chance-main}) used in the primal problem \eqref{eq:survival-chance-main}. The pathwise energy is the scaled endpoint reward
    \[
    \widehat R_\gamma
    =
    \frac{B_{T_\gamma}r_{T_\gamma}}{1-\gamma}.
    \]
    This is not a pathwise rewriting of a geometrically stopped prefix return: the two random variables differ. What is preserved exactly is the policy-level SDH reward functional. The endpoint construction has expectation $\mu^\pi_\alpha(r)$ and is the direct variational representation of the random-horizon reward in \Cref{def:survival-chance-main}.
    
    The remaining CaI question is therefore entirely about policy information cost. A weighted log-ratio is legitimate only when it is the trajectory KL between actual probability laws. Under absorbing-state semantics, continuation failure ends the controlled decision process. Under virtual termination (objective-level reward termination without environment reset), the controller continues to act after reward credit has terminated. The two semantics share the endpoint optimality potential and differ only in which action likelihood factors remain active.
	
	\subsection{The ELBO constraint}
	
	We use the Gibbs variational principle in a form that keeps the trajectory measures explicit. Let \(X\) collect the random variables of the trajectory model, let \(p\) be the reference trajectory law, and let \(q\ll p\) be a variational trajectory law. For the unnormalized optimality potential
    \[
    \psi(X)=\exp\!\Big(\frac{R(X)}{\kappa}\Big),
    \qquad \kappa>0,
    \]
    define the evidence \(Z:=\mathbb E_p[\psi(X)]\). Assume \(0<Z<\infty\), \(\mathbb E_q[|R(X)|]<\infty\), and \(D_{\mathrm{KL}}(q\|p)<\infty\). The corresponding Gibbs posterior is
    the probability law
    \[
    \frac{dp^\star}{dp}(X)
    =
    \frac{1}{Z}
    \exp\!\Big(\frac{R(X)}{\kappa}\Big).
    \]
    Since its density with respect to \(p\) is strictly positive, \(q\ll p\) implies \(q\ll p^\star\). Expanding \(D_{\mathrm{KL}}(q\|p^\star)\) then gives
    \[
    D_{\mathrm{KL}}(q\|p^\star)
    =
    D_{\mathrm{KL}}(q\|p)
    -\frac{1}{\kappa}\mathbb E_q[R(X)]
    +\log Z,
    \]
    or, equivalently,
    \[
    \kappa\log Z
    =
    \mathbb E_q[R(X)]
    -\kappa D_{\mathrm{KL}}(q\|p)
    +\kappa D_{\mathrm{KL}}(q\|p^\star).
    \]
    By nonnegativity of KL divergence,
    \[
    \kappa\log Z
    \ge
    \mathbb E_q[R(X)]
    -\kappa D_{\mathrm{KL}}(q\|p).
    \]
    Thus the CaI objective is the standard variational lower bound on the log evidence: reward favors trajectories of high optimality potential, while the trajectory KL measures their departure from the reference law.
	
	\subsection{Stochastic horizon, continuation, and endpoint reward}
	
	Consider an infinite-horizon discounted MDP $(\mathcal S,\mathcal A,P,r,\gamma)$, with $\gamma\in(0,1)$, bounded reward, reference policy $\pi_0(a\mid s)$, and candidate policy $\pi(a\mid s)$. For $\tau=(s_0,a_0,s_1,a_1,\ldots)$, write $r_t:=r(s_t,a_t)$ and $\alpha_t:=\alpha(s_t,a_t)$.

    \begin{definition}[Geometric evaluation horizon]
    \label{def:sdh-geom-horizon}
    Let $T_\gamma$ be independent of the MDP trajectory and continuation variables, with
    \[
    \Pr(T_\gamma=t)=(1-\gamma)\gamma^t,
    \qquad t\in\{0,1,\ldots\},
    \]
    and define $H_t:=\mathbf 1\{T_\gamma\ge t\}$. Then $\mathbb E[H_t]=\gamma^t$.
    \end{definition}
    
    \begin{definition}[Continuation process]
    \label{def:sdh-continuation-process}
    Given the model-defined continuation probability $\alpha:\mathcal S\times\mathcal A\to[0,1]$, let
    \[
    C_t\mid s_t,a_t\sim\mathrm{Bernoulli}(\alpha(s_t,a_t)),
    \qquad
    B_t:=\prod_{k=0}^{t}C_k.
    \]
    Thus $B_t=1$ iff the model-defined survival process continues through the state-action pair at time $t$.
    \end{definition}
    
    \begin{definition}[Geometric-endpoint reward]
    \label{def:sdh-endpoint-reward}
    Define
    \[
    \widehat R_\gamma
    :=
    \frac{B_{T_\gamma}r_{T_\gamma}}{1-\gamma},
    \qquad r_t:=r(s_t,a_t).
    \]
    The factor $(1-\gamma)^{-1}$ is a normalization convention chosen so that the expected reward term is exactly the unnormalized SDH return $\mu^\pi_\alpha(r)$ used by the Bellman equations. Omitting this factor would merely require the corresponding rescaling of the CaI temperature. With the present convention, \Cref{thm:survival-lagrangian-main} gives $\mathbb E_\pi[\widehat R_\gamma]=\mu^\pi_\alpha(r)$ under the ordinary $\pi$-controlled trajectory law.
    \end{definition}
    
    \begin{definition}[Active decision gates]
    \label{def:sdh-decision-gates}
    The active decision indicators are
    \[
    A_t^{\mathrm{AS}}
    :=
    H_t\prod_{k=0}^{t-1}C_k,
    \qquad
    A_t^{\mathrm{VT}}
    :=
    H_t,
    \]
    with the empty product equal to one. Under AS semantics, no policy decision is active after the first continuation failure. Under VT semantics, policy decisions remain active until the geometric evaluation horizon, even after reward credit has terminated at the objective level.
    \end{definition}
    
    Notably, there is a one-step offset between $B_t$ and $A_t^{\mathrm{AS}}$. The action $a_t$ is selected before $C_t$ is realized. Thus $C_t$ enters $B_t$ and determines whether an endpoint at time $t$ or later receives reward credit, but it does not gate the likelihood factor for the already selected action $a_t$.
	
    \subsection{Semantic trajectory laws}
    
    Choose a fixed inactive action distribution $\lambda_\bot$ on $\mathcal A$, shared by the variational and prior laws. Define the semantic variational action kernel
    \[
    K_{q,\sigma,t}(\cdot\mid s_t,A_t^\sigma)
    :=
    \begin{cases}
    \pi(\cdot\mid s_t),&A_t^\sigma=1,\\
    \lambda_\bot(\cdot),&A_t^\sigma=0,
    \end{cases}
    \]
    and the semantic prior action kernel
    \[
    K_{p,\sigma,t}(\cdot\mid s_t,A_t^\sigma)
    :=
    \begin{cases}
    \pi_0(\cdot\mid s_t),&A_t^\sigma=1,\\
    \lambda_\bot(\cdot),&A_t^\sigma=0.
    \end{cases}
    \]
    
    The semantic laws $q_\sigma$ and $p_\sigma$ first sample the same $T_\gamma$ and $s_0\sim\nu$. For each $t=0,\ldots,T_\gamma$, they sample action $a_t$ from the corresponding kernel above and then sample continuation $C_t\sim\mathrm{Bernoulli}(\alpha(s_t,a_t))$; when $t<T_\gamma$, they also sample next-state $s_{t+1}\sim P(\cdot\mid s_t,a_t)$. After an AS continuation failure, both laws use the same inactive action kernel for the remaining bookkeeping variables. Those variables cannot affect the endpoint return because then $B_{T_\gamma}=0$.
    
    Thus the two laws agree on the geometric-horizon distribution, initial-state distribution, environment dynamics, continuation kernels, and inactive action kernel. They differ only at active decision times. Assume $\pi(\cdot\mid s)\ll\pi_0(\cdot\mid s)$ on active states and
    \[
    \mathbb E_{q_\sigma}\!\Big[
    \sum_{t\ge0}A_t^\sigma
    \Big|\log\frac{\pi(a_t\mid s_t)}{\pi_0(a_t\mid s_t)}\Big|
    \Big]<\infty.
    \]
    
    \begin{lemma}[Semantic trajectory KL]
    \label{lem:sdh-semantic-kl}
    For either semantics $\sigma\in\{\mathrm{AS},\mathrm{VT}\}$,
    \[
    D_{\mathrm{KL}}(q_\sigma\|p_\sigma)
    =
    \mathbb E_{q_\sigma}\!\Big[
    \sum_{t\ge0}A_t^\sigma
    \log\frac{\pi(a_t\mid s_t)}{\pi_0(a_t\mid s_t)}
    \Big].
    \]
    \end{lemma}
    
    \begin{proof}
    Condition on $T_\gamma=m$. The two laws have the same horizon mass, initial-state law, transition kernels, continuation kernels, and inactive action kernel. These factors cancel in the likelihood ratio. At each $t\le m$, the action factor also cancels when $A_t^\sigma=0$; when $A_t^\sigma=1$, its ratio is $\pi(a_t\mid s_t)/\pi_0(a_t\mid s_t)$. Hence
    \[
    \log\frac{q_\sigma}{p_\sigma}
    =
    \sum_{t=0}^{m}A_t^\sigma
    \log\frac{\pi(a_t\mid s_t)}{\pi_0(a_t\mid s_t)}.
    \]
    Taking expectation over the shared geometric-horizon law gives the result. The stated absolute-integrability assumption justifies the expectation of the random finite sum.
    \end{proof}
    
    Crucially, the gated log-ratio is not an adhoc regularizer: it is exactly the trajectory KL induced by semantic laws that differ only at active decision times.
	
	\subsection{The SDH ELBO}\label{app:sdh-elbo}
	
	For either semantics define the endpoint optimality potential and evidence by
    \[
    \psi_\sigma
    :=
    \exp\!\Big\{\frac{\widehat R_\gamma}{\kappa}\Big\},
    \qquad
    Z_\sigma:=\mathbb E_{p_\sigma}[\psi_\sigma].
    \]
    
    \begin{theorem}[Semantic endpoint SDH ELBO]
    \label{thm:sdh-semantic-elbo}
    Assume $\lVert r\rVert_\infty\le R_{\max}<\infty$. For either semantics $\sigma\in\{\mathrm{AS},\mathrm{VT}\}$ and every $\kappa>0$,
    \[
    0<Z_\sigma
    \le
    \exp\!\Big\{\frac{R_{\max}}{\kappa(1-\gamma)}\Big\}
    <\infty,
    \]
    and
    \[
    \kappa\log Z_\sigma
    \ge
    \mathcal J_\sigma(\pi,\pi_0),
    \]
    where
    \[
    \mathcal J_\sigma(\pi,\pi_0)
    :=
    \mathbb E_{q_\sigma}\!\Big[
    \widehat R_\gamma
    -\kappa\sum_{t\ge0}A_t^\sigma
    \log\frac{\pi(a_t\mid s_t)}{\pi_0(a_t\mid s_t)}
    \Big].
    \]
    \end{theorem}
    
    \begin{proof}
    Since $B_{T_\gamma}\in\{0,1\}$,
    \[
    \widehat R_\gamma
    \le
    \frac{\lVert r\rVert_\infty}{1-\gamma}
    \le
    \frac{R_{\max}}{1-\gamma}
    \]
    pathwise. Therefore $0<\psi_\sigma\le\exp\{R_{\max}/[\kappa(1-\gamma)]\}$, which proves the evidence bound. Applying the generic variational identity to $R=\widehat R_\gamma$ gives
    \[
    \kappa\log Z_\sigma
    \ge
    \mathbb E_{q_\sigma}[\widehat R_\gamma]
    -\kappa D_{\mathrm{KL}}(q_\sigma\|p_\sigma).
    \]
    Lemma~\ref{lem:sdh-semantic-kl} yields the stated objective.
    \end{proof}
    
    For any fixed finite $\eta\ge0$, replace $r$ by $r+\eta$. Then
    \[
    Z_{\sigma,\eta}
    \le
    \exp\!\Big\{
    \frac{\lVert r\rVert_\infty+\eta}{\kappa(1-\gamma)}
    \Big\}<\infty.
    \]
    After the endpoint reduction below, subtracting $\eta\rho/(1-\gamma)$ gives exactly the scaled fixed-$\eta$ Lagrangian from \Cref{thm:survival-lagrangian-main}, together with the semantic policy-KL cost.
	
	\subsection{Reducing the endpoint and decision gates}
	
	Let $\mathbb P_\pi$ denote the joint law generated by the policy $\pi$, the environment dynamics, the independent geometric horizon $T_\gamma$, and continuation variables that are conditionally independent given the state-action trajectory $\tau$, with $C_t\mid\tau\sim\mathrm{Bernoulli}\alpha(s_t,a_t))$. Expectations over its state-action marginal are denoted by
    $\mathbb E_{\tau\sim\pi}$.

    \begin{lemma}[Endpoint and active-decision transfer]
    \label{lem:sdh-gated-prefix-transfer}
    Let $f_t$ be an integrable function of the trajectory prefix up to $(s_t,a_t)$. For either semantics $\sigma\in\{\mathrm{AS},\mathrm{VT}\}$,
    \[
    \mathbb E_{q_\sigma}[B_{T_\gamma}f_{T_\gamma}]
    =
    \mathbb E_{\mathbb P_\pi}[B_{T_\gamma}f_{T_\gamma}],
    \]
    and, for every $t\ge0$,
    \[
    \mathbb E_{q_\sigma}[A_t^\sigma f_t]
    =
    \mathbb E_{\mathbb P_\pi}[A_t^\sigma f_t].
    \]
    \end{lemma}
    
    \begin{proof}
    For the endpoint term, condition on $T_\gamma=t$. On the event $B_t=1$, all continuation variables $C_0,\ldots,C_t$ equal one. Hence every decision needed to generate the prefix through $a_t$ is active under AS, and all such decisions are active by definition under VT. The semantic law therefore uses $\pi$ on the entire contributing prefix, exactly as $\mathbb P_\pi$ does. On $B_t=0$, the contribution vanishes. Summing over $t$ proves the first identity.
    
    For the second identity, on $A_t^\sigma=1$ every decision needed to generate the prefix through $a_t$ is active under the corresponding semantics, so the contributing prefix law again agrees with $\mathbb P_\pi$. On $A_t^\sigma=0$, the contribution is zero.
    \end{proof}
    
    \begin{lemma}[Endpoint reduction and active-decision weights]
    \label{lem:sdh-gate-expectations}
    For either semantics $\sigma\in\{\mathrm{AS},\mathrm{VT}\}$,
    \[
    \mathbb E_{q_\sigma}[\widehat R_\gamma]
    =
    \mu^\pi_\alpha(r)
    =
    \mathbb E_{\tau\sim\pi}\!\Big[
    \sum_{t\ge0}\gamma^t
    \Big(\prod_{k=0}^{t}\alpha_k\Big)r_t
    \Big].
    \]
    Moreover, under $\mathbb P_\pi$, conditioned on the state-action trajectory $\tau$,
    \[
    \mathbb E_{\mathbb P_\pi}
    [A_t^{\mathrm{AS}}\mid\tau]
    =
    \gamma^t\prod_{k=0}^{t-1}\alpha_k,
    \qquad
    \mathbb E_{\mathbb P_\pi}
    [A_t^{\mathrm{VT}}\mid\tau]
    =
    \gamma^t.
    \]
    \end{lemma}
    
    \begin{proof}
    By Lem.~\ref{lem:sdh-gated-prefix-transfer} and the geometric-horizon identity in \Cref{thm:survival-lagrangian-main},
    \[
    \mathbb E_{q_\sigma}[\widehat R_\gamma]
    =
    \frac{1}{1-\gamma}
    \mathbb E_{\mathbb P_\pi}[B_{T_\gamma}r_{T_\gamma}]
    =
    \mu^\pi_\alpha(r),
    \]
    which is the displayed discounted expansion by the definition of $\mu^\pi_\alpha$.
    
    For the decision gates, under $\mathbb P_\pi$ the geometric horizon $T_\gamma$ is independent of the trajectory and continuation variables, with $\Pr(T_\gamma\ge t)=\gamma^t$. Conditional on the state-action trajectory $\tau$, the continuation variables are independent and satisfy
    \[
    \mathbb E_{\mathbb P_\pi}[C_k\mid\tau]=\alpha_k.
    \]
    Therefore
    \[
    \mathbb E_{\mathbb P_\pi}
    [A_t^{\mathrm{AS}}\mid\tau]
    =
    \gamma^t\prod_{k=0}^{t-1}\alpha_k,
    \qquad
    \mathbb E_{\mathbb P_\pi}
    [A_t^{\mathrm{VT}}\mid\tau]
    =
    \gamma^t.
    \]
    \end{proof}

	\subsection{AS/VT classification theorem}
	
	Informally, AS and VT have the same survival-weighted task reward after marginalizing the geometric endpoint and continuation process. They differ only in the KL term, because they make different commitments about which post-continuation-failure actions remain semantic policy decisions.

    \begin{theorem}[Exact AS and VT objectives]
    \label{thm:as-vt-elbo}
    The semantic endpoint SDH ELBOs reduce to
    \[
    \begin{aligned}
    \mathcal J_{\mathrm{AS}}(\pi,\pi_0)
    =
    \mathbb E_{\tau\sim\pi}\!\Big[
    &\sum_{t\ge0}\gamma^t\Big(\prod_{k=0}^{t}\alpha_k\Big)r_t
    -\kappa\sum_{t\ge0}\gamma^t\Big(\prod_{k=0}^{t-1}\alpha_k\Big)
    \log\frac{\pi(a_t\mid s_t)}{\pi_0(a_t\mid s_t)}
    \Big],
    \end{aligned}
    \]
    and
    \[
    \begin{aligned}
    \mathcal J_{\mathrm{VT}}(\pi,\pi_0)
    =
    \mathbb E_{\tau\sim\pi}\!\Big[
    &\sum_{t\ge0}\gamma^t\Big(\prod_{k=0}^{t}\alpha_k\Big)r_t
    -\kappa\sum_{t\ge0}\gamma^t
    \log\frac{\pi(a_t\mid s_t)}{\pi_0(a_t\mid s_t)}
    \Big].
    \end{aligned}
    \]
    \end{theorem}
    
    \begin{proof}
    By Thm.~\ref{thm:sdh-semantic-elbo},
    \[
    \mathcal J_\sigma(\pi,\pi_0)
    =
    \mathbb E_{q_\sigma}[\widehat R_\gamma]
    -
    \kappa\mathbb E_{q_\sigma}\!\Big[
    \sum_{t\ge0}A_t^\sigma\ell_t^\pi
    \Big],
    \qquad
    \ell_t^\pi:=\log\frac{\pi(a_t\mid s_t)}{\pi_0(a_t\mid s_t)}.
    \]
    Lemma~\ref{lem:sdh-gate-expectations} gives the common reward term. For the policy-KL term, Lemma~\ref{lem:sdh-gated-prefix-transfer}, the stated absolute
    integrability, and the tower property give
    \begin{equation*}
    \mathbb E_{q_\sigma}\!\Big[
    \sum_{t\ge0}A_t^\sigma\ell_t^\pi
    \Big]
    =
    \mathbb E_{\mathbb P_\pi}\!\Big[
    \sum_{t\ge0}A_t^\sigma\ell_t^\pi
    \Big]
    =
    \mathbb E_{\tau\sim\pi}\!\Big[
    \sum_{t\ge0}
    \mathbb E_{\mathbb P_\pi}
    [A_t^\sigma\mid\tau]\,
    \ell_t^\pi
    \Big].
    \end{equation*}
    Substituting the two conditional weights from Lemma~\ref{lem:sdh-gate-expectations} yields the AS and VT expressions.
    \end{proof}
    
    \begin{remark}[Pathwise inference model]
    The endpoint potential is not a pathwise rewriting of a stopped-prefix CaI potential: exponentiating $B_{T_\gamma}r_{T_\gamma}/(1-\gamma)$ generally gives a different Gibbs posterior from exponentiating a cumulative prefix reward. The exact invariant needed here is the reduced variational policy objective. The endpoint construction yields the same survival-weighted reward functional $\mu^\pi_\alpha(r)$ together with the same semantic trajectory KLs, and all downstream Bellman and algorithmic results depend only on this reduced objective.
    \end{remark}
	
	\subsection{Consequences for dynamic programming}
    \label{app:bellman-consequences}
	
	The classification has a direct algorithmic consequence.  AS preserves the
	standard soft Bellman structure under a variable discount.  VT separates the
	survival-return critic from an ordinary discounted policy-divergence term.
	
	For AS, define
	\[
	u_t:=\gamma^t\prod_{k=0}^{t-1}\alpha_k.
	\]
	Then Thm.~\ref{thm:as-vt-elbo} gives
	\[
	\mathcal J_{\mathrm{AS}}(\pi,\pi_0)
	=
	\mathbb E_{\tau\sim\pi}
	\Big[
	\sum_{t\ge0}
	u_t
	\Big(
	\alpha_t r_t
	-
	\kappa \ell_t^\pi
	\Big)
	\Big].
	\]
	Thus AS admits the soft Bellman recursion
	\[
	Q_{\mathrm{AS}}^\pi(s,a)
	=
	\alpha(s,a)r(s,a)
	+
	\gamma\alpha(s,a)
	\mathbb E_{s'\sim P(\cdot\mid s,a)}
	\left[
	V_{\mathrm{AS}}^\pi(s')
	\right],
	\]
	with
	\[
	V_{\mathrm{AS}}^\pi(s)
	=
	\mathbb E_{a\sim\pi(\cdot\mid s)}
	\Big[
	Q_{\mathrm{AS}}^\pi(s,a)
	-
	\kappa
	\log\frac{\pi(a\mid s)}{\pi_0(a\mid s)}
	\Big].
	\]
	This preserves the usual timing convention: the reward appears in the $Q$-backup, while the policy regularizer appears in the $V$-backup. The current model-defined continuation probability $\alpha(s,a)$ gates the current reward and future value, but it does not gate the policy KL cost of choosing the current action.
	
	VT has a different structure.  Its objective decomposes as
	\[
	\mathcal J_{\mathrm{VT}}(\pi,\pi_0)
	=
	R_{\mathrm{surv}}^\pi
	-
	\kappa K^\pi,
	\]
	where
	\[
	R_{\mathrm{surv}}^\pi
	:=
	\mathbb E_{\tau\sim\pi}
	\Big[
	\sum_{t\ge0}
	\gamma^t
	\Big(
	\prod_{k=0}^{t}\alpha_k
	\Big)
	r_t
	\Big]
	\quad\text{and}\quad
	K^\pi
	:=
	\mathbb E_{\tau\sim\pi}
	\Big[
	\sum_{t\ge0}
	\gamma^t
	\log\frac{\pi(a_t\mid s_t)}{\pi_0(a_t\mid s_t)}
	\Big].
	\]
	The survival reward component satisfies the variable-discount recursion
	\[
	Q_{\mathrm{surv}}^\pi(s,a)
	=
	\alpha(s,a)r(s,a)
	+
	\gamma\alpha(s,a)
	\mathbb E_{s'\sim P(\cdot\mid s,a)}
	\left[
	V_{\mathrm{surv}}^\pi(s')
	\right],
	\]
	with
	\[
	V_{\mathrm{surv}}^\pi(s)
	=
	\mathbb E_{a\sim\pi(\cdot\mid s)}
	\left[
	Q_{\mathrm{surv}}^\pi(s,a)
	\right].
	\]
	The KL component instead follows the ordinary discounted recursion
	\[
	K^\pi(s)
	=
	\mathbb E_{a\sim\pi(\cdot\mid s)}
	\Big[
	\log\frac{\pi(a\mid s)}{\pi_0(a\mid s)}
	+
	\gamma
	\mathbb E_{s'\sim P(\cdot\mid s,a)}
	K^\pi(s')
	\Big].
	\]
	Therefore VT is not a single standard soft Bellman recursion on the original MDP
	with one \(Q\)-function and one \(V\)-function.  Its reward and KL components live
	under different continuation structures.  This is precisely the algorithmic
	distinction used in the main text: AS yields a SAC-style variable-discount soft
	Bellman operator, whereas VT naturally separates a survival-return critic from an
	ordinary discounted policy-divergence constraint, matching the MPO-style
	variational picture.
	
	\subsection{Stability of the variable-discount critic}
    \label{app:contraction}
	
	Both semantics use the same survival-return critic.  The relevant evaluation
	operator has reward \(\bar r(s,a)=\alpha(s,a)r(s,a)\) and discount
	\(\bar\gamma(s,a)=\gamma\alpha(s,a)\).  Since \(\alpha(s,a)\in[0,1]\), the
	effective discount satisfies \(0\le \bar\gamma(s,a)\le \gamma<1\).
	
	\begin{lemma}[Variable-discount contraction]
		\label{lem:contraction-var-disc}
		\label{lem:sdh-variable-discount-contraction}
		Let \(\bar r\) be bounded and let
		\(\bar\gamma:\mathcal S\times\mathcal A\to[0,\gamma]\).  For a fixed policy
		\(\pi\), define
		\[
		(\mathcal T_{\bar r,\bar\gamma}^{\pi}V)(s)
		:=
		\mathbb E_{a\sim\pi(\cdot\mid s)}
		\left[
		\bar r(s,a)
		+
		\bar\gamma(s,a)
		\mathbb E_{s'\sim P(\cdot\mid s,a)}
		V(s')
		\right].
		\]
		Then \(\mathcal T_{\bar r,\bar\gamma}^{\pi}\) is a \(\gamma\)-contraction in
		the sup norm:
		\[
		\left\|
		\mathcal T_{\bar r,\bar\gamma}^{\pi}V
		-
		\mathcal T_{\bar r,\bar\gamma}^{\pi}W
		\right\|_\infty
		\le
		\gamma
		\|V-W\|_\infty.
		\]
		Consequently, it has a unique fixed point and value iteration converges to it
		from any bounded initialization.
	\end{lemma}
	
	\begin{proof}
		For any bounded \(V,W\) and any \(s\),
		\begin{eqnarray*}
			\left|
			(\mathcal T_{\bar r,\bar\gamma}^{\pi}V)(s)
			-
			(\mathcal T_{\bar r,\bar\gamma}^{\pi}W)(s)
			\right|
			&=&
			\left|
			\mathbb E_{a\sim\pi(\cdot\mid s)}
			\left[
			\bar\gamma(s,a)
			\mathbb E_{s'\sim P(\cdot\mid s,a)}
			\big(V(s')-W(s')\big)
			\right]
			\right|
			\\
			&\le&
			\mathbb E_{a\sim\pi(\cdot\mid s)}
			\left[
			\bar\gamma(s,a)
			\mathbb E_{s'\sim P(\cdot\mid s,a)}
			|V(s')-W(s')|
			\right]
			\\
			&\le&
			\gamma\|V-W\|_\infty.
		\end{eqnarray*}
		Taking the supremum over \(s\) gives the contraction bound.  Existence,
		uniqueness, and convergence follow from the Banach fixed point theorem.
	\end{proof}
	
	Thus SDH shortens credit assignment locally without breaking the contraction
	structure of discounted policy evaluation.
	
	\subsection{Summary: ELBO-valid SDH regularization}
	
    The survival-constrained formulation identifies the random-horizon payoff optimized by SDH. This section applies the Gibbs variational construction directly to that same geometric endpoint, scaled by $(1-\gamma)^{-1}$ to match the Bellman-return normalization. Because the endpoint payoff is uniformly bounded for bounded rewards, the corresponding evidence is finite for every $\kappa>0$.
    
    The reward side and policy-KL side play different roles. The common endpoint potential uses $B_{T_\gamma}r_{T_\gamma}/(1-\gamma)$. The active decision gates
    \[
    A_t^{\mathrm{AS}}=H_t\prod_{k=0}^{t-1}C_k,
    \qquad
    A_t^{\mathrm{VT}}=H_t
    \]
    determine which policy likelihood factors appear in the semantic trajectory laws. The resulting weighted log-ratio terms are therefore genuine trajectory KLs, not pseudo-KL penalties.
    
    After marginalization, AS and VT share exactly the same survival-weighted task objective and differ only in the policy-KL horizon induced by their post-continuation semantics. AS yields the survival-weighted KL-regularized Bellman recursion used by AS-SAC; VT pairs the same survival-weighted task critic with ordinary discounted, KL-constrained policy improvement in VT-MPO. The endpoint Gibbs posterior itself need not coincide with a stopped-prefix Gibbs posterior; the exact object preserved, and the only object used downstream, is the reduced regularized policy objective.

    \newpage

    \section{Absorbing State: Uniform Prior and AS-SAC}
    \label{app:as-sac-unified}
    
    This section derives the algorithmic consequences of the absorbing-state (AS) semantics when the CaI prior is a normalized uniform policy. The main point is simple but important: under AS, constants in the policy-ratio term are no longer globally constant, because they are accumulated only while the stochastic decision horizon survives. This yields a survival-weighted living cost and a SAC-style Bellman recursion under survival shaping, while temperature adaptation remains standard SAC-style policy-improvement tuning.
    
    \subsection{AS objective under a normalized uniform prior}
    \label{app:as:uniform-obj}
    
    Recall the AS objective from Thm.~\ref{thm:as-vt-elbo}:
    \begin{equation}
    	\label{eq:as-obj-recall}
    	\mathcal J_{\mathrm{AS}}(\pi)
    	=
    	\mathbb E_{\tau\sim\pi}\Bigg[
    	\sum_{t\ge 0}\gamma^t
    	\Big(\prod_{k=0}^{t}\alpha_k\Big)r_t
    	-
    	\kappa
    	\sum_{t\ge 0}\gamma^t
    	\Big(\prod_{k=0}^{t-1}\alpha_k\Big)
    	\log\frac{\pi(a_t\mid s_t)}{\pi_0(a_t\mid s_t)}
    	\Bigg],
    \end{equation}
    where $\alpha_t=\alpha(s_t,a_t)$ and $r_t=r(s_t,a_t)$.
    Assume throughout this section that the reference policy is normalized uniform:
    \begin{equation}
    	\label{eq:uniform-prior}
    	\log \pi_0(a\mid s)=-\ell_c
    	\qquad \forall (s,a).
    \end{equation}
    For a finite action space, $\ell_c=\log|\mathcal A|$. In continuous control, $\ell_c$ should be interpreted relative to a chosen reference measure, and is implemented through a target-entropy constant.
    
    Define the AS survival-to-decision weight
    \begin{equation}
    	\label{eq:wt-def}
    	w_t(\tau)
    	:=
    	\gamma^t\prod_{k=0}^{t-1}\alpha(s_k,a_k),
    	\qquad
    	w_0=1,
    \end{equation}
    and the survival-shaped reward and discount
    \begin{equation}
    	\label{eq:tilde-r-gamma}
    	\tilde r(s,a):=\alpha(s,a)r(s,a),
    	\qquad
    	\tilde\gamma(s,a):=\gamma\alpha(s,a).
    \end{equation}
    
    \begin{corollary}[CaI-AS with a normalized uniform prior]
    	\label{cor:as-uniform-structure}
    	Under \eqref{eq:uniform-prior}, the AS objective can be written as
    	\begin{equation}
    		\label{eq:as-uniform-obj}
    		\mathcal J_{\mathrm{AS}}(\pi)
    		=
    		\mathbb E_{\tau\sim\pi}
    		\Bigg[
    		\sum_{t\ge0}
    		w_t(\tau)
    		\Big(
    		\tilde r(s_t,a_t)
    		+
    		\kappa\,\mathcal H(\pi(\cdot\mid s_t))
    		-
    		\kappa\,\ell_c
    		\Big)
    		\Bigg].
    	\end{equation}
    	Equivalently, with
    	\begin{equation}
    		\label{eq:Zpi-def}
    		Z(\pi):=
    		\mathbb E_{\tau\sim\pi}
    		\Big[
    		\sum_{t\ge0}w_t(\tau)
    		\Big],
    	\end{equation}
    	we have the decomposition
    	\begin{equation}
    		\label{eq:as-decomp}
    		\mathcal J_{\mathrm{AS}}(\pi)
    		=
    		\underbrace{
    		\mathbb E_{\tau\sim\pi}
    		\Big[
    		\sum_{t\ge0}
    		w_t(\tau)
    		\big(
    		\tilde r_t+\kappa\mathcal H(\pi(\cdot\mid s_t))
    		\big)
    		\Big]
    		}_{\text{survival-shaped maximum-entropy return}}
    		-
    		\underbrace{\kappa\ell_c Z(\pi)}_{\text{feasibility-weighted living cost}}.
    	\end{equation}
    \end{corollary}
    
    \begin{proof}
    	Substituting $\log\pi_0(a\mid s)=-\ell_c$ into \eqref{eq:as-obj-recall} gives
    	\[
    	-\kappa
    	\log\frac{\pi(a_t\mid s_t)}{\pi_0(a_t\mid s_t)}
    	=
    	-\kappa\log\pi(a_t\mid s_t)-\kappa\ell_c .
    	\]
    	Moreover,
    	\[
    	\gamma^t\prod_{k=0}^{t}\alpha_k r_t
    	=
    	w_t(\tau)\alpha_t r_t
    	=
    	w_t(\tau)\tilde r_t .
    	\]
    	Taking the conditional expectation over $a_t\sim\pi(\cdot\mid s_t)$ turns
    	$-\mathbb E[\log\pi(a_t\mid s_t)\mid s_t]$ into
    	$\mathcal H(\pi(\cdot\mid s_t))$, yielding \eqref{eq:as-uniform-obj}.
    	The decomposition \eqref{eq:as-decomp} follows by collecting the constant per-decision term into $Z(\pi)$.
    \end{proof}
    
    The corollary identifies the main difference from standard maximum-entropy RL. With ordinary geometric discounting, a constant per step contributes a fixed offset proportional to $(1-\gamma)^{-1}$. Under AS, the corresponding mass $Z(\pi)$ depends on the policy through $\alpha(s,a)$. Hence the living-cost term cannot generally be discarded without changing the optimization problem.
    
    \begin{remark}[Finite action spaces]
    	\label{rem:finite-A-horizon-regularization}
    	If $|\mathcal A|<\infty$ and $\pi_0(a\mid s)=1/|\mathcal A|$, then $\ell_c=\log|\mathcal A|$ and
    	\[
    	\kappa\ell_c Z(\pi)
    	=
    	\kappa\log(|\mathcal A|)\,\cdot\,
    	\mathbb E_{\tau\sim\pi}
    	\Big[
    	\sum_{t\ge0}
    	\gamma^t\prod_{k=0}^{t-1}\alpha(s_k,a_k)
    	\Big].
    	\]
    	Thus the normalized uniform prior induces a regularization on the expected feasibility-weighted decision horizon.
    \end{remark}
    
    \begin{counterexample}[Dropping the living cost can change the optimizer]
    	\label{cex:living-cost-changes-opt-stochastic}
    	Consider the one-state MDP with $\mathcal S=\{s\}$ and
    	$\mathcal A=\{a_{\mathrm{cont}},a_{\mathrm{stop}}\}$.
    	The transition is $P(s\mid s,a)=1$ for both actions,
    	\[
    	\alpha(s,a_{\mathrm{cont}})=1,\qquad
    	\alpha(s,a_{\mathrm{stop}})=0,
    	\]
    	and
    	\[
    	r(s,a_{\mathrm{cont}})=r>0,\qquad
    	r(s,a_{\mathrm{stop}})=0.
    	\]
    	Let $\pi_p(a_{\mathrm{cont}}\mid s)=p$ and
    	$\pi_p(a_{\mathrm{stop}}\mid s)=1-p$, with $p\in(0,1)$.
    	For a normalized uniform prior over the two actions, $\ell_c=\log2$.
    	Writing $h(p):=-p\log p-(1-p)\log(1-p)$, the exact AS objective is
    	\begin{equation}
    		\mathcal J_{\mathrm{AS}}(\pi_p)
    		=
    		\frac{pr+\kappa h(p)-\kappa\log2}{1-\gamma p}.
    	\end{equation}
    	If the living-cost term is removed, the resulting objective becomes
    	\begin{equation}
    		\mathcal J_{\mathrm{AS\text{-}N}}(\pi_p)
    		=
    		\frac{pr+\kappa h(p)}{1-\gamma p}.
    	\end{equation}
    	These objectives can have different strictly stochastic maximizers. For example, with
    	$(\gamma,\kappa,r)=(0.9,1,0.4)$,
    	\[
    	\arg\max_{p\in(0,1)}\mathcal J_{\mathrm{AS}}(\pi_p)\approx0.707,
    	\qquad
    	\arg\max_{p\in(0,1)}\mathcal J_{\mathrm{AS\text{-}N}}(\pi_p)\approx0.984.
    	\]
    \end{counterexample}
    
    \subsection{AS admits a SAC-style Bellman recursion}
    \label{app:as:sac-recursion}
    
    The preceding result explains why AS is not identical to standard SAC. Nevertheless, AS retains the essential algorithmic structure of SAC: policy evaluation is a soft Bellman recursion on the survival-shaped MDP.
    
    \begin{theorem}[AS soft policy evaluation]
    	\label{thm:as-soft-policy-evaluation}
    	Fix a policy $\pi$ and a normalized uniform prior. Under AS semantics, the corresponding soft action-value function satisfies
    	\begin{align}
    		\label{eq:as-soft-bellman}
    		Q^\pi(s,a)
    		&:=
    		\tilde r(s,a)
    		+
    		\tilde\gamma(s,a)
    		\mathbb E_{s'\sim P(\cdot\mid s,a)}
    		\big[V^\pi(s')\big],
    		\\
    		V^\pi(s)
    		&:=
    		\mathbb E_{a\sim\pi(\cdot\mid s)}
    		\big[
    		Q^\pi(s,a)-\kappa(\log\pi(a\mid s)+\ell_c)
    		\big].
    	\end{align}
    	The associated Bellman operator is a $\gamma$-contraction whenever rewards are bounded.
    \end{theorem}
    
    \begin{proof}
    	Starting from \eqref{eq:as-uniform-obj}, condition on an initial state-action pair $(s,a)$.
    	The immediate reward contribution is $\tilde r(s,a)$, the immediate prior-normalization contribution is
    	$-\kappa\ell_c$, and future terms are discounted by
    	$\tilde\gamma(s,a)=\gamma\alpha(s,a)$.
    	The entropy term appears in the next state value through the standard soft value definition.
    	Because $\tilde\gamma(s,a)\le\gamma<1$, contraction follows from Lem.~\ref{lem:contraction-var-disc}.
    \end{proof}
    
    Thus AS-SAC differs from SAC only through the shaped reward, the state-action-dependent discount, and the living-cost baseline.
    
    For a replay transition $(s,a,r,s',d)\sim\mathcal D$, define
    \[
    \alpha=\alpha(s,a),
    \qquad
    \tilde r=\alpha r,
    \qquad
    \tilde\gamma=(1-d)\gamma\alpha.
    \]
    Sampling $a'\sim\pi_\phi(\cdot\mid s')$, the single-critic AS-SAC target is
    \begin{equation}
    	\label{eq:as-sac-target-single}
    	\hat V(s')
    	:=
    	\min_{i\in\{1,2\}}\bar Q_{\bar\theta_i}(s',a')
    	-
    	\kappa\big(
    	\log\pi_\phi(a'\mid s')
    	+
    	\ell_c
    	\big),
    	\qquad
    	y
    	:=
    	\tilde r+\tilde\gamma\hat V(s').
    \end{equation}
    In continuous-action experiments, $\ell_c$ is implemented by a target-entropy baseline, so the soft value subtracts the current KL-to-prior surrogate rather than only the log-policy term. Setting this baseline to zero gives the ablation that drops the AS living cost.
    
    \subsubsection{Adaptive temperature as a policy-improvement dual}
    \label{app:as:sac-temperature}
    
    The CaI derivation determines the AS objective for any fixed temperature
    $\kappa$; it does not require $\kappa$ itself to be adapted through a
    trajectory-level survival-weighted information constraint. We instead
    adapt $\kappa$ through a local policy-improvement constraint.
    
    At each actor update, hold fixed the current regularized AS critic
    estimate
    \[
        \widehat Q^{\rm AS}(s,a)
        :=
        \min_i Q_{\phi_i}(s,a),
    \]
    and let $\pi'$ denote the candidate next policy. For a fixed state
    distribution $d$, taken here as the replay-state distribution
    $\mathcal D$, consider
    \begin{equation}
        \label{eq:as-local-kl-pi}
        \max_{\pi'}\;
        \E_{s\sim d,\;a\sim\pi'(\cdot\mid s)}
        \Big[
            \widehat Q^{\rm AS}(s,a)
        \Big]
        \quad
        \text{subject to}
        \quad
        \E_{s\sim d}
        D_{\rm KL}\!\Big(
            \pi'(\cdot\mid s)
            \,\|\,\pi_0(\cdot\mid s)
        \Big)
        \le \varepsilon .
    \end{equation}
    
    Its policy-improvement Lagrangian is
    \begin{equation}
        \label{eq:as-local-kl-lagrangian}
        \mathcal L_{\rm PI}(\pi',\kappa)
        =
        \E_{s\sim d}
        \Big[
            \E_{a\sim\pi'(\cdot\mid s)}
            \widehat Q^{\rm AS}(s,a)
            -
            \kappa
            D_{\rm KL}\!\Big(
                \pi'(\cdot\mid s)
                \,\|\,\pi_0(\cdot\mid s)
            \Big)
        \Big]
        +
        \kappa\varepsilon .
    \end{equation}
    
    For fixed $\kappa>0$, the nonparametric optimizer is the usual soft
    policy-improvement operator
    \begin{equation}
        \label{eq:as-soft-pi}
        \pi'_\kappa(a\mid s)
        \propto
        \pi_0(a\mid s)
        \exp\!\Big\{
            \frac{\widehat Q^{\rm AS}(s,a)}{\kappa}
        \Big\}.
    \end{equation}
    
    Let
    $g(\kappa):=\max_{\pi'}\mathcal L_{\rm PI}(\pi',\kappa)$
    denote the dual function. Under exact inner maximization, the envelope
    theorem gives
    \[
        g'(\kappa)
        =
        \varepsilon
        -
        \E_{s\sim d}
        D_{\rm KL}\!\Big(
            \pi'_\kappa(\cdot\mid s)
            \,\|\,\pi_0(\cdot\mid s)
        \Big).
    \]
    Projected dual descent therefore yields
    \begin{equation}
        \label{eq:as-kappa-kl-update}
        \kappa
        \leftarrow
        \Big[
            \kappa
            +
            \beta_\kappa
            \Big(
                \E_{s\sim d}
                D_{\rm KL}\!\Big(
                    \pi'_\kappa(\cdot\mid s)
                    \,\|\,\pi_0(\cdot\mid s)
                \Big)
                -
                \varepsilon
            \Big)
        \Big]_+ .
    \end{equation}
    
    For a normalized uniform prior,
    $\log\pi_0(a\mid s)=-\ell_c$, and therefore
    \[
        D_{\rm KL}(\pi'\|\pi_0)
        =
        \ell_c-\mathcal H(\pi').
    \]
    Defining
    $\bar{\mathcal H}:=\ell_c-\varepsilon$
    gives
    \begin{equation}
        \label{eq:as-sac-temperature-update}
        \kappa
        \leftarrow
        \Big[
            \kappa
            +
            \beta_\kappa
            \Big(
                \bar{\mathcal H}
                -
                \E_{s\sim d}
                \mathcal H(\pi(\cdot\mid s))
            \Big)
        \Big]_+ ,
    \end{equation}
    which is the standard SAC-style temperature controller
    \citep{haarnoja2018soft}. In our off-policy realization we take
    $d=\mathcal D$, so the entropy target is imposed on the replay-state
    distribution used for actor improvement. With a parametric actor, the
    SAC actor update approximately solves the inner policy-improvement
    problem within the policy class, while the temperature update provides
    a stochastic dual step evaluated at the current actor. Under exact
    inner maximization, this reduces to the dual gradient above.
    
    This local policy-improvement interpretation does not remove the AS
    living cost. For a normalized uniform prior,
    \[
        -\kappa
        \log
        \frac{\pi'(a\mid s)}{\pi_0(a\mid s)}
        =
        -\kappa\log\pi'(a\mid s)
        -
        \kappa\ell_c .
    \]
    The current $-\kappa\ell_c$ term is independent of the candidate action
    distribution and therefore does not affect the local actor argmax.
    Future occurrences of this term under the currently evaluated policy
    are already represented in the regularized AS critic
    $\widehat Q^{\rm AS}$ through the AS Bellman recursion. Thus the actor
    uses the usual SAC objective, while AS policy evaluation retains the
    CaI-derived survival-weighted living cost.
    
    Importantly, the temperature update above is the dual of the local
    replay-state policy-improvement constraint
    \eqref{eq:as-local-kl-pi}, not of a trajectory-level
    survival-weighted information constraint. Its entropy statistic is
    therefore averaged under $d=\mathcal D$ without an additional survival
    weight. Survival semantics remain in policy evaluation through the
    regularized AS critic and its state-action-dependent discount.
    
\end{document}